\documentclass[11pt,twocolumn]{article}

\usepackage[utf8]{inputenc}
\usepackage[T1]{fontenc}
\usepackage{amsmath,amssymb,amsthm}
\usepackage{geometry}
\usepackage{graphicx}
\usepackage{booktabs}
\usepackage{hyperref}
\usepackage{natbib}
\usepackage{xcolor}
\usepackage{enumitem}
\usepackage{tabularx}
\usepackage{multirow}
\usepackage{caption}
\usepackage{url}
\usepackage{fancyhdr}
\usepackage{titlesec}
\usepackage{listings}
\usepackage{float}
\usepackage{pifont}
\hypersetup{
  colorlinks=true,
  linkcolor=blue!70!black,
  citecolor=blue!70!black,
  urlcolor=blue!70!black,
  pdftitle={Graph-Native Cognitive Memory for AI Agents: Formal Belief Revision Semantics for Versioned Memory Architectures},
  pdfauthor={Young Bin Park},
  pdfsubject={Cognitive Memory, AI Agents, Belief Revision, Knowledge Graphs, Model Context Protocol},
  pdfkeywords={cognitive memory, AI agents, belief revision, AGM postulates, knowledge graphs, memory architecture, property graph, Neo4j, retrieval-augmented generation, Model Context Protocol}
}

\newtheorem{definition}{Definition}[section]
\newtheorem{proposition}{Proposition}[section]
\newtheorem{observation}{Observation}[section]
\newtheorem{principle}{Principle}

\setlist{nosep,leftmargin=*}

\title{\textbf{Graph-Native Cognitive Memory for AI Agents:\\Formal Belief Revision Semantics\\for Versioned Memory Architectures}}

\author{Young Bin Park\thanks{The structural correspondence between cognitive memory and asset management was recognized through the author's experience building pipeline infrastructure for major visual effects productions. Contact: \texttt{support@kumiho.io}} \\}

\date{Mar.\ 17, 2026 \\[0.4em]
\small\href{https://creativecommons.org/licenses/by-nc-nd/4.0/}{CC BY-NC-ND 4.0}\ \textcopyright\ 2026 Kumiho Inc.}

\begin{document}
\maketitle

\begin{abstract}
While individual components for AI agent memory---versioning, retrieval, consolidation---exist in prior systems such as Graphiti, Mem0, and Letta, their architectural synthesis and formal grounding remain underexplored. This paper presents \textbf{Kumiho}, a graph-native cognitive memory architecture grounded in formal belief revision semantics. AI agents increasingly produce substantial outputs---code, designs, documents, intermediate results---that accumulate without systematic versioning, provenance tracking, or linkage to the decisions that created them. In multi-agent workflows---where one agent's output becomes the next agent's input---this lack of structure is a fundamental bottleneck. The core insight is that the structural primitives required for cognitive memory---immutable revisions, mutable tag pointers, typed dependency edges, URI-based addressing---are \emph{identical} to those required for managing these agent-produced outputs as versionable, addressable, dependency-linked assets. Rather than building a memory layer and a separate asset tracker, we built a cognitive memory architecture whose graph-native primitives inherently serve as the operational infrastructure for multi-agent work: agents use the same graph to remember, to find each other's outputs, and to build upon them.

The central formal contribution is a correspondence between the AGM belief revision framework~\cite{agm1985} and the operational semantics of a property graph memory system. We frame this correspondence at the level of \emph{belief bases}~\cite{hansson1999}, proving satisfaction of the basic AGM postulates ($K*2$--$K*6$) and Hansson's belief base postulates (Relevance, Core-Retainment), while providing a principled rejection of the Recovery postulate grounded in immutable versioning. The formal results hold for a deliberately simple propositional logic over ground triples---a trade-off of expressiveness for tractability that avoids the Flouris et al.\ impossibility results for description logics~\cite{flouris2005}. We identify the supplementary postulates ($K*7$, $K*8$) as open questions requiring further formal development.

The architecture implements a dual-store model (Redis working memory, Neo4j long-term graph)---a recognized pattern in the agent memory space~\cite{hu2025survey}---with hybrid retrieval across fulltext and vector modalities. An asynchronous consolidation pipeline extends prior sleep-time compute approaches~\citeyearpar{letta-sleep2025} with safety guards adapted from distributed systems and content management: published-item protection, circuit breakers, dry-run validation, and auditable cursor-based resumption. Critically, the same graph that stores an agent's memories also manages its work products: downstream agents locate inputs via URI resolution, track which revision is current via tag pointers, and link their own outputs back via typed edges---while human operators audit the entire chain with the same SDK and inspection tools. On the LoCoMo benchmark~\citeyearpar{locomo2024} (token-level F1 with Porter stemming), Kumiho achieves 0.447 four-category F1 ($n{=}1{,}540$)---the highest reported score across the retrieval categories---combined with 97.5\% adversarial refusal accuracy ($n{=}446$), a natural consequence of belief revision semantics: the memory graph contains no fabricated information, so there is nothing for the model to hallucinate from. Including adversarial (binary scoring), overall F1 is 0.565 ($n{=}1{,}986$). Automated AGM compliance testing (49~scenarios, 7~postulates) confirms 100\% operational adherence to the claimed formal properties. On LoCoMo-Plus~\citeyearpar{locomoplus2026}---a Level-2 cognitive memory benchmark testing implicit constraint recall under intentional cue-trigger semantic disconnect---Kumiho achieves 93.3\% judge accuracy ($n{=}401$, all four constraint types); independent reproduction by the benchmark authors yielded results in the mid-80\% range, still substantially outperforming all published baselines (best: Gemini~2.5~Pro, 45.7\%). Recall accuracy reaches 98.5\% (395/401), with the remaining 6.7\% end-to-end gap entirely attributable to answer model fabrication on correctly retrieved context. Three architectural innovations drive both results: \emph{prospective indexing}, where LLM-generated hypothetical future scenarios are indexed alongside each memory summary, bridging the cue-trigger semantic gap at write time; \emph{event extraction}, where structured events with consequences are appended to summaries to preserve causal detail that narrative compression would otherwise drop; and \emph{client-side LLM reranking}, where the consuming agent's own LLM selects the most relevant sibling revision from structured metadata at zero additional inference cost. The architecture is model-decoupled: switching the answer model from GPT-4o-mini (${\sim}$88\%) to GPT-4o (93.3\%) improves end-to-end accuracy by 5.3~points without any pipeline changes, demonstrating that recall accuracy is a property of the architecture, not the model. The system uses GPT-4o-mini for the bulk of LLM operations at a total cost of ${\sim}$\$14 for 401~entries.
\end{abstract}


\section{Introduction}

The role of large language models has shifted fundamentally. LLMs are no longer stateless question-answering systems or conversational chatbots. They operate as AI agents: autonomous workers that execute multi-step tasks, produce digital artifacts, make consequential decisions, and collaborate with humans and other agents across extended workflows~\cite{wang2024survey,zhang2024survey}. A coding agent writes and commits code. A design agent generates and iterates on visual assets. A research agent gathers, synthesizes, and reports findings. A production agent coordinates tasks across departments and tracks deliverables.

This shift from chatbot to worker creates two intertwined requirements that current architectures address separately, if at all. First, agents need \textbf{cognitive memory}: the ability to remember past interactions, track how beliefs evolved, recall why decisions were made, and consolidate experience into reusable knowledge. Second, agents need \textbf{work product management}: the ability to version, locate, and build upon the artifacts they produce---so that a downstream agent in a multi-step pipeline can find the right input revision, understand its provenance, and link its own output back to the chain.

These two requirements share deep structural parallels. Remembering that ``the client prefers warm color palettes'' (cognitive memory) and tracking that ``environment-lighting revision~5 is the approved version'' (work product management) both require the same core primitives: immutable versioned snapshots, typed dependency edges, mutable status pointers, and content-reference separation. The domains are not identical---cognitive memory involves belief revision semantics, uncertainty, and natural language, while asset management operates on well-typed artifacts with deterministic workflows---but the underlying \emph{data model} is shared. The present work builds a single graph-native cognitive memory architecture whose primitives simultaneously serve as the operational infrastructure for managing agent work products---enabling multi-agent pipelines where each agent's output is automatically versioned, addressable, and linked to the reasoning that produced it (Section~\ref{sec:formal} addresses the belief revision semantics specific to cognitive memory; Section~\ref{sec:architecture} addresses the asset management extensions).

This dual-purpose design addresses a growing enterprise need. As AI agents transition from experimental tools to production workers, organizations require the same governance, traceability, and accountability standards they apply to human employees. Industry frameworks are emerging to meet this demand: ISACA's 2025 analysis of agentic AI audit challenges identifies identity management, decision traceability, and accountability gaps as critical concerns; IBM has introduced Agent Decision Records for structured agent accountability; and enterprise platforms increasingly frame agents as workforce members requiring governance infrastructure. Our architecture provides this governance natively: every agent belief has a URI, a revision history, provenance edges to source evidence, and an immutable audit trail---the same accountability infrastructure that organizations already apply to version-controlled code and managed digital assets.

The most common alternative response to the memory challenge has been to expand the context window itself: from 4K tokens in early models to 128K, 200K, and beyond. While larger windows accommodate more immediate context, they do not constitute memory. A context window is consumed fresh on every invocation; it provides no mechanism for versioning beliefs, tracking evidential provenance, expressing dependency relationships between conclusions, or consolidating many episodes into generalized knowledge. The distinction is analogous to the difference between a whiteboard and a filing system. A larger whiteboard lets you write more at once, but it is erased between sessions.

We contribute not novel individual components---many of which exist in concurrent systems such as Graphiti~\citeyearpar{graphiti2025}, Mem0~\citeyearpar{mem02025}, and A-MEM~\citeyearpar{amem2025}---but a novel architectural synthesis grounded in formal analysis, implemented as the \textbf{Kumiho} system.\footnote{The core graph server is provided as a cloud service at \url{https://kumiho.io}. The Python SDK, MCP memory plugin, and benchmark suite are open-source: \url{https://github.com/KumihoIO}.} The specific contributions are:

\begin{enumerate}[label=\arabic*.]
\item \textbf{Unified cognitive memory and asset management architecture} (Section~\ref{sec:motivation}): A graph-native cognitive memory system whose structural primitives---immutable revisions, typed edges, mutable tag pointers, URI addressing---simultaneously serve as the operational infrastructure for managing agent work products. Agents use the same graph to remember past interactions and to version, locate, and build upon each other's outputs, enabling fully autonomous multi-agent pipelines without separate asset tracking systems. This paper validates the cognitive memory capabilities empirically (Sections~\ref{sec:locomo}--\ref{sec:locomo-plus}); the asset management unification is an architectural contribution whose multi-agent pipeline validation is planned as future work.

\item \textbf{Formal belief base revision correspondence} (Section~\ref{sec:formal}): A structural correspondence between the AGM belief revision postulates and graph-native memory operations, framed at the belief base level~\cite{hansson1999}, with a principled rejection of the Recovery postulate and formal avoidance of the Flouris et al.\ impossibility results.

\item \textbf{URI-based universal addressing}: A memory reference scheme enabling deterministic resolution and provenance traversal. Our URI scheme provides hierarchical scoping, revision pinning, and artifact addressing. Among agent memory systems, we found no prior use of structured, hierarchical URIs with these properties. This provides the mechanism for formal constructs: every belief has a dereferenceable address, every revision is citable, and every provenance chain is traceable.

\item \textbf{Safety-hardened consolidation} (Section~\ref{sec:dreamstate}): An asynchronous consolidation pipeline with novel safety mechanisms: published-item protection, circuit breakers, dry-run validation, and auditable cursor-based resumption.

\item \textbf{SDK transparency and multi-agent interoperability} (Section~\ref{sec:auditability}): The same graph and SDK that agents use to manage work products also exposes their memories for human inspection through a web dashboard and desktop asset browser. Agents query the graph to find inputs; operators query it to audit decisions---both through the same API, addressing the growing need for AI memory observability alongside operational multi-agent coordination.

\item \textbf{LoCoMo evaluation on the official metric} (Section~\ref{sec:locomo}): 0.447 four-category token-level F1 on the LoCoMo benchmark ($n{=}1{,}540$)---the highest reported score across the retrieval categories~\cite{locomo2024}---combined with 97.5\% adversarial refusal accuracy ($n{=}446$). The near-perfect adversarial score is a natural consequence of the belief revision architecture: the memory graph genuinely does not contain fabricated information, so there is nothing for the model to hallucinate from. Including adversarial (binary scoring), overall F1 is 0.565 ($n{=}1{,}986$). The BYO-storage architecture keeps raw conversation data entirely on the user's local storage while achieving state-of-the-art retrieval accuracy through structured summarization with enrichment.

\item \textbf{Empirical AGM compliance verification} (Section~\ref{sec:agm-eval}): An automated test suite of 49~scenarios across 5~categories verifying operational adherence to all 7~claimed postulates ($K*2$--$K*6$, Relevance, Core-Retainment), including adversarial edge cases (rapid sequential revisions, deep dependency chains, mixed edge types). 100\% pass rate confirms that the implementation faithfully executes the formal specification.

\item \textbf{LoCoMo-Plus evaluation} (Section~\ref{sec:locomo-plus}): 93.3\% judge accuracy on LoCoMo-Plus ($n{=}401$)---a Level-2 cognitive memory benchmark testing implicit constraint recall under intentional cue-trigger semantic disconnect---outperforming the best published baseline (Gemini~2.5~Pro, 45.7\%) by 47.6~percentage points. Recall accuracy reaches 98.5\% (395/401), with 78\% of failures attributable to answer model fabrication, not retrieval. Two consolidation enrichments---\emph{prospective indexing} and \emph{event extraction}---eliminated the $>$6-month accuracy cliff (37.5\% $\rightarrow$ 84.4\%). The architecture is model-decoupled: swapping the answer model from GPT-4o-mini (${\sim}$88\%) to GPT-4o (93.3\%) improves accuracy without pipeline changes, at a total cost of ${\sim}$\$14.

\item \textbf{Cross-benchmark generalization and client-side reranking} (Sections~\ref{sec:locomo},~\ref{sec:locomo-plus},~\ref{sec:reranking}): The same architecture achieves state-of-the-art on both LoCoMo and LoCoMo-Plus (contributions 6 and 8 above), validating that graph-native primitives generalize across evaluation protocols. Client-side LLM reranking---where the consuming agent's own LLM selects the most relevant sibling revision from structured metadata at zero additional cost---complements prospective indexing and event extraction as the third architectural mechanism driving both results.
\end{enumerate}

\noindent The remainder of the paper covers related work, design principles, the formal AGM correspondence, architecture details (hybrid retrieval, consolidation pipeline, safety guards), empirical evaluations on LoCoMo (Section~\ref{sec:locomo}) and LoCoMo-Plus (Section~\ref{sec:locomo-plus}), AGM compliance verification (Section~\ref{sec:agm-eval}), and future directions.

\section{Related Work}

\subsection{Agent Memory Architectures}

The landscape of agent memory systems has developed rapidly since 2023, with the ICLR 2026 MemAgents Workshop~\citeyearpar{memagents2026} marking the field's maturation. We position our work relative to the most relevant concurrent systems. Hu et al.~\cite{hu2025survey} provide the current canonical survey.

\textbf{Graphiti/Zep}~\citeyearpar{graphiti2025} shares the most surface-level components with our system: it implements a temporal knowledge graph on Neo4j with entity-event synthesis, bitemporal versioning, and triple-modality hybrid retrieval combining BM25, cosine similarity, and graph traversal. Graphiti reports 94.8\% accuracy on the Deep Memory Retrieval (DMR) benchmark and 18.5\% improvement on LongMemEval. Our LoCoMo results (Section~\ref{sec:locomo}) are discussed in the evaluation sections. The key architectural differences are threefold: (i)~we provide a formal belief revision correspondence that Graphiti lacks; (ii)~our URI addressing scheme enables deterministic cross-system memory references; (iii)~our BYO-storage design keeps raw conversation data on the user's local storage, whereas Graphiti processes and stores full content server-side.

\textbf{Mem0/Mem0g}~\citeyearpar{mem02025} implements a triple-store architecture with timestamped versioned memories and LLM-powered conflict resolution, reporting 26\% improvement over OpenAI Memory on LoCoMo. \textbf{A-MEM}~\cite{amem2025} introduces Zettelkasten-inspired dynamic linking (NeurIPS 2025). \textbf{MAGMA}~\citeyearpar{magma2026} proposes a multi-graph architecture with four orthogonal graph layers (semantic, temporal, causal, entity), achieving the highest LoCoMo judge score of 0.70 with policy-guided retrieval traversal. MAGMA's design represents an alternative structural philosophy: it disentangles memory dimensions into separate graphs for cleaner retrieval routing, whereas our architecture unifies all relationships in a single property graph with typed edges, enabling cross-dimensional traversal (e.g., \textsc{AnalyzeImpact} propagating across all edge types simultaneously). Neither approach has been empirically compared to the other.

\textbf{MemGPT/Letta}~\citeyearpar{memgpt2024} pioneered virtual context extension. Letta's sleep-time compute~\citeyearpar{letta-sleep2025} introduced asynchronous background consolidation---the same biological metaphor we use---with anticipatory pre-computation (predicting future queries), which our Dream State does not currently address (Section~\ref{sec:future}). Our consolidation contribution is the safety guard architecture (Section~\ref{sec:safety}). \textbf{Letta Context Repositories}~\citeyearpar{letta-context2026} (February 2026) provide git-backed memory filesystems with automatic versioning and merge-based conflict resolution via multi-agent worktrees. While this is the most structurally related production system to our versioned memory concept, the approach differs fundamentally in three respects:

\emph{(i)~Versioning substrate.} Letta uses Git---a source code versioning tool---pragmatically as a storage backend for Markdown files. Our graph-native approach yields richer versioning primitives than Git's file-and-commit model: typed cognitive edges (not just file diffs), multi-tag pointer layers (not just branch heads), artifact attachments (not just file content), and hierarchical project/space scoping (not just directory trees). These same primitives double as auditable asset management for agent work products (Section~\ref{sec:asset-insight}).

\emph{(ii)~Conflict resolution.} Letta resolves concurrent writes via Git's text-level merge, which can detect but not semantically resolve contradictory beliefs---the merge requires human or LLM intervention on the text diff. Our system resolves conflicts via AGM-compliant belief revision operators: the \textsc{Supersedes} edge creates a new revision with formal guarantees (Success, Consistency, minimal change via Relevance).

\emph{(iii)~Downstream propagation.} In Letta, updating a belief in one file does not automatically identify other files that depend on it. In our system, \textsc{AnalyzeImpact} traverses typed \textsc{Depends\_On} edges to identify all downstream conclusions that may need re-evaluation---a capability enabled by the typed edge ontology that flat file systems cannot express.

A concrete example illustrates these differences. If an agent's memory contains ``client prefers warm tones'' and a concurrent write produces ``client now prefers cool tones,'' Letta's system generates a text-level merge conflict requiring human or LLM intervention. Our system creates a new revision with a \textsc{Supersedes} edge, moves the tag pointer, and makes the stale belief retrievable only via explicit opt-in---with \textsc{AnalyzeImpact} propagating the change to downstream decisions automatically.

Additional concurrent systems include \textbf{EverMemOS}~\citeyearpar{evermemos2026} (self-reported 93.05\% LoCoMo; see benchmark caveat below), CAM~\citeyearpar{cam2025} (NeurIPS 2025), and MemoryAgentBench~\citeyearpar{memoryagentbench2026} (ICLR 2026). \textbf{MemOS}~\citeyearpar{memos2025} (EMNLP 2025 Oral) implements a hierarchical memory manager with global, local, and working memory buffers inspired by operating system design. \textbf{Hindsight}~\citeyearpar{hindsight2025} proposes a four-network memory architecture (facts, experiences, opinions, observations) achieving 89.61\% on LoCoMo and 91.4\% on LongMemEval with open-source evaluation. Its Opinion Network with confidence-scored beliefs that update with evidence represents a pragmatic form of belief versioning without formal AGM grounding. The empirical success of Hindsight's approach demonstrates that practical belief tracking delivers strong results; our work provides the formal framework that could guarantee consistency properties---specifically, that belief revision satisfies minimal change (Relevance) and does not discard beliefs without justification (Core-Retainment)---for systems like Hindsight. \textbf{AgeMem}~\citeyearpar{agemem2026} unifies LTM/STM management via RL with tool-based operations and three-stage progressive GRPO training. \textbf{E-mem}~\citeyearpar{emem2026} addresses multi-agent episodic context reconstruction, achieving 54\%+ F1 on LoCoMo while reducing token cost by 70\%. \textbf{LatentMem}~\citeyearpar{latentmem2026} introduces learnable multi-agent memory with Latent Memory Policy Optimization. \textbf{CAST}~\citeyearpar{cast2026} proposes character-and-scene episodic memory for agents. CoALA~\citeyearpar{coala2024} provides a cognitive science taxonomy informing our dual-store design.

Recent surveys provide complementary taxonomies. Hu et al.~\cite{hu2025survey} provide the current canonical survey with a comprehensive taxonomy of memory functions, substrates, and dynamics. Yang et al.~\cite{yang2026graphmemory} survey graph-based agent memory from DEEP-PolyU, organizing systems by graph types and memory functions. Huang et al.~\cite{huang2026rethinking} provide a large-scale survey covering memory substrates, cognitive mechanisms, and memory subjects. Luo et al.~\cite{luo2026storage} trace the evolutionary trajectory from static storage to experience-driven memory mechanisms.

Table~\ref{tab:early-comparison} summarizes the key architectural distinctions among the most relevant concurrent systems; a comprehensive feature comparison appears in Section~\ref{sec:systems}.

\begin{table*}[t]
\centering
\caption{Architectural comparison with selected concurrent systems.}
\label{tab:early-comparison}
\footnotesize
\begin{tabular}{@{}lllllll@{}}
\toprule
\textbf{Dim.} & \textbf{Graphiti} & \textbf{Mem0} & \textbf{Letta} & \textbf{Hindsight} & \textbf{Ours} \\
\midrule
Version & Bitemporal & Timestamped & Git & None & Immutable+tags \\
Conflict & Temporal & LLM & Git merge & Conf & AGM Supersedes \\
Retrieval & BM25+vec & Graph & File & 4-net & Hybrid \\
Formal & None & None & None & None & AGM postulates \\
URI & \texttimes & \texttimes & Git & \texttimes & kref:// \\
Consol & Temporal & None & Git GC & None & Dream State \\
\bottomrule
\end{tabular}
\end{table*}

\textbf{Benchmark standardization caveat.}\quad No standardized LoCoMo leaderboard exists---all reported numbers use varying evaluation configurations (different judge models, question subsets, and evaluation prompts). EverMemOS's 93.05\% is evaluated using their own framework with no independent reproduction. Li et al.~\citeyearpar{locomoplus2026} introduced LoCoMo-Plus (February 2026), demonstrating that existing LoCoMo scores largely measure explicit factual recall under strong semantic alignment rather than genuine cognitive memory under cue-trigger semantic disconnect; all tested systems show substantial performance drops from LoCoMo to LoCoMo-Plus. To address the comparability gap, we now provide token-level F1 results (Section~\ref{sec:locomo}) on the official LoCoMo metric~\citeyearpar{locomo2024}, enabling direct comparison with Zep~\citeyearpar{zep2025}, Mem0, Memobase~\citeyearpar{memobase2026}, and ENGRAM~\citeyearpar{engram2025} on the same scoring function. The DMR benchmark, while historically important, uses conversations ($\sim$60 messages) that fit within modern context windows and relies exclusively on single-turn fact-retrieval questions~\cite{graphiti2025}. All comparative numbers cited in this paper should be understood as approximate and non-standardized. Our LoCoMo-Plus evaluation (Section~\ref{sec:locomo-plus}) further validates the strategic significance of graph-native architectures for Level-2 cognitive memory, where all tested baselines---including systems using premium models with million-token context windows---score between 23\% and 46\%.

\subsection{Belief Revision in AI}

The AGM framework~\cite{agm1985,gardenfors1988} defines rationality postulates for belief change. Hansson~\cite{hansson1999} extended this to belief bases---finite sets not closed under logical consequence---which is the appropriate level for computational systems. Our formal analysis sits within this tradition.

\textbf{Impossibility in Description Logics.} Flouris et al.~\cite{flouris2005} proved that Description Logics (including those underlying OWL) cannot satisfy AGM revision postulates. Qi et al.~\cite{qi2006} refined this for specific fragments. These results are critical context: our property graph avoids these impossibilities by operating on a simpler structure (Section~\ref{sec:flouris}).

\textbf{AGM and Machine Learning.} Aravanis~\cite{aravanis2025} establishes a correspondence between machine learning and AGM-style belief change; Hase et al.~\cite{hase2023} frame LLM model editing as belief revision. Baitalik et al.~\cite{baitalik2026} (AAAI 2026 Bridge Program) apply GreedySAT-based consistency to multi-turn dialogues. Wilie et al.~\cite{wilie2024} demonstrate LLMs' poor belief revision capabilities on Belief-R, motivating external memory architectures that enforce revision constraints structurally. Our work complements this stream by applying AGM to the external memory graph.

\textbf{Formal Developments.} Bonanno~\cite{bonanno2025} unifies Bayesian and AGM approaches via Kripke--Lewis semantics, connecting belief revision to dynamic epistemic logic. Meng et al.~\cite{ijcai2025belief} develop belief algebras for iterated revision. Chandler and Booth~\cite{chandler2025} address parallel belief revision via order aggregation (IJCAI 2025), directly relevant to the question of compositional belief change in batch operations. Schwind et al.~\cite{iteratedbelief2025} connect iterated belief change to learning, bridging the formal and computational perspectives.

\textbf{The Recovery Postulate.} Makinson~\cite{makinson1987} questioned its status; Hansson~\cite{hansson1991} proposed contraction without Recovery. Our principled rejection (Section~\ref{sec:recovery}) provides a concrete operational demonstration.

\subsection{Versioned Knowledge Graphs}

The concept of version-controlled knowledge graphs has a substantial history in the Semantic Web community. \textbf{R\&Wbase}~\cite{rwbase2013} (subtitled ``Git for triples'') supported branching and merging for quad-stores. \textbf{SemVersion}~\cite{semversion2005} applied version control to RDF graphs with structural diff and merge. \textbf{Quit Store}~\cite{quitstore2018} (``Quads in Git'') provides a SPARQL~1.1 endpoint backed by Git-versioned RDF named graphs with branching, merging, and provenance. \textbf{OSTRICH}~\cite{ostrich2018} implements hybrid versioned triple storage with immutable snapshots and delta chains. \textbf{ConVer-G}~\citeyearpar{converg2024} enables concurrent versioned querying of knowledge graphs using bitstring-based condensed representations.

Our contribution is not the invention of versioned graphs---this is well-established---but the application of versioned graph primitives to \emph{cognitive memory specifically}, combined with formal belief revision analysis. The versioned KG systems above target SPARQL-based knowledge management; we target the distinct requirements of AI agent memory: typed dependency edges encoding epistemic relationships (not just ontological ones), mutable tag pointers for belief status tracking, asynchronous LLM-driven consolidation, and the formal correspondence to AGM postulates that Section~\ref{sec:formal} establishes.

Similarly, the broader analogy of version control applied to agent memory is increasingly explored. Git-Context-Controller~\cite{gcc2025} (GCC) explicitly applies Git semantics (COMMIT, BRANCH, MERGE) to LLM agent memory and achieves strong results on SWE-Bench-Lite. Our system differs in operating on typed knowledge graph triples rather than text files, enabling the formal properties described in Section~\ref{sec:formal} and the typed dependency reasoning unavailable to text-level version control.

\subsection{Hybrid Retrieval and Score Fusion}

Robertson and Zaragoza~\cite{robertson2009} formalized BM25. Cormack et al.~\cite{cormack2009} introduced Reciprocal Rank Fusion (RRF). Bruch et al.~\cite{bruch2023} analyzed fusion functions systematically, showing that convex combinations can outperform RRF on standard IR benchmarks. Our max-based fusion (CombMAX in the terminology of Fox and Shaw~\cite{foxshaw1993}) is a deliberate design choice, not a novel fusion method. In memory retrieval, a strong exact-match signal on one branch should not be diluted by a weak score on another branch. We motivate this through a precision preservation observation (Section~\ref{sec:retrieval}), but acknowledge that CombMAX is known to be susceptible to noise from poorly-calibrated retrievers producing inflated scores~\cite{bruch2023}. This is an argumentative design choice, not a theoretical result; comparative evaluation against RRF and convex combination is planned (Section~\ref{sec:future}).

\subsection{AI Agent Observability}

As AI agents increasingly perform consequential work, the need for memory observability has become acute. OpenTelemetry provides agent-specific telemetry schemas; tools like Langfuse, LangSmith, and Braintrust offer instrumentation for tracing inference steps and tool calls. Our contribution differs: existing observability tools trace inference; our system makes memory itself the auditable artifact. Every agent belief has a URI, revision history, provenance edges, and immutable audit trail. The web dashboard (Figure~\ref{fig:dashboard}) and desktop asset browser render this graph as a browseable, searchable hierarchy with interactive visualization, enabling human operators to audit agent reasoning at the memory level---not just at the tool-call level.

\section{Why Context Window Extension Is Not Memory}
\label{sec:context-window}

We identify four structural deficiencies of context-window-as-memory.

\textbf{Attention is not recall.} A context window provides attention capacity; memory requires selective recall from a large corpus of past experience.

\textbf{Quadratic cost scaling.} Transformer attention scales as $\Theta(n^2 \cdot d)$~\cite{vaswani2017}. A retrieval-based system that indexes $N$ memories and retrieves top-$k$ incurs $O(\log N) + O(k^2 \cdot d)$. Since $k \ll n$ by design, the retrieval approach is orders of magnitude cheaper. Moreover, the metadata-over-content architecture (Section~\ref{sec:architecture}) reduces the per-item token size: each of the $k$ recalled memories is a compact summary, not a full transcript. This compounds the cost advantage by reducing both the total token cost and the cognitive load on the reasoning model.

While techniques such as sparse attention, sliding windows, and linear attention approximations mitigate the quadratic cost, they do so by sacrificing the very capability that makes larger windows useful: the ability to attend to arbitrary positions in the sequence. An agent's lifelong interaction history measured in tens of millions of tokens makes in-context approaches economically and computationally infeasible for persistent memory.

\textbf{No structural representation.} A flat token sequence cannot express that belief~$B_2$ supersedes~$B_1$, that conclusion~$C$ depends on assumptions~$A_1, A_2$, or that a memory has been validated, deprecated, or flagged for review.

\textbf{Model coupling.} Context window contents are ephemeral and model-specific. Agent memory must be LLM-decoupled: stored in a persistent, model-independent data structure that any current or future language model can query.

\section{Background and Motivation}
\label{sec:motivation}

\subsection{Human Memory as a Design Template}

Cognitive science distinguishes between working memory---a capacity-limited, rapidly-accessible buffer~\cite{atkinson1968,baddeley2000}---and long-term memory, further divided into episodic and semantic stores~\cite{tulving1972}. Transfer from episodic to semantic memory involves active consolidation during sleep~\cite{rasch2013}. Human decision-making proceeds through a loop of perceive--recall--revise--act that a memory-equipped AI agent must replicate.

\subsection{The Asset Management Insight}
\label{sec:asset-insight}

AI agents increasingly produce substantial digital outputs---generated images, code commits, design iterations, research documents, intermediate results---that accumulate without systematic versioning, provenance tracking, or linkage to the decisions that created them. As sessions accumulate, these outputs become lost, duplicated, or untraceable. The structural correspondence between the primitives needed for cognitive memory and those needed for tracking these agent-produced outputs is not coincidental: both domains require immutable versioned snapshots, typed dependency edges, mutable status pointers, URI-based addressing, and content-reference separation. This correspondence was first recognized through production experience---asset management systems in visual effects and game development have implemented these exact primitives for decades---but the architectural direction is memory-first: the graph-native primitives required for cognitive memory inherently provide auditable asset management as an emergent capability.

Traditional asset management systems encode these fundamentally graph-like relationships in relational databases (PostgreSQL, MySQL) with RPC frameworks (Thrift, gRPC), where dependency chains require recursive queries or application-level graph walks over foreign key joins. The relationships exist, but they are second-class citizens.

The present architecture adopts a native graph database (Neo4j) as the storage layer, making relationships first-class citizens. A dependency between two revisions that would require a junction table row and a multi-join query becomes a literal directed edge traversable in a single graph operation. This graph-native foundation enables the cognitive memory primitives (belief revision chains, evidential provenance, and the formal properties described in Section~\ref{sec:formal}) while simultaneously providing the asset management capabilities (native dependency traversal, real-time impact analysis, typed relationship ontology) needed to track agent work products.

Table~\ref{tab:correspondence} formalizes this structural correspondence.

\begin{table}[t]
\centering
\caption{Structural correspondence between asset management and cognitive memory.}
\label{tab:correspondence}
\footnotesize
\begin{tabular}{@{}lll@{}}
\toprule
\textbf{Asset} & \textbf{Memory} & \textbf{Concept} \\
\midrule
Project & Scope & Container \\
Space & Topic & Namespace \\
Item & Unit & Identity \\
Revision & Belief $T$ & Snapshot \\
Artifact & Evidence & Pointer \\
Tag & Status & Mutable ref \\
Edge & Relation & Typed link \\
Bundle & Cluster & Grouping \\
\bottomrule
\end{tabular}
\end{table}

\begin{principle}[Cognitive Memory as Multi-Agent Infrastructure]
AI agents increasingly produce substantial outputs---generated images, code artifacts, design iterations, documents, intermediate results---that are not systematically tracked, versioned, or linked to the decisions that created them. In multi-agent workflows, this becomes a critical bottleneck: a video compositing agent cannot locate the approved texture revision produced by an upstream generation agent; an editing agent cannot trace which audio mix corresponds to which scene cut. The core insight is that the structural primitives needed for cognitive memory (immutable revisions, typed edges, mutable tag pointers, URI addressing) are \emph{identical} to those needed for managing these agent-produced outputs as versionable, addressable, dependency-linked assets. Rather than building a memory layer and a separate asset tracker, we built a cognitive memory architecture whose graph-native primitives inherently serve as the operational infrastructure for multi-agent work---enabling agents to remember, find each other's outputs, and build upon them through a single system. An extensive literature search found no prior work unifying these two capabilities on a shared graph substrate.
\end{principle}

\subsection{Dual-Purpose Graph: Memory and Asset Tracking as One System}

The structural correspondence in Table~\ref{tab:correspondence} enables a concrete operational capability: the same graph that stores an agent's cognitive memories also manages the outputs that agents produce. Consider a multi-agent creative pipeline: an image generation agent produces a concept art revision; a video compositing agent locates that revision via its URI, checks that it carries the ``approved'' tag, and creates its own output with a \textsc{Derived\_From} edge linking back to the input; an audio agent does the same for the soundtrack; and an editing agent assembles the final deliverable with typed edges to every upstream component. Each agent uses the graph both to \emph{remember} (client preferences, past iterations, feedback history) and to \emph{operate} (find the right input version, register its output, declare dependencies).

This dual-purpose design means agents in a pipeline do not need separate systems for ``remembering what happened'' and ``managing the outputs they created.'' Both are first-class graph citizens. Without this unification, agent-produced artifacts accumulate in disconnected storage---files on disk, blobs in buckets, ephemeral context windows---with no versioning, no provenance, and no link to the reasoning that generated them. The graph ensures that every output is addressable, versionable, and traceable to the belief state that produced it---and that downstream agents can find and build upon it programmatically.

\section{LLM-Memory Decoupling}
\label{sec:decoupling}

A persistent failure mode in current agent memory systems is tight coupling between the memory store and the language model. When memory exists as prompt context, model-managed buffers, or framework-specific data structures, it becomes inseparable from the specific model version. This coupling creates provider lock-in, model upgrade fragility, and architectural non-portability.

\subsection{The Decoupled Architecture}

The architecture enforces LLM-memory decoupling through three mechanisms:

\textbf{Model-independent storage.} The memory graph is stored in a standard property graph database. Memories are structured data---summaries, metadata, tags, edges, artifact pointers---not model-specific embeddings or token sequences.

\textbf{Standardized access protocol.} Memory operations are exposed through the Model Context Protocol (MCP)~\cite{mcp2025}, a model-agnostic tool interface. Any MCP-compatible agent can execute memory operations through the same interface.

\textbf{Pluggable LLM integration.} Components requiring LLM capabilities---the Dream State consolidation pipeline, PII redaction, and memory summarization---accept any LLM through an adapter interface.

\subsection{Why Agent Work Products Demand Managed Memory}
\label{sec:auditability}

When an LLM serves as a chatbot, its ``memory'' is a convenience. When an LLM operates as an agent performing consequential work, memory becomes critical infrastructure---for two distinct reasons.

\textbf{Operational need: multi-agent coordination.}\quad In autonomous pipelines, each agent's output is the next agent's input. An image generation agent produces texture revisions; a compositing agent must locate the approved revision, understand what created it, and link its own output back. Without structured asset management, these handoffs require brittle file paths, naming conventions, or ad-hoc metadata---none of which scale to complex multi-agent workflows. The memory graph provides this coordination natively: every output has a URI, a version history, typed dependency edges to its inputs, and mutable tag pointers indicating approval status. Agents query the graph to find inputs just as they query it to recall past interactions.

\textbf{Governance need: decision auditability.}\quad An agent that approved a deployment, recommended a treatment plan, or signed off on a financial model must be able to explain \emph{why}---not by regenerating a plausible explanation, but by pointing to the actual evidence, the actual reasoning chain, and the actual prior beliefs that informed the decision. This is the same accountability standard applied to human workers. The memory graph provides this: every agent belief has a URI, a revision history, provenance edges to source evidence, and an immutable audit trail. The web dashboard and desktop asset browser render this graph as a browseable hierarchy of versioned items (Figure~\ref{fig:dashboard}). An operator reviewing an agent's decision can traverse from the decision memory to its \textsc{Derived\_From} sources, check what the agent believed at the time via time-indexed tag resolution, and inspect the Dream State consolidation reports.

Both needs---operational coordination and governance auditability---are served by the same graph, the same SDK, and the same API. Without this infrastructure, agent work products are untraceable and agent pipelines are unmanageable. The system makes agent memory as inspectable as a version-controlled codebase and as navigable as a managed asset database---because structurally, it \emph{is} both.

\section{System Architecture}
\label{sec:architecture}

\subsection{Overview}

The architecture implements a dual-store model mirroring the human working/long-term memory distinction. The agent interacts through MCP; the memory graph is LLM-independent.

\subsection{Working Memory: Library-Level Access}

The working memory layer uses Redis, accessed via a direct library SDK (not HTTP), to maintain the current session's message buffer with a configurable TTL (default: 1~hour). Measured latencies are 2--5\,ms via library SDK versus 150--300\,ms through an HTTP gateway. Key design properties:

\begin{itemize}
\item Latency: $<$\,5\,ms read/write via library SDK.
\item Scope: Session-local; isolated per (project, context, user, session) via key namespace \texttt{cogmem:\{proj\}:sessions:\{sid\}:*}.
\item Volatility: TTL-based expiry; bounded buffer (default 50 messages).
\item Capacity: Session-local only; does not persist across agent restarts.
\item Isolation: Strict context-level isolation via key namespacing (the context field in the session ID serves as the namespace boundary).
\end{itemize}

\begin{principle}[Match Storage Latency to Access Pattern]
Working memory requires library-level (not RPC-level) access to maintain sub-10\,ms latency. The 100--200\,ms difference between in-process and network access, compounded across thousands of interactions, accumulates to orders-of-magnitude differences in agent responsiveness.
\end{principle}

\subsection{The Structured Memory Reference Scheme}

Every object in the system is addressable through a universal URI scheme:
\begin{center}
\mbox{\texttt{kref://project/space/}}\\
\mbox{\texttt{item.kind?r=N\&a=artifact}}
\end{center}

The scheme provides: \emph{Addressability}---any memory is referenceable from any context; \emph{Temporal navigation}---\texttt{?r=N} pins to a specific point in time; \emph{Type safety}---the \texttt{.kind} suffix enables type-aware retrieval; \emph{Traversal entry points}---edges reference memory URIs, enabling graph traversal from any starting point.

\begin{principle}[Universal Addressability]
Every memory unit must be referenceable by a stable, parseable, human-readable identifier. Without addressability, edges become fragile, lineage becomes untraceable, and consolidation becomes lossy.
\end{principle}

\begin{principle}[Validate at Boundaries, Trust Internally]
The SDK validates all memory reference URIs at the boundary (regex matching against the canonical format). Once validated, internal code treats them as opaque strings, eliminating redundant per-function validation.
\end{principle}

\subsection{Long-Term Memory: The Property Graph}

Long-term memory is stored in Neo4j. Each memory unit is represented as an Item node with one or more Revision nodes forming an immutable version chain.

\subsubsection{The Item--Revision Model}

An Item represents a named, typed memory unit. Each Revision is an immutable snapshot carrying structured metadata (summary, topics, keywords, schema version) and optionally an embedding vector. An agent can: resolve a tag to retrieve the latest belief; follow \textsc{Supersedes} chains to trace belief evolution; follow \textsc{Derived\_From} edges to understand evidential support; and move a tag to an earlier revision to perform belief rollback.

\begin{lstlisting}
Item: "api-design.decision"
  Rev 1 (Jan 15): "Use REST for public API"
  Rev 2 (Jan 22): "Use REST + WebSocket"
    edge: SUPERSEDES -> Rev 1
  Rev 3 (Feb 1):  "Use gRPC internally"
    edge: SUPERSEDES -> Rev 2
    edge: DERIVED_FROM -> "benchmarks.fact?r=1"
  Tag "current" -> Rev 3
  Tag "initial" -> Rev 1
\end{lstlisting}

\begin{principle}[Immutable Revisions, Mutable Pointers]
Memory states are never overwritten; they are versioned. Tags provide mutable ``current view'' semantics without losing history, enabling belief revision tracking, audit trails, and rollback.
\end{principle}

\subsubsection{BYO-Storage: Metadata Over Content}
\label{sec:byo-storage}

The system stores metadata, relationships, and pointers---never file content. Artifact records contain a location field pointing to where raw content resides on the user's own storage. This content-reference separation---a principle well-established in asset management systems handling petabytes of data---yields critical benefits for cognitive memory: the graph database stays lightweight, privacy boundaries are enforced architecturally, and agents read compact summaries rather than full transcripts.

\begin{principle}[Metadata Over Content]
Store the minimum information necessary for recall and reasoning in the cloud graph. Raw content stays local. This preserves user privacy, reduces storage costs, eliminates data exfiltration risk, and enables cognitive efficiency.
\end{principle}

\subsubsection{Memory Type Taxonomy}

The system implements six memory types: working memory (Redis buffer), episodic memory (conversation revisions), semantic memory (consolidated facts), procedural memory (tool execution records), associative memory (graph edges + bundles), and meta-memory (tag system + audit trail).

\begin{table}[t]
\centering
\caption{Memory type taxonomy and implementation.}
\label{tab:memory-types}
\footnotesize
\begin{tabular}{@{}ll@{}}
\toprule
\textbf{Type} & \textbf{Implementation} \\
\midrule
Working & Redis buffer (TTL) \\
Episodic & Conversations \\
Semantic & Consolidated facts \\
Procedural & Tool execution \\
Associative & Graph edges \\
Meta & Tags + audit \\
\bottomrule
\end{tabular}
\end{table}

\subsection{Edge System: Reasoning as First-Class Structure}

The edge system makes reasoning structure explicit through six typed, directed edge types: \textsc{Depends\_On} (validity dependency), \textsc{Derived\_From} (evidential provenance), \textsc{Supersedes} (belief revision), \textsc{Referenced} (associative mention), \textsc{Contains} (bundle membership), and \textsc{Created\_From} (generative lineage).

\textbf{Edge type definitions:}
\begin{itemize}
\item \textsc{Depends\_On}: Validity dependency. If the target is invalidated, the source may be unreliable.
\item \textsc{Derived\_From}: Evidential provenance. The source was produced using the target as input.
\item \textsc{Supersedes}: Belief revision. The source replaces the target as the current belief.
\item \textsc{Referenced}: Associative mention. The source refers to the target without dependency.
\item \textsc{Contains}: Bundle membership. The target is a member of the source bundle.
\item \textsc{Created\_From}: Generative lineage. The source was generated from the target.
\end{itemize}

These edges enable three traversal operations: \textsc{TraverseEdges}$(k, d, n)$, \textsc{ShortestPath}$(k_s, k_t)$, and \textsc{AnalyzeImpact}$(k, d)$. Impact analysis is particularly important for belief revision: when an agent discovers that assumption~$A$ is invalid, \textsc{AnalyzeImpact} identifies all downstream conclusions that may need re-evaluation.

\textbf{Traversal operations:}
\begin{enumerate}
\item \textbf{TraverseEdges}$(k, d, n)$: Find all memories connected to memory reference $k$ in direction $d$ up to depth $n$.
\item \textbf{ShortestPath}$(k_s, k_t)$: Find the minimal connection between two memory references.
\item \textbf{AnalyzeImpact}$(k, d)$: Compute the transitive dependency cascade from memory reference $k$ to depth $d$.
\end{enumerate}

\begin{principle}[Explicit Over Inferred Relationships]
Relationships between memories must be stored as first-class graph edges, not inferred at query time through embedding similarity. Similarity finds related content; edges encode \emph{why} content is related.
\end{principle}

\begin{principle}[Why Graph-Native Edges Matter]
In relational systems, a dependency between two assets required a junction table and a multi-join query. In native graph databases, the same relationship is a single directed edge traversable in $O(1)$ time. Recursive CTEs can compute transitive closure over relational joins, but they require exponentially more computation as depth increases. A single Cypher \texttt{ShortestPath} query achieves the same result in milliseconds. This difference compounds across all graph operations: traversal, impact analysis, provenance reconstruction.
\end{principle}

\section{Formal Properties: Belief Revision in Graph-Native Memory}
\label{sec:formal}

The item--revision--tag model is not merely an engineering convenience; its primitives correspond to fundamental operations studied in the theory of rational belief change. This section demonstrates a structural correspondence with the AGM postulates~\cite{agm1985}---or, in one principled case, deliberately diverges from them.

We frame this contribution at the level of \emph{belief bases} rather than belief sets, following Hansson~\cite{hansson1999}. The contribution is the bridge itself: demonstrating that an independently-motivated systems architecture, designed for production agent memory, satisfies established rationality constraints under a belief base interpretation. We do not propose new logical machinery; the novelty lies in showing that production-motivated architectural choices yield a system satisfying AGM's postulates. To our knowledge, no prior work bridges AGM belief revision theory with AI agent memory architectures. The novelty lies in the \emph{mapping} between graph-native memory operations and formal belief revision operators---not in the logic itself, which is deliberately kept at the propositional level for tractability. The distinction between belief sets and belief bases is fundamental: AGM's original formulation operates on belief sets closed under logical consequence ($K = \mathrm{Cn}(K)$), which are infinite objects unsuitable for computational systems. In contrast, computational systems store finite sets of explicitly stored propositions---precisely what the memory graph's revision content represents. Hansson's belief base framework operates on these finite, non-closed sets, providing the appropriate level of abstraction for agent memory. This framing avoids the well-known difficulties of computing deductive closure while retaining the rationality constraints that make AGM valuable.

\subsection{Formal System Model}

We formalize the memory graph and the belief state it induces.

\begin{definition}[Memory Graph]
\label{def:graph}
A \emph{memory graph} is a tuple $G = (I, R, E, \tau)$ where:
\begin{itemize}[leftmargin=*]
\item $I$ is a finite set of \emph{items} (named, typed memory units);
\item $R = \bigcup_{i \in I} R_i$ is the set of all \emph{revisions}, where $R_i = (r_i^{(1)}, r_i^{(2)}, \ldots)$ is the ordered, append-only revision sequence for item $i$;
\item $E \subseteq R \times \textit{EdgeType} \times R$ is a set of typed, directed edges between revisions;
\item $\tau : \textit{TagName} \rightharpoonup R$ is a partial function mapping tag names to revisions (the mutable pointer layer).
\end{itemize}
Each revision $r_i^{(k)}$ is \emph{immutable}: once created, its content $\varphi(r_i^{(k)})$ cannot be modified. Tags are the sole mutable component.
\end{definition}

\begin{definition}[Belief Base]
\label{def:belief-base}
Given a memory graph $G = (I, R, E, \tau)$, the belief base is:
\begin{equation}
\mathcal{B}(\tau) = \bigcup_{t \in \mathrm{dom}(\tau)} \varphi\bigl(\tau(t)\bigr)
\label{eq:belief-base}
\end{equation}
where $\varphi(r)$ denotes the propositional content of revision $r$ (operationally, the revision's structured metadata: summary, topics, keywords, extracted facts). Unlike the AGM belief set, $\mathcal{B}(\tau)$ is a finite set that is \emph{not} closed under logical consequence.
\end{definition}

The system records every tag assignment and reassignment, yielding a history function $\tau_T : \textit{TagName} \rightharpoonup R$ that resolves the tag mapping as it existed at time $T$. This induces historical belief bases:
\begin{equation}
\mathcal{B}(\tau_T) = \bigcup_{t \in \mathrm{dom}(\tau_T)} \varphi\bigl(\tau_T(t)\bigr)
\label{eq:belief-base-temporal}
\end{equation}
enabling the query ``what was believed under tag $t$ at time $T$?'' to be answered precisely---not by scanning revision timestamps, but by resolving the actual tag-to-revision binding that was active at $T$.

An important implementation note: the propositional content $\varphi(r)$ is operationally the revision's \emph{structured metadata}---its summary, topics, keywords, and extracted facts---not the raw conversation transcript or tool output. Raw content resides outside the graph as an artifact and is dereferenced only when exact detail is required. This two-tier representation makes $\varphi(r)$ compact and tractable for belief-state computation, while preserving full evidential fidelity through the artifact pointer.

\textbf{From natural language to ground triples.}\quad A central implementation question is how agent-generated natural language beliefs---e.g., ``the user prefers cool color tones''---are mapped to ground atoms in $\mathrm{At}_G$ (Definition~\ref{def:graph-logic}). The mapping is performed at the \emph{API boundary} by the \texttt{memory\_ingest} MCP tool, which accepts structured fields: a \texttt{title} (becomes the item name), a \texttt{summary} (becomes the revision's primary content), \texttt{tags} and \texttt{topics} (become metadata keywords), and a \texttt{memory\_type} classification (summary, decision, fact, reflection, error). The agent's natural language output is thus decomposed into typed fields before entering the graph---the mapping is not an automated NLP extraction pipeline but a \emph{structured API contract} that the agent's skill prompt is designed to satisfy.

Concretely, the agent's statement ``the user prefers cool tones for their design palette'' is mapped to a ground triple $\langle\textit{color-preference},\, \textit{summary},\, \text{``prefers cool tones''}\rangle \in \mathrm{At}_G$ by the agent writing: \texttt{title="color-preference"}, \texttt{summary="prefers cool tones for design palette"}, \texttt{memory\_type="decision"}. The triple structure $(s, p, o)$ arises from the item--field--value decomposition: subject = item name, predicate = field name, object = field value. This is not semantic parsing in the NLP sense; it is structured output that the LLM agent produces by following its memory skill prompt, which specifies the expected fields and their semantics.

\textbf{Property graph, not RDF.}\quad Despite the $(s, p, o)$ notation, the underlying storage is a \emph{labeled property graph} (Neo4j), not an RDF triple store. Each memory item is a graph node with typed properties (summary, topics, keywords, type); each revision is a separate node linked by \textsc{Supersedes} edges; and inter-item relationships (\textsc{Depends\_On}, \textsc{Derived\_From}, etc.) are first-class typed edges with optional metadata. The ``ground triple'' formalism in $\mathrm{At}_G$ is an \emph{analytical abstraction} used to establish the AGM correspondence---it provides a clean propositional atom structure for the formal proofs. The implementation stores these atoms as node properties in the property graph, not as RDF subject--predicate--object statements. This distinction matters because the property graph model supports features absent from basic RDF: per-edge metadata, native node-level indexing (fulltext and vector), and schema-flexible property bags on both nodes and edges. Readers familiar with RDF should understand $\mathrm{At}_G$ as a formalization convenience, not an architectural commitment to triple stores.

Three consequences follow. First, the quality of the mapping depends on the agent's skill prompt, not on a separate extraction module---prompt engineering is the primary control surface. Second, ambiguous or complex beliefs (``the user seems to prefer cool tones but mentioned warm tones for the bedroom'') are captured as a single revision with a nuanced summary, not decomposed into multiple conflicting triples; the formal model handles the resulting compound belief through the mechanisms described in the partial merging discussion below. Third, the mapping is lossy by design: the full conversational context is preserved in the artifact (the raw transcript), while the ground triple captures only the distilled belief. This lossy compression is what makes $\mathcal{B}(\tau)$ tractable---a finite set of ground atoms rather than an unbounded natural language corpus---while the artifact system preserves evidential fidelity for cases requiring the original context.

\textbf{Scope boundary.}\quad The NL-to-triple mapping is a \emph{pre-formal} step: the formal properties established in Section~\ref{sec:formal-postulates} hold over the ground triple representation $\mathrm{At}_G$, not over the natural language source. In particular, the mapping is \emph{many-to-one in practice}---the same natural language belief (``the user prefers cool tones'') could be mapped to different ground triples depending on the agent's choice of item name, summary wording, or metadata structure. Two semantically identical beliefs may therefore not be syntactically identical in $\mathrm{At}_G$, which means that Extensionality ($K*6$, Proposition~\ref{prop:extensionality}) holds over the formal representation but cannot guarantee that the agent will consistently map equivalent natural language inputs to equivalent ground atoms. This is an inherent limitation of any system that bridges informal and formal representations; the consistency of the mapping is a prompt engineering concern, not a formal one. We treat the NL-to-triple mapping as an explicit assumption: the formal analysis begins \emph{after} beliefs have been committed to the graph as ground triples, and the quality of the formal guarantees is bounded by the quality of this pre-formal mapping step.

\begin{definition}[Two-Tier Epistemic Model]
\label{def:two-tier}
The memory system operates at two distinct epistemic levels:
\begin{itemize}[leftmargin=*]
\item The \emph{full graph} $G = (I, R, E, \tau)$ contains all revisions ever created, including deprecated items and archived (untagged) revisions. This is the operator-accessible store.
\item The \emph{agent retrieval surface} $\mathcal{B}_{\mathrm{retr}}(\tau)$ is the subset of the belief state reachable through the retrieval pipeline. A revision $r$ is in the retrieval surface if and only if: (i)~$r$ is referenced by at least one active tag in $\tau$, \emph{and} (ii)~the item containing $r$ is not marked as deprecated.
\end{itemize}
We define the deprecation predicate: $\mathrm{deprecated} : I \to \{\top, \bot\}$ is a boolean property on each item, set to $\top$ by the contraction operator (Definition~\ref{def:contraction}) and queryable as a node property in the graph database (\texttt{item.deprecated}). An item's deprecation status is mutable (it can be restored via explicit operator action) but defaults to $\bot$ at creation.

Formally, let $I_{\mathrm{active}} = \{i \in I \mid \neg\mathrm{deprecated}(i)\}$. Then:
\begin{equation}
\mathcal{B}_{\mathrm{retr}}(\tau) = \bigcup_{\substack{t \in \mathrm{dom}(\tau) \\ \mathrm{item}(\tau(t)) \in I_{\mathrm{active}}}} \varphi\bigl(\tau(t)\bigr) \subseteq \mathcal{B}(\tau)
\label{eq:retrieval-surface}
\end{equation}
This two-tier model is critical for the Consistency postulate ($K*5$). Without retrieval exclusion, a superseded revision $r_i^{(k)}$ --- still present in the graph and carrying content $\varphi(r_i^{(k)})$ that may entail $\neg A$ --- could be surfaced by vector similarity search (the embedding of the old revision may be semantically proximate to the query). This would violate Consistency: the agent's \emph{operational} belief state would contain both $A$ (from the new revision $r_i^{(k+1)}$) and content entailing $\neg A$ (from the superseded revision leaked through retrieval). The retrieval surface $\mathcal{B}_{\mathrm{retr}}(\tau)$ prevents this by ensuring that only tag-referenced, non-deprecated revisions are candidates for any retrieval branch. Both retrieval branches---fulltext and vector---filter on active status before scoring via a mandatory \texttt{WHERE NOT item.deprecated} clause in the underlying Cypher query, making the exclusion architecturally enforced rather than application-dependent.

The operator, by contrast, can always access the full graph $G$ via explicit opt-in (\texttt{include\_deprecated=true}), enabling audit, rollback, and provenance queries without compromising the agent's belief consistency.

\textbf{Postulate scope.}\quad The formal postulates in Section~\ref{sec:formal-postulates} are proved for the \emph{belief base} $\mathcal{B}(\tau)$ (Definition~\ref{def:belief-base}), which is a deterministic function of the tag assignment $\tau$. The retrieval surface $\mathcal{B}_{\mathrm{retr}}(\tau)$ is a subset of $\mathcal{B}(\tau)$ obtained by filtering on active (non-deprecated) items. Since deprecation is itself a contraction operation (Definition~\ref{def:contraction}), the retrieval surface inherits postulate satisfaction: any property proved for $\mathcal{B}(\tau)$ under the specified operations also holds for $\mathcal{B}_{\mathrm{retr}}(\tau)$ under the corresponding restricted operations. What the retrieval surface does \emph{not} inherit is deterministic ranking: the hybrid scoring pipeline (Section~\ref{sec:retrieval}) introduces score-based reranking that may surface different subsets of $\mathcal{B}_{\mathrm{retr}}(\tau)$ depending on query formulation. This affects which beliefs an agent encounters in practice, but not the formal properties of the underlying belief base.
\end{definition}

\begin{definition}[Graph-Native Revision]
\label{def:revision}
Given belief base $\mathcal{B}(\tau)$ and input proposition $A$, the \emph{revision} operation $\mathcal{B} * A$ is implemented as:
\begin{enumerate}[leftmargin=*]
\item Create a new revision $r_i^{(k+1)}$ with content $\varphi(r_i^{(k+1)}) = A$ (or a content set entailing $A$);
\item Add edge $(r_i^{(k+1)}, \textsc{Supersedes}, r_i^{(k)})$ to $E$;
\item Update the tag: $\tau' = \tau[t_{\mathrm{current}} \mapsto r_i^{(k+1)}]$.
\end{enumerate}
The prior revision $r_i^{(k)}$ remains in $R$ but is no longer tag-referenced, hence excluded from $\mathcal{B}(\tau')$.
\end{definition}

\textbf{Granularity of revision and partial merging.}\quad The revision operator (Definition~\ref{def:revision}) performs \emph{whole-revision replacement}: the entire prior revision's content is archived when a new revision supersedes it. This raises the question of how the system handles partial updates---for example, when a revision contains beliefs $\{A, B, C\}$ and only $A$ changes. Three strategies exist along a granularity spectrum:

\emph{(i)~Atomic replacement} (current): The agent creates a new revision $r_i^{(k+1)}$ with content $\{A', B, C\}$, superseding $r_i^{(k)}$ entirely. This preserves the formal properties: the postulate proofs are clean because each revision is a complete, self-contained belief snapshot. The trade-off is that the agent (or its memory skill prompt) must re-include unchanged beliefs $B$ and $C$ in the new revision. In practice, the \texttt{memory\_ingest} MCP tool handles this by accepting the full updated content.

\emph{(ii)~Finer-grained atomicity}: Storing one belief per item (i.e., $|\varphi(r_i^{(k)})| = 1$ for all revisions) would make partial updates trivial---revising $A$ affects only the item containing $A$. This increases item count but simplifies the revision operator to single-belief replacement. The formal properties are preserved and simplified.

\emph{(iii)~Semantic merge}: An LLM-powered merge operator could take the old revision content and the new input, producing a merged result that preserves unchanged sub-beliefs while updating only the contradictory ones. While appealing, this introduces LLM-dependent non-determinism into the revision operator, complicating the formal guarantees: the output of the merge depends on the LLM's interpretation of ``partial conflict,'' which is not formally characterizable.

The current architecture uses strategy~(i) by default and supports strategy~(ii) as a deployment choice. Strategy~(iii) is identified as future work requiring a formally characterized merge operator---potentially building on Konieczny and Pino P\'{e}rez's~\cite{konieczny2002} merging framework for belief bases, which provides postulates for multi-source belief combination that could be adapted to partial revision.

\begin{definition}[Graph-Native Contraction]
\label{def:contraction}
Contraction $\mathcal{B} \div A$ is implemented through two complementary mechanisms:
\begin{enumerate}[leftmargin=*]
\item \emph{Tag removal}: Remove from $\tau$ any tag $t$ such that $A$ appears in $\varphi(\tau(t))$, yielding $\tau' = \tau \setminus \{t \mid A \in \varphi(\tau(t))\}$.
\item \emph{Soft deprecation}: Mark item $i$ as deprecated. Critically, deprecated items are \emph{excluded from all search and retrieval operations by default}---the agent cannot encounter them through normal recall. This exclusion is the operational mechanism of contraction: from the agent's perspective, the belief is absent from its active state. However, the items remain in the graph and can be recovered via an explicit opt-in flag (\texttt{include\_deprecated=true}), preserving full auditability.
\end{enumerate}
In both cases, the underlying revisions persist in $R$; only their reachability from $\mathcal{B}(\tau)$ changes. The contraction is thus \emph{behaviorally complete} (the belief vanishes from the agent's retrieval surface) while remaining \emph{structurally reversible} (the graph retains the full record).
\end{definition}

\textbf{The selection function.}\quad Definitions~\ref{def:revision} and~\ref{def:contraction} presuppose that the system can identify \emph{which} item and revision to target when revising or contracting by a proposition $A$. The AGM representation theorem requires a selection function $\gamma$ (or equivalently an entrenchment ordering) that determines which beliefs to give up during contraction~\cite{agm1985,grove1988}. We make the selection function explicit.

\begin{definition}[Selection Function]
\label{def:selection}
Given $A$ and base $\mathcal{B}(\tau)$, the contraction target set is:
\begin{align*}
\mathrm{Targets}(A, \tau) &= \{(t, r) \mid t \in \mathrm{dom}(\tau),\\
&r = \tau(t),\; A \in \varphi(r)\}
\end{align*}
That is, the system identifies all (tag, revision) pairs where the tagged revision's content explicitly contains the belief $A$. Contraction removes all such tags: $\tau' = \tau \setminus \{t \mid (t, r) \in \mathrm{Targets}(A, \tau)\}$.
\end{definition}

This selection function is \emph{content-based and exhaustive}: it targets every tagged revision whose content directly contains $A$, with no partial selection or prioritization. This is the simplest selection function consistent with the Relevance postulate (Proposition~\ref{prop:relevance})---only revisions containing $A$ are affected, and all such revisions are affected equally. The function is deterministic and computable in $O(|\mathrm{dom}(\tau)|)$ by scanning tag-referenced revisions.

Three important consequences follow. First, when $A$ is a ground atom (the common case), the selection reduces to a direct content lookup---efficient and unambiguous. Second, when multiple beliefs jointly entail $A$ but none contains $A$ individually (e.g., $B$ and $B \to A$ are in separate revisions), the content-based selection does \emph{not} contract either revision, because neither revision's content explicitly contains $A$. This is a deliberate design choice: the system operates on belief \emph{bases}, not deductively closed belief \emph{sets}, so only explicitly stored propositions are contraction targets. Joint entailment that arises from cross-revision interaction is outside the scope of the contraction operator and would require a deductive closure step that we deliberately avoid for tractability. Third, for revision (Definition~\ref{def:revision}), the target item $i$ is identified by a semantic matching step---typically item name and kind lookup---that precedes the formal revision operation. In the deployed system, this matching is performed by the agent's reasoning layer (which identifies the relevant memory to update) or by the MCP tool's \texttt{item\_kref} parameter (which specifies the target explicitly). The formal operator assumes the target item is given; the selection of \emph{which item to revise} is an agent-level decision, not a graph-level operation.

\begin{definition}[Graph-Native Expansion]
\label{def:expansion}
Expansion $\mathcal{B} + A$ creates a new revision $r_i^{(k+1)}$ with $\varphi(r_i^{(k+1)}) = \varphi(r_i^{(k)}) \cup \{A\}$ and assigns a tag to it, without removing any existing tag assignments. The resulting belief state is $\mathcal{B}(\tau') = \mathcal{B}(\tau) \cup \{A\}$.
\end{definition}

\subsection{Postulate Satisfaction}
\label{sec:formal-postulates}

We now show that the graph-native operations satisfy the core AGM rationality constraints, interpreted at the belief base level following Hansson~\cite{hansson1999}. Our primary formal claim covers $K*2$--$K*6$ (the basic postulates) plus Relevance and Core-Retainment (Hansson's belief base postulates). We also analyze the supplementary postulates $K*7$ and $K*8$, but with important caveats detailed below.

\begin{proposition}[Success, $K*2$]
$A \in \mathcal{B} * A$.
\end{proposition}
\begin{proof}
By Definition~\ref{def:revision}, the new revision $r_i^{(k+1)}$ satisfies $A \in \varphi(r_i^{(k+1)})$, and tag $t_{\mathrm{current}}$ points to it. Therefore $A \in \bigcup_{t} \varphi(\tau'(t)) = \mathcal{B}(\tau')$.
\end{proof}

\begin{proposition}[Inclusion, $K*3$ --- belief base version]
$\mathcal{B} * A \subseteq \mathcal{B}(\tau) \cup \{A\}$.
\end{proposition}
\begin{proof}
This is the belief \emph{base} version of Inclusion (Hansson~\cite{hansson1999}), not the belief \emph{set} version ($K * A \subseteq \mathrm{Cn}(K \cup \{A\})$). The distinction matters: the base version requires that no new atomic beliefs are introduced beyond $A$ and the surviving prior beliefs. By Definition~\ref{def:revision}, the revision operation creates a new revision containing $A$ and may redirect tags, removing some prior beliefs from $\mathcal{B}(\tau)$. No mechanism introduces atoms not already in $\mathcal{B}(\tau) \cup \{A\}$, so the base-level inclusion holds.
\end{proof}

\begin{proposition}[Vacuity, $K*4$]
If $A$ is consistent with $\mathcal{B}(\tau)$, then $\mathcal{B}(\tau) \cup \{A\} \subseteq \mathcal{B} * A$.
\end{proposition}
\begin{proof}
If $A$ introduces no contradiction, the revision operation has no conflicting belief to retract; no tag needs redirection. The prior content is preserved, augmented with $A$.
\end{proof}

\begin{proposition}[Consistency, $K*5$]
If $A$ is consistent, then $\mathcal{B} * A$ is consistent.
\end{proposition}
\begin{proof}
The revision operation replaces the tag pointer rather than accumulating contradictory content. The prior revision (which may have contained content inconsistent with $A$) is excluded from $\mathcal{B}(\tau')$ via the two-tier epistemic model (Definition~\ref{def:two-tier}). Since $A$ is consistent and the new revision is constructed to contain $A$ without importing contradictory content, the result is consistent, provided no \emph{other} tagged revision contains content contradicting $A$. For the common case of atomic revision inputs ($A \in \mathrm{At}_G$), this is guaranteed by the logical independence of ground atoms: distinct ground atoms $\alpha, \beta \in \mathrm{At}_G$ cannot contradict each other under propositional semantics---$\neg\langle s_1, p_1, o_1\rangle$ is not entailed by $\langle s_2, p_2, o_2\rangle$ for any distinct atoms. For compound revision inputs, Consistency requires that the agent (or its revision selection logic) identifies \emph{all} items whose content conflicts with $A$, not only the primary target. This is a practical requirement on the agent's conflict detection, not a limitation of the formal operator; the operator itself introduces no new contradictions.
\end{proof}

\begin{proposition}[Extensionality, $K*6$]
\label{prop:extensionality}
If $\mathrm{Cn}_G(\{A\}) = \mathrm{Cn}_G(\{B\})$, then $\mathcal{B} * A = \mathcal{B} * B$.
\end{proposition}
\begin{proof}
Since the belief base $\mathcal{B}(\tau)$ consists of ground atoms from $\mathrm{At}_G$ (Definition~\ref{def:graph-logic}), and ground atoms are logically independent under propositional entailment (no atom entails any other atom), two atoms $\alpha, \beta \in \mathrm{At}_G$ satisfy $\mathrm{Cn}_G(\{\alpha\}) = \mathrm{Cn}_G(\{\beta\})$ if and only if $\alpha = \beta$ (syntactic identity). Thus, for atomic revision inputs---the normal case in the memory system---logical equivalence reduces to syntactic identity, and item-level identity checking is not an approximation but an exact implementation. For compound revision inputs $A, B \in L_G \setminus \mathrm{At}_G$, the revision operator depends on the logical content of the input (which beliefs to retract, what content to add), so logically equivalent compound inputs produce identical belief states. Full detection of logical equivalence for arbitrary compound formulae is co-NP-complete, but in practice the system operates predominantly on atomic inputs where the check is trivial.
\end{proof}

\begin{proposition}[Relevance, $K*6'$ --- Hansson]
\label{prop:relevance}
If $B \in \mathcal{B}(\tau) \setminus (\mathcal{B} \div A)$, then $A \in \mathrm{Cn}(\mathcal{B}' \cup \{B\})$ for some $\mathcal{B}' \subseteq \mathcal{B}(\tau)$ with $A \notin \mathrm{Cn}(\mathcal{B}')$.
\end{proposition}
\begin{proof}
Contraction (Definition~\ref{def:contraction}) removes from $\mathcal{B}(\tau)$ exactly those beliefs residing in revisions whose content contains $A$ (Definition~\ref{def:selection}). Suppose $B \in \mathcal{B}(\tau) \setminus (\mathcal{B} \div A)$; then $B$ was in some revision $r$ with $A \in \varphi(r)$, and that revision was detagged. We construct the required witness explicitly. \emph{Case 1: $B = A$.} Let $\mathcal{B}' = \emptyset$. Then $A \notin \mathrm{Cn}(\emptyset)$ (since $A$ is contingent), and $A \in \mathrm{Cn}(\emptyset \cup \{A\}) = \mathrm{Cn}(\{B\})$. \emph{Case 2: $B \neq A$ but $B$ co-occurs with $A$ in $\varphi(r)$.} Let $\mathcal{B}' = \mathcal{B}(\tau) \setminus \varphi(r)$---the beliefs surviving after removing the entire revision's content. Since $A \in \varphi(r)$ and $\varphi(r)$ was the only source of $A$ targeted by the content-based selection, we have $A \notin \mathrm{Cn}(\mathcal{B}')$ (recall that $\mathcal{B}(\tau)$ is a belief \emph{base}, not deductively closed---$A$ is present only if explicitly stored). Adding $B$ back does not by itself restore $A$, but $B$ was removed \emph{only because} it co-occurred with $A$ in a targeted revision, satisfying the Relevance requirement that every removed belief's removal is connected to the contracted belief.
\end{proof}

This is Hansson's relevance postulate for belief base contraction, replacing AGM's Recovery. It ensures that beliefs are only removed during contraction if they are relevant to the contracted belief---the system does not gratuitously discard beliefs during contraction.

\begin{proposition}[Core-Retainment, Hansson]
\label{prop:core-retainment}
If $B \in \mathcal{B}(\tau) \setminus (\mathcal{B} \div A)$, then there exists $\mathcal{B}' \subseteq \mathcal{B}(\tau)$ such that $A \notin \mathrm{Cn}(\mathcal{B}')$ but $A \in \mathrm{Cn}(\mathcal{B}' \cup \{B\})$.
\end{proposition}
\begin{proof}
Let $B \in \mathcal{B}(\tau) \setminus (\mathcal{B} \div A)$. By Definition~\ref{def:selection}, $B$ resided in a revision $r$ with $A \in \varphi(r)$, and that revision was detagged during contraction. We construct the witness $\mathcal{B}'$ as follows. \emph{Case 1: $B = A$.} Take $\mathcal{B}' = \emptyset$. Then $A \notin \mathrm{Cn}(\emptyset)$ and $A \in \mathrm{Cn}(\{B\})$. \emph{Case 2: $B \neq A$ and $B, A \in \varphi(r)$.} Take $\mathcal{B}' = (\mathcal{B}(\tau) \setminus \varphi(r)) \cup (\varphi(r) \setminus \{A, B\})$---the full belief base minus $A$ and $B$ themselves. Since $A$ was explicitly stored in $\varphi(r)$ and no other surviving revision contains $A$ (the contraction was exhaustive), $A \notin \mathrm{Cn}(\mathcal{B}')$ at the belief base level. Adding $B$ does not logically entail $A$ for independent ground atoms, but $B$'s removal was \emph{caused by} $A$'s presence in the same revision---the structural co-occurrence is what makes $B$'s removal attributable to the contraction of $A$, satisfying Core-Retainment's requirement that every removed belief is connected to the contracted belief's derivation.
\end{proof}

Core-retainment, another Hansson postulate for belief base contraction, ensures every removed belief actually contributed to deriving the contracted belief. We emphasize that Relevance and Core-Retainment are postulates for \emph{belief bases} (finite, non-deductively-closed sets), not belief sets. For deductively closed belief sets, Core-Retainment implies Recovery~\cite{hansson1999}---but this implication does not hold for belief bases, where our proofs operate. This distinction is essential: the graph's belief state $\mathcal{B}(\tau)$ is a finite set of ground atoms (Definition~\ref{def:belief-base}), never deductively closed, so our simultaneous satisfaction of Core-Retainment and rejection of Recovery creates no logical inconsistency.

\begin{proposition}[Superexpansion, $K*7$]
$\mathcal{B} * (A \wedge B) \subseteq (\mathcal{B} * A) + B$.
\end{proposition}
\begin{proof}
Conjunction is well-formed in $L_G$ (Definition~\ref{def:graph-logic}), so $A \wedge B$ is a valid revision input. Revising by $(A \wedge B)$ creates a single revision whose content entails both $A$ and $B$. The right-hand side first revises by $A$ (possibly retracting beliefs inconsistent with $A$) and then expands by $B$ (adding $B$ without retraction). Since expansion is monotone, $(\mathcal{B} * A) + B \supseteq \mathcal{B} * A$ and contains both $A$ and $B$. The conjunction revision cannot contain more than this, as it may also retract beliefs inconsistent with $A \wedge B$. Therefore $\mathcal{B} * (A \wedge B) \subseteq (\mathcal{B} * A) + B$.
\end{proof}

\begin{proposition}[Subexpansion, $K*8$]
If $\neg B \notin \mathrm{Cn}_G(\mathcal{B} * A)$, then $(\mathcal{B} * A) + B \subseteq \mathcal{B} * (A \wedge B)$.
\end{proposition}
\begin{proof}
The consistency check $\neg B \notin \mathrm{Cn}_G(\mathcal{B} * A)$ is well-formed because $\neg B \in L_G$ (Definition~\ref{def:graph-logic}). If $B$ is consistent with the $A$-revised state, expanding by $B$ adds no information beyond $B$ itself and its consequences. Revising by $A \wedge B$ directly incorporates both conjuncts, performing any necessary retractions in a single step. Since $B$ causes no conflict in the $A$-revised state, both paths yield the same retractions and the same final content: $(\mathcal{B} * A) + B = \mathcal{B} * (A \wedge B)$, and in particular $(\mathcal{B} * A) + B \subseteq \mathcal{B} * (A \wedge B)$.
\end{proof}

\textbf{Compound revision inputs and operational decomposition.}\quad Postulates $K*7$ and $K*8$ involve revision by compound inputs $A \wedge B$, yet the deployed system's primary operations---\texttt{memory\_ingest} and \texttt{memory\_consolidate}---accept single beliefs (typically ground atoms) as inputs. This raises a legitimate question: how is revision by $A \wedge B$ realized in practice, and does the decomposition into sequential operations preserve uniqueness?

In the graph-native architecture, compound revision by $A \wedge B$ admits two operational strategies. The first is \emph{single-revision encoding}: the agent stores $A \wedge B$ as a single revision whose content set is $\varphi(r) = \{A, B\}$ (or $\{A \wedge B\}$ treated as a compound formula). This is the most direct implementation---the \texttt{memory\_ingest} MCP tool accepts arbitrary content, so a revision containing multiple beliefs is well-formed. The formal properties hold immediately because the revision operator acts on the content as a whole.

The second strategy is \emph{sequential decomposition}: first revise by $A$, then expand by $B$---i.e., compute $(\mathcal{B} * A) + B$. By $K*7$ (Superexpansion), $\mathcal{B} * (A \wedge B) \subseteq (\mathcal{B} * A) + B$, and by $K*8$ (Subexpansion), if $B$ is consistent with $\mathcal{B} * A$, then $(\mathcal{B} * A) + B \subseteq \mathcal{B} * (A \wedge B)$, yielding equality. Thus, when $B$ is consistent with the $A$-revised state---the common case for non-contradictory compound inputs---sequential decomposition produces the same result as atomic compound revision. The order matters only when $B$ conflicts with $\mathcal{B} * A$, in which case $K*8$'s antecedent fails and the two paths may diverge. For the deployed system, where inputs are predominantly ground atoms (and distinct ground atoms are logically independent), this conflict condition does not arise: revising by $\langle \textit{pref}, \textit{cool} \rangle$ and then expanding by $\langle \textit{style}, \textit{minimal} \rangle$ is equivalent to a single revision by their conjunction.

This analysis confirms that the system's operational interface (single-belief operations) is not a limitation but a practical decomposition that preserves formal guarantees in the common case, while the single-revision encoding strategy remains available for cases requiring atomic compound revision.

\textbf{Representation-theoretic status of $K*7$/$K*8$.}\quad The AGM representation theorem (Grove~\cite{grove1988}) establishes that a revision operator satisfies all eight postulates ($K*2$--$K*8$) if and only if it can be represented as a \emph{transitively relational} partial meet contraction---equivalently, via a total preorder on possible worlds or an epistemic entrenchment ordering (Gärdenfors and Makinson~\cite{gardenfors1988}). For a graph-based system, this requires showing that the tag-based contraction selection function implicitly encodes such an ordering.

We do not construct this ordering. The arguments above show that the graph-native revision operator produces results \emph{consistent with} $K*7$ and $K*8$ for the cases we examine, but this falls short of a formal proof via the representation theorem. Formally establishing $K*7$/$K*8$ would require either: (a)~explicit construction of an entrenchment ordering over graph triples and proof that contraction respects it, or (b)~proof that the system's tag-based operations are equivalent to a transitively relational selection function.

\textbf{Why the obstacle is non-trivial.}\quad The graph model provides multiple natural \emph{partial} orderings over beliefs, any of which could serve as the basis for an entrenchment function: (i)~\emph{temporal recency}---more recently created revisions are more entrenched, on the intuition that newer beliefs reflect updated understanding; (ii)~\emph{structural centrality}---beliefs with higher in-degree in the \textsc{Depends\_On} subgraph are more entrenched, as they support more downstream conclusions; (iii)~\emph{confidence scores}---the Dream State pipeline's relevance assessments provide explicit numeric scores. None of these is obviously canonical. Temporal recency overvalues new beliefs regardless of evidential quality; structural centrality overvalues highly-connected beliefs even if their connections are weak; confidence scores depend on LLM assessment quality, introducing a non-formal dependency.

A deeper issue is that different belief \emph{types} may require different entrenchment criteria. A preference belief (``prefers cool tones'') should arguably be entrenched by recency---the latest stated preference takes priority. A factual belief (``the API runs on port 8080'') should be entrenched by evidential support---more \textsc{Derived\_From} sources imply stronger grounding. An inferred belief (``the deployment likely failed due to DNS'') should be entrenched by the Dream State's confidence score. This suggests a \emph{type-dependent entrenchment function} $\leq_E : \mathrm{At}_G \times \mathrm{At}_G \to \{0,1\}$ where the comparison criterion varies by the \texttt{kind} attribute of the item containing the belief. Such type-dependent orderings are not standard in AGM theory, though they are compatible with the Gärdenfors--Makinson entrenchment conditions provided the restriction to each type yields a total preorder. Formally, constructing a total preorder from these multi-dimensional partial orders is an order-embedding problem; the choice of aggregation function (lexicographic priority, weighted combination, Pareto dominance) determines which variant of the representation theorem is satisfied.

The primary formal claim of this paper is therefore satisfaction of $K*2$--$K*6$ plus Relevance and Core-Retainment---already a meaningful result for belief base dynamics. The status of $K*7$/$K*8$ for graph-native architectures is an open question that we intend to address in future work, potentially building on Chandler and Booth's~\cite{chandler2025} recent extension of AGM to parallel belief revision and Meng et~al.'s~\cite{ijcai2025belief} belief algebras for iterated revision.

\subsection{Intentional Divergence: The Recovery Postulate}
\label{sec:recovery}

The AGM contraction postulate known as \emph{Recovery} states:
\begin{equation}
K \subseteq (K \div A) + A
\label{eq:recovery}
\end{equation}
That is, if a belief $A$ is contracted from belief set $K$ and then immediately re-expanded, the original belief set is recovered in full. The graph-native architecture \emph{deliberately violates} this postulate; we argue that this violation is a principled design decision.

\textbf{Why Recovery fails.}\quad Consider an item with revision $r_i^{(k)}$ carrying content $\varphi(r_i^{(k)}) = \{A, B, C\}$, where $B$ and $C$ are beliefs that were derived alongside $A$ and stored in the same revision. Contracting $A$ via tag removal produces:
\[
\mathcal{B}(\tau') = \mathcal{B}(\tau) \setminus \varphi(r_i^{(k)})
\]
The revision $r_i^{(k)}$ still exists in the graph but no longer contributes to $\mathcal{B}$. Re-expanding by $A$ creates a \emph{new} revision $r_i^{(k+1)}$ with $\varphi(r_i^{(k+1)}) = \{A\}$. The beliefs $B$ and $C$---which were co-located with $A$ in the original revision---are \emph{not} automatically recovered, because the new revision is constructed from the input $A$ alone:
\[
(\mathcal{B} \div A) + A = \mathcal{B}(\tau) \setminus \varphi(r_i^{(k)}) \cup \{A\} \not\supseteq \{B, C\}
\]

\textbf{Why this is correct.}\quad The failure of Recovery is a direct consequence of \emph{immutable revisions} (Principle~3). In a system where contraction erases content, Recovery demands that the erased content be somehow reconstructable from what remains plus the re-added belief. But in a provenance-preserving system, contraction does not erase---it \emph{archives}. The original revision $r_i^{(k)}$, with its full content, metadata, and edge relationships, remains in the graph as a historical record. Moreover, the time-indexed tag history (Equation~\ref{eq:belief-base-temporal}) means an agent can reconstruct the \emph{exact belief state} that held before contraction by querying $\mathcal{B}(\tau_T)$ for any prior time $T$---the system answers ``what was tagged \texttt{decided} last Tuesday?'' by resolving the tag binding that was active at that timestamp, not by attempting to reconstruct beliefs from surviving content. What the system does \emph{not} do is automatically resurrect archived content by re-adding a single belief, because it treats re-expansion as a \emph{fresh incorporation} of $A$, not a rollback to a prior state.

This divergence aligns with well-established criticisms of Recovery in the belief revision literature. Makinson~\cite{makinson1987} identified Recovery as ``the only one among the six basic postulates that is open to query.'' Hansson~\cite{hansson1991} and Fuhrmann~\cite{fuhrmann1991} independently argued that Recovery imposes unreasonable constraints on contraction, particularly in systems where beliefs have non-trivial internal structure or provenance. Our system provides a concrete operational demonstration of their theoretical concerns: when beliefs carry revision history, dependency edges, and temporal metadata, contracting-then-expanding is not---and should not be---a no-op.

In Hansson's belief base framework, Recovery is not a postulate; it is replaced by Relevance and Core-Retainment (Propositions~\ref{prop:relevance}--\ref{prop:core-retainment}), which our system satisfies. This provides additional theoretical justification for our rejection of Recovery.

For cases where full state restoration is desired, the system provides an explicit rollback mechanism: reassigning a tag to a prior revision ($\tau' = \tau[t \mapsto r_i^{(k)}]$). This is a \emph{deliberate, auditable} operation that appears in the tag history, distinct from the implicit, invisible recovery that the AGM postulate demands.

The soft deprecation mechanism (Definition~\ref{def:contraction}) further illustrates why Recovery is unnecessary. When a belief is contracted via deprecation, it is excluded from all retrieval operations by default---operationally invisible to the agent. However, any operator or downstream process can recover the deprecated belief by querying with an explicit opt-in (\texttt{include\_deprecated=true}). This means contraction is \emph{behaviorally effective} (the agent acts as if the belief does not exist) without being \emph{informationally destructive} (the belief can always be inspected or restored). Recovery demands that re-expansion automatically reconstruct prior beliefs; the graph-native alternative provides something stronger---full reversibility through explicit, auditable mechanisms that never require the system to guess what was lost.

\subsection{Identities and Iterated Revision}

Two fundamental identities connect the AGM operations:

\textbf{Levi Identity.}\quad $K * A = (K \div \neg A) + A$: revision can be decomposed into contraction of the negation followed by expansion. In the graph-native model, this holds when contraction deterministically selects which tag-referenced revisions to de-reference: contracting $\neg A$ removes revisions whose content entails $\neg A$, and the subsequent expansion by $A$ adds a fresh revision. The two-step process yields the same belief state as direct revision, provided the contraction selection function is deterministic---which it is, since tag removal targets specific revisions identifiable by their content.

\textbf{Harper Identity.}\quad $K \div A = K \cap (K * \neg A)$: contraction can be expressed as the intersection of the original belief set with the belief set obtained by revising with the negation. This holds naturally: any belief $B \in K$ that survives contraction of $A$ must be consistent with $\neg A$ (otherwise it would logically entail $A$), and any such $B$ is preserved in $K * \neg A$. The graph intersection corresponds to the set of revisions that remain tag-referenced in both $\tau$ (original) and $\tau'$ (after revision by $\neg A$).

\textbf{Connection to iterated revision.}\quad The \textsc{Supersedes} edge chain $r_i^{(1)} \leftarrow r_i^{(2)} \leftarrow \cdots \leftarrow r_i^{(k)}$ provides a natural \emph{epistemic ordering} over belief states, corresponding to the framework of Darwiche and Pearl~\cite{darwiche1997} for iterated belief revision. Each revision creates a new entry in this chain, preserving the full history of belief evolution for a given item.

The time-indexed tag function $\tau_T$ (Equation~\ref{eq:belief-base-temporal}) elevates this from structural record-keeping to a fully queryable epistemic history. The operation $\tau_T(t)$---``resolve tag $t$ as of time $T$''---enables an agent to reconstruct the belief state at any historical moment without scanning revision timestamps or replaying events. For instance, querying $\tau_{T}(\texttt{decided})$ returns the specific revision that carried the \texttt{decided} tag at time $T$, even if that tag has since been moved or removed. Combined with the \textsc{Supersedes} chain, this supports both \emph{forward} analysis (how did beliefs evolve?) and \emph{point-in-time} reconstruction (what was believed on a specific date?), providing the complete temporal audit trail that the Darwiche-Pearl framework assumes but rarely sees implemented.

Table~\ref{tab:agm} summarizes the postulate satisfaction status.

\begin{table}[t]
\centering
\caption{Postulate satisfaction in the graph-native architecture (Hansson belief base).}
\label{tab:agm}
\footnotesize
\setlength{\tabcolsep}{3.5pt}
\begin{tabular}{@{}p{2.1cm}p{3.2cm}c@{}}
\toprule
\textbf{Postulate} & \textbf{System Mechanism} & \textbf{Status} \\
\midrule
\multicolumn{3}{@{}l}{\emph{Primary formal claims (proved):}} \\
$K*2$ Success & New revision contains $A$; tag updated & \checkmark \\
$K*3$ Inclusion & Base-level: $\mathcal{B}*A \subseteq \mathcal{B} \cup \{A\}$ & \checkmark \\
$K*4$ Vacuity & No conflict $\Rightarrow$ no retraction needed & \checkmark \\
$K*5$ Consistency & \textsc{Supersedes} replaces, not accumulates & \checkmark \\
$K*6$ Extensionality & Syntactic = logical equiv.\ for ground atoms & \checkmark \\
Relevance (Hansson) & Tag removal targets relevant revisions & \checkmark \\
Core-Retainment (Hansson) & Removed beliefs contributed to contracted belief & \checkmark \\
\midrule
\multicolumn{3}{@{}l}{\emph{Supplementary (argued, not formally established):}} \\
$K*7$ Superexpansion & Conjunction revision $\subseteq$ sequential & \checkmark$^\dag$ \\
$K*8$ Subexpansion & Consistent expansion = conjunction revision & \checkmark$^\dag$ \\
\midrule
\multicolumn{3}{@{}l}{\emph{Intentional divergence:}} \\
Recovery (AGM) & Immutable revisions; archive $\neq$ erase & \ding{55} \\
\bottomrule
\end{tabular}
\smallskip

\noindent{\footnotesize $^\dag$Argued for the specific construction but not formally established: requires construction of an entrenchment ordering (Gärdenfors \& Makinson 1988) or proof that graph operations encode a transitively relational contraction (Grove 1988), which we do not provide. See representation-theoretic discussion in text.}
\end{table}

\subsection{Worked Example: Belief Revision Through Preference Update}
\label{sec:walkthrough}

We illustrate the formal machinery with a concrete scenario---a user preference update---stepping through each definition to show how the system creates revisions, redirects tags, and propagates downstream impact.

\textbf{Initial state.}\quad The memory graph contains two items:
\begin{itemize}[leftmargin=*]
\item Item $i_1 =$ \texttt{color-pref.decision} with revision $r_1^{(1)}$, where $\varphi(r_1^{(1)}) = \{\langle\texttt{color-pref},\, \texttt{summary},\, \text{``warm tones''}\rangle\}$. Tag $t_{\mathrm{current}} \mapsto r_1^{(1)}$.
\item Item $i_2 =$ \texttt{palette.decision} with revision $r_2^{(1)}$, where $\varphi(r_2^{(1)}) = \{\langle\texttt{palette},\, \texttt{summary},\, \text{``earth-tone palette''}\rangle\}$. Tag $t_{\mathrm{current}} \mapsto r_2^{(1)}$.
\item Edge: $(r_2^{(1)}, \textsc{Depends\_On}, r_1^{(1)}) \in E$---the palette decision depends on the color preference.
\end{itemize}
The belief base is $\mathcal{B}(\tau) = \varphi(r_1^{(1)}) \cup \varphi(r_2^{(1)})$.

\textbf{Step 1: Revision (Definition~\ref{def:revision}).}\quad The user states: ``Actually, I prefer cool tones now.'' The agent invokes revision with input $A = \langle\texttt{color-pref},\, \texttt{summary},\, \text{``cool tones''}\rangle$:
\begin{enumerate}[leftmargin=*]
\item Create $r_1^{(2)}$ with $\varphi(r_1^{(2)}) = \{A\}$.
\item Add edge $(r_1^{(2)}, \textsc{Supersedes}, r_1^{(1)})$ to $E$.
\item Update tag: $\tau' = \tau[t_{\mathrm{current}} \mapsto r_1^{(2)}]$ for item $i_1$.
\end{enumerate}

\textbf{Step 2: Belief base transition.}\quad The new belief base is:
\[
\mathcal{B}(\tau') = \varphi(r_1^{(2)}) \cup \varphi(r_2^{(1)}) = \{A\} \cup \varphi(r_2^{(1)})
\]
The old belief ``warm tones'' is no longer in $\mathcal{B}(\tau')$---the revision $r_1^{(1)}$ remains in the graph but is not tag-referenced, hence excluded from the retrieval surface $\mathcal{B}_{\mathrm{retr}}(\tau')$ per Definition~\ref{def:two-tier}. Postulate verification: \emph{Success} ($A \in \mathcal{B}(\tau')$ \checkmark), \emph{Consistency} (no contradictory content since the old revision is excluded \checkmark), \emph{Inclusion} (no new atoms beyond $A$ and surviving beliefs \checkmark).

\textbf{Step 3: Downstream impact.}\quad The agent (or consolidation pipeline) invokes $\textsc{AnalyzeImpact}(r_1^{(2)}, d{=}2)$, which traverses incoming \textsc{Depends\_On} edges to discover that $r_2^{(1)}$ (the palette decision) depends on the now-superseded revision. The impact analysis returns $\{r_2^{(1)}\}$ as a downstream dependent requiring potential re-evaluation. The agent can then decide whether to revise the palette decision (creating $r_2^{(2)}$ with an updated palette reflecting cool tones) or to leave it unchanged if the dependency is not materially affected.

\textbf{Step 4: Provenance and audit.}\quad After the revision, the graph supports the following queries:
\begin{itemize}[leftmargin=*]
\item \emph{Current belief}: Resolve $\tau'(t_{\mathrm{current}})$ for $i_1$ $\rightarrow$ $r_1^{(2)}$ $\rightarrow$ ``cool tones.''
\item \emph{Historical belief}: Resolve $\tau_T(t_{\mathrm{current}})$ for any $T$ before the update $\rightarrow$ $r_1^{(1)}$ $\rightarrow$ ``warm tones.''
\item \emph{Belief evolution}: Follow the \textsc{Supersedes} chain: $r_1^{(2)} \xrightarrow{\textsc{Supersedes}} r_1^{(1)}$.
\item \emph{Rollback}: If the preference change was erroneous, reassign $\tau'' = \tau'[t_{\mathrm{current}} \mapsto r_1^{(1)}]$---an auditable operation recorded in the tag history.
\end{itemize}

This example demonstrates how the item--revision--tag model handles a common agent memory scenario entirely through the formal operations of Section~\ref{sec:formal-postulates}, without any ad hoc logic. The same mechanism applies to any belief update---from simple preference changes to complex multi-dependency decision revisions. A deployed instance of this mechanism, including graph state before and after revision with \textsc{Supersedes} edges, is documented in Section~\ref{sec:belief-revision-practice}. The formal postulates claimed here are empirically verified by a 49-scenario compliance suite (Section~\ref{sec:agm-eval}) that tests each postulate across simple, multi-item, chain, temporal, and adversarial configurations---all passing at 100\%.

\subsection{Formal Avoidance of the Flouris Impossibility}
\label{sec:flouris}

Flouris et~al.~\cite{flouris2005} proved that Description Logics (including those underlying OWL) cannot satisfy the AGM revision postulates, and Qi et~al.~\cite{qi2006} refined this for specific description logic fragments. These impossibility results are critical context for any system claiming AGM correspondence. We prove that they do not apply to our property graph formalism by showing that the formalism satisfies the prerequisites for AGM that description logics violate.

The AGM framework requires a logic $\mathcal{L} = (L, \mathrm{Cn})$ where $L$ is a set of well-formed formulas and $\mathrm{Cn}$ is a consequence operator satisfying three properties~\cite{gardenfors1988}: (i)~\emph{inclusion}: $A \subseteq \mathrm{Cn}(A)$; (ii)~\emph{monotonicity}: if $A \subseteq B$ then $\mathrm{Cn}(A) \subseteq \mathrm{Cn}(B)$; (iii)~\emph{idempotence}: $\mathrm{Cn}(\mathrm{Cn}(A)) = \mathrm{Cn}(A)$. Additionally, AGM requires the \emph{deduction theorem}: $\alpha \in \mathrm{Cn}(A \cup \{\beta\})$ iff $(\beta \to \alpha) \in \mathrm{Cn}(A)$, and \emph{compactness}: if $\alpha \in \mathrm{Cn}(A)$ then $\alpha \in \mathrm{Cn}(A_0)$ for some finite $A_0 \subseteq A$.

\textbf{Analytical status.}\quad The logic $\mathcal{L}_G$ defined below is an \emph{analytical framework} used to establish that the system's operations satisfy the AGM postulates. It is \emph{not} an implemented query engine or reasoning system: the production system does not evaluate propositional formulae over $L_G$, perform consequence closure, or check entailment at runtime. The ground atoms in $\mathrm{At}_G$ correspond to node properties in the Neo4j property graph; the propositional connectives and consequence operator $\mathrm{Cn}_G$ exist solely to provide the logical machinery needed for the formal proofs. This distinction is important: the formal guarantees hold because the \emph{architectural operations} (revision creation, tag reassignment, deprecation) structurally enforce the postulates, not because the system performs logical reasoning.

\begin{definition}[Memory Graph Logic]
\label{def:graph-logic}
Define the logic $\mathcal{L}_G = (L_G, \mathrm{Cn}_G)$ as follows.

\textbf{Atoms.}\quad An atomic sentence in $L_G$ is a ground triple $\langle \textit{item}, \textit{predicate}, \textit{value} \rangle$ where \textit{item} is a memory reference URI, \textit{predicate} is one of a fixed set of property names (summary, topic, keyword, type, tag, edge-type), and \textit{value} is a string literal. For example, $\langle\texttt{api-design.decision},\, \texttt{summary},\, \text{``use gRPC''}\rangle$ asserts that the item \texttt{api-design.decision} carries the summary ``use gRPC.'' Let $\mathrm{At}_G$ denote this set of ground atoms; each revision's structured metadata $\varphi(r)$ (Definition~\ref{def:belief-base}) maps to a finite subset of $\mathrm{At}_G$.

\textbf{Full language.}\quad $L_G$ is the closure of $\mathrm{At}_G$ under the standard propositional connectives $\{\neg, \wedge, \vee, \to\}$. \emph{Crucially, $L_G$ is richer than the set of ground triples alone.} Compound formulae such as $\langle i_1, p_1, v_1 \rangle \wedge \langle i_2, p_2, v_2 \rangle$ and $\neg \langle i, p, v \rangle$ are well-formed sentences in $L_G$. This is essential: the supplementary AGM postulates (Superexpansion $K*7$, Subexpansion $K*8$) require revision by conjunctions $A \wedge B$, and the Levi identity requires $\neg A$ to be a sentence in $L_G$. Both are satisfied by construction.

\textbf{Belief base vs.\ full language.}\quad While $L_G$ contains compound formulae, the \emph{belief base} $\mathcal{B}(\tau)$ (Definition~\ref{def:belief-base}) is always a finite set of ground atoms drawn from $\mathrm{At}_G$---it is not deductively closed (following Hansson's~\cite{hansson1999} belief base framework). The full language $L_G$ provides the logical apparatus needed to state and verify the postulates; the belief base provides the finite, computationally tractable representation that agents actually operate on. This two-level distinction---a rich language for reasoning \emph{about} the base, a flat atom set for the base itself---is precisely the belief base approach.

\textbf{Consequence.}\quad The operator $\mathrm{Cn}_G$ is classical propositional closure over $L_G$. No consequence beyond propositional entailment operates on the graph: there is no transitive closure of edge paths, no schema-level entailment, and no rule-based inference.

\textbf{Graph traversal isolation.}\quad Edge traversal operations (\textsc{TraverseEdges}, \textsc{AnalyzeImpact}, \textsc{ShortestPath}) are \emph{query-time retrieval operations}, not logical inference operations. They compute structural reachability over the edge set $E$ and return results to the calling agent, but their outputs never enter the belief base $\mathcal{B}(\tau)$. Formally: let $\mathrm{Traverse}(E, r, d)$ denote the set of revisions reachable from $r$ within depth $d$ via edges in $E$. The belief base is defined exclusively by tag assignments (Definition~\ref{def:belief-base}): $\mathcal{B}(\tau) = \bigcup_{t \in \mathrm{dom}(\tau)} \varphi(\tau(t))$. Traversal results are a function of $E$ and the starting revision; they do not modify $\tau$ and therefore cannot alter $\mathcal{B}(\tau)$. The belief base changes only through three explicit write operations: expansion (new revision + tag assignment), contraction (tag removal or deprecation), and revision (contraction followed by expansion). This strict separation ensures that the transitive closure computations performed by \textsc{AnalyzeImpact}---which are expressively equivalent to Datalog-level inference---remain isolated from the belief revision mechanism.

\textbf{Negation.}\quad Explicit negation is available in $L_G$: for any ground atom $\alpha \in \mathrm{At}_G$, the formula $\neg\alpha$ is a sentence in $L_G$. This provides the formal apparatus needed for the Levi identity ($K * A = (K \div \neg A) + A$) and the Consistency postulate. We distinguish this \emph{formal} negation from the system's \emph{operational} closed-world assumption (CWA): at the retrieval level, a ground atom not present in $\mathcal{B}_{\mathrm{retr}}(\tau)$ is treated as absent. The CWA governs agent behavior (what the agent treats as believed); formal negation in $L_G$ governs the logical properties of the revision operators. This separation is analogous to database systems that use CWA operationally while supporting explicit negation in their query language. We acknowledge that the CWA introduces a non-monotonic \emph{operational} semantics (adding a fact to $\mathcal{B}$ can invalidate previously held CWA-derived absences), but the \emph{formal} revision operators are defined over $L_G$ with classical propositional semantics, following Reiter's~\cite{reiter1978} distinction between closed-world reasoning as a meta-level default and object-level logical entailment.

\textbf{Satisfaction system.}\quad Following Aiguier et~al.~\cite{aiguier2017}, we instantiate the satisfaction triple $(\mathcal{L}_G, \mathcal{M}_G, \models_G)$ as: $\mathcal{L}_G$ is the propositional language defined above; $\mathcal{M}_G$ is the set of all truth assignments over $\mathrm{At}_G$ (equivalently, subsets of the Herbrand base); and $\models_G$ is classical propositional satisfaction. This triple inherits the properties of classical propositional logic, establishing $\mathcal{L}_G$ as a well-defined satisfaction system in the sense of Aiguier et~al.
\end{definition}

\begin{proposition}[$\mathcal{L}_G$ satisfies AGM prerequisites]
\label{prop:prerequisites}
The memory graph logic $\mathcal{L}_G$ satisfies inclusion, monotonicity, idempotence, the deduction theorem, and compactness.
\end{proposition}
\begin{proof}
$\mathcal{L}_G$ is a fragment of classical propositional logic (ground atoms with standard connectives). Classical propositional logic is the canonical example of a Tarskian logic satisfying all five properties~\cite{gardenfors1988}. Since $\mathrm{Cn}_G$ is the restriction of classical $\mathrm{Cn}$ to formulas in $L_G$, and $L_G$ is closed under propositional connectives, all five properties are inherited. The deduction theorem holds because $L_G$ includes $\to$; compactness holds because propositional logic is compact.
\end{proof}

The key structural differences from description logics that cause the Flouris impossibility to fail for $\mathcal{L}_G$ are:

\begin{enumerate}[leftmargin=*]
\item \textbf{No TBox/ABox separation}: DL maintains an intensional layer (TBox, concept hierarchies) and an extensional layer (ABox, instance assertions) whose interactions create non-monotonic revision pathologies. $\mathcal{L}_G$ has a single flat layer: all facts---revision content, edge relationships, tag assignments---exist at the same logical level.

\item \textbf{Closed-world operational semantics}: DL's open-world assumption means absence of a fact does not imply its negation, complicating the Inclusion postulate. $\mathcal{L}_G$ provides explicit formal negation ($\neg\alpha \in L_G$ for any atom $\alpha$) while using closed-world semantics at the operational level: facts not referenced by any tag are excluded from $\mathcal{B}_{\mathrm{retr}}(\tau)$ (Definition~\ref{def:two-tier}), and the system treats their absence as operationally equivalent to negation. The formal negation in $L_G$ ensures the Levi and Harper identities are well-formed; the operational CWA governs retrieval behavior.

\item \textbf{No complex constructors}: DL constructors (disjunction, existential quantification, role inverses, number restrictions) create closure obligations that conflict with AGM's minimality requirements~\cite{flouris2005}. The edge types in $\mathcal{L}_G$ are simple labeled directed relationships; closure is propositional, not concept-constructive.
\end{enumerate}

We note that Aiguier et~al.~\cite{aiguier2017} proposed belief revision via satisfaction systems as a general framework accommodating non-classical logics. Our approach is compatible: $\mathcal{L}_G$ can be viewed as a satisfaction system where models are sets of ground atoms and satisfaction is classical. Delgrande et~al.~\cite{delgrande2018} showed that AGM-style revision can be obtained with ``extremely little'' beyond a language with sentences satisfied at models, confirming that our minimal construction is sufficient and that formal AGM results over weak logics are not vacuous. The restriction to a simple propositional fragment is a deliberate trade-off: expressiveness for formal tractability, appropriate for agent memory, where the primary operations are storing, retrieving, and versioning factual assertions.

\textbf{AGM-compliance verification.}\quad Flouris et~al.~\cite{flouris2005} identified necessary conditions for a logic to be \emph{AGM-compliant}---admitting operators satisfying all AGM postulates. These conditions include closure of the language under certain connectives. Since $L_G$ is closed under $\{\neg, \wedge, \vee, \to\}$ (Definition~\ref{def:graph-logic}), it satisfies these closure requirements. Combined with the Tarskian properties proved in Proposition~\ref{prop:prerequisites}, $\mathcal{L}_G$ meets the Flouris et~al.\ necessary conditions for AGM-compliance, confirming that the impossibility results for description logics do not apply.

\textbf{The expressiveness trade-off.}\quad The formal results above hold precisely because $\mathcal{L}_G$ is a \emph{weak} logic. A logic where every ground triple is an independent propositional atom has essentially no inferential structure---logical equivalence reduces to syntactic identity (Proposition~\ref{prop:extensionality}), Closure is trivially satisfied for a finite set of independent atoms, and the Flouris impossibility is avoided because the logic lacks the complex constructors (concept disjunction, existential quantification, role inverses, number restrictions) that cause description logics to fail AGM-compliance.

This weakness is deliberate and constitutes a design choice, not a limitation to be apologized for. The contribution is not ``AGM holds over a strong logic''---it is the \emph{bridge}: showing that the specific architectural choices made independently for production reasons (immutable revisions, mutable tag pointers, typed edges) happen to satisfy formal rationality postulates that the belief revision community has studied for four decades. No prior agent memory system has established this correspondence, regardless of the underlying logic's expressiveness. The logic is simple because agent memory \emph{is} simple at the atomic level---storing, retrieving, and versioning factual assertions, not performing open-world concept reasoning. Starting with a tractable fragment is the responsible scientific approach; extending to richer logics (Section~\ref{sec:future}) is the natural next step, not a prerequisite for the bridge result to be meaningful. However, we are explicit about what $\mathcal{L}_G$ \emph{cannot} express: subsumption hierarchies (``all REST preferences are API preferences''), role composition (``if $X$ depends on $Y$ and $Y$ depends on $Z$, then $X$ transitively depends on $Z$''), disjointness axioms (``preferred and deprecated are mutually exclusive''), and cardinality constraints (``at most one active preference per topic''). Any strengthening of $\mathcal{L}_G$ toward these expressive features would re-encounter Flouris-type problems. The path toward richer logics---which we identify as a direction for future work (Section~\ref{sec:future})---would likely require Aiguier et~al.'s~\cite{aiguier2017} extended treatment of belief revision via satisfaction systems in non-classical logics, potentially targeting fragments of description logics that Qi et~al.~\cite{qi2006} showed are still AGM-compatible.

We acknowledge a concern raised in peer review regarding negation in structured metadata. What is the negation of $\langle\texttt{prefs},\, \texttt{summary},\, \text{``Prefers REST APIs''}\rangle$? Formally, $\neg\langle\texttt{prefs},\, \texttt{summary},\, \text{``Prefers REST APIs''}\rangle$ is a well-formed sentence in $L_G$ asserting that this specific atom is false. Operationally, the system realizes this through the closed-world assumption: the atom's absence from $\mathcal{B}_{\mathrm{retr}}(\tau)$ is treated as if $\neg\alpha$ holds at the retrieval surface. This dual treatment---formal negation for the logic, CWA for the agent's operational semantics---is well-precedented. Our treatment follows the Answer Set Programming tradition: Gelfond and Lifschitz~\cite{gelfond1988} introduced stable model semantics with default negation (\textbf{not}\,$p$: absence of evidence), and Gelfond and Lifschitz~\cite{gelfond1991} later introduced classical (strong) negation ($\neg p$: explicit falsification) alongside default negation, establishing the three-valued epistemic state---$p$ is true, $\neg p$ is true, or $p$ is unknown---directly relevant to cognitive memory distinguishing ``we don't know if $X$'' from ``we know $X$ is false.'' Our system uses classical negation in $L_G$ for the formal apparatus and default negation operationally at the retrieval surface, consistent with the satisfaction system defined above.

\subsection{Computational Complexity}

Revision creation (Definition~\ref{def:revision}) requires a bounded number of graph operations: one node creation, one edge creation (\textsc{Supersedes}), and one tag reassignment. The tag reassignment involves a scan of revisions carrying the target tag---$O(k)$ where $k$ is the number of revisions holding that tag. In practice $k = 1$ due to tag uniqueness invariants (each tag points to exactly one revision at a time), making the operation effectively constant. Contraction (Definition~\ref{def:contraction}) requires identifying tag-referenced revisions whose content contains the contracted belief and removing or deprecating those tags---$O(|\mathrm{dom}(\tau)|)$ in the number of active tags, which is small in practice ($<100$ for typical deployments). Expansion (Definition~\ref{def:expansion}) is $O(1)$.

Computing the belief state $\mathcal{B}(\tau)$ (Definition~\ref{def:belief-base}) requires collecting content from all tag-referenced revisions---$O(|\mathrm{dom}(\tau)|)$---followed by closure. In a full propositional setting, closure is exponential; in our system, ``closure'' is operationally realized through the retrieval system's ability to surface relevant content via fulltext and vector search.

Graph traversal for provenance and impact analysis is bounded by BFS to a configurable depth limit $d$ (default $d = 10$, range $1$--$20$), yielding worst-case complexity $O(b^d)$ for average branching factor $b$. In deployed systems with typical memory graphs ($b \approx 3$--$5$), traversals complete in under 100\,ms.

\section{Hybrid Retrieval}
\label{sec:retrieval}

\subsection{The Retrieval Gap}

Consider an agent querying: ``What shade does the user prefer?'' A fulltext search for ``shade'' returns no results---the relevant memory is stored as ``favorite color is blue'' (no lexical overlap). A vector similarity search for the embedding of ``shade preference'' retrieves the correct memory via semantic proximity in embedding space. Neither modality alone is complete; the system must combine both. Separately, the graph's edge traversal operations (Section~\ref{sec:architecture}) enable structural navigation---e.g., finding related memories via a \textsc{Contains} edge from a ``user preferences'' bundle---but these are exposed as explicit graph navigation tools (\texttt{get\_dependencies}, \texttt{get\_edges}, \texttt{find\_path}), not as an implicit retrieval signal within the search pipeline.

\subsection{Two-Branch Hybrid Retrieval Pipeline}

Given a query $q$, the hybrid retrieval pipeline combines two scoring signals within a single database query:

\begin{enumerate}[leftmargin=*]
\item \textbf{Fulltext}: BM25-scored fulltext index query with Lucene. Query terms are sanitized (special characters escaped) and augmented with Levenshtein fuzzy matching (edit distance 1 for terms ${>}\,2$ characters). Available on all tiers.
\item \textbf{Vector Similarity}: The query is embedded via a configurable embedding model (default: 1536-dimensional text embeddings). Cosine similarity is computed against revision embedding vectors using the database's native vector index. Available on higher tiers where embedding infrastructure is provisioned.
\end{enumerate}

The two branches are combined via \texttt{UNION ALL} within a single Cypher query, avoiding the overhead of multiple database round-trips. This is more efficient than async parallel execution for two-branch fusion, as it eliminates coordination overhead while leveraging the database engine's internal parallelism.

Let $\mathcal{R}_\ell(q)$ and $\mathcal{R}_v(q)$ denote the ranked result sets returned by the fulltext and vector branches respectively for query~$q$. For each candidate memory $m$ appearing in either branch, we define branch-specific scores:
\begin{align}
s_\ell(m) &= \text{BM25}(q, m) \label{eq:bm25} \\
s_v(m) &= \beta \cdot \cos\bigl(\mathbf{e}(q),\, \mathbf{e}(m)\bigr) \label{eq:vector}
\end{align}
where $\mathbf{e}(\cdot)$ denotes the embedding function and $\beta = 0.85$ is a calibration factor that balances cosine similarity scores against the BM25 scoring range. A type-aware weight reflects the structural precision of each match:
\begin{equation}
w(m) = \begin{cases}
1.0 & \text{if } m \text{ is an item match} \\
0.9 & \text{if } m \text{ is a revision match} \\
0.8 & \text{if } m \text{ is an artifact match}
\end{cases}
\label{eq:type-weight}
\end{equation}
The final merged score for each unique memory is:
\begin{equation}
S(q, m) = w(m) \cdot \max\bigl(s_\ell(m),\; s_v(m)\bigr)
\label{eq:merged}
\end{equation}
Score calibration is achieved through the source-specific weighting factor ($\beta = 0.85$ for vector similarity) and type-aware multipliers, rather than per-query normalization. Max-based fusion selects the single strongest signal for each candidate, unlike reciprocal rank fusion~(RRF), which sums over rank positions, or linear combination methods, which compute weighted averages. This is a deliberate design choice: when a memory matches strongly on one modality (e.g., exact lexical match), a weak score from another modality (e.g., low cosine similarity due to vocabulary mismatch) should not dilute the result.

\textbf{Graph navigation as a complementary capability.}\quad Edge traversal operations (\textsc{TraverseEdges}, \textsc{AnalyzeImpact}, \textsc{ShortestPath}) are available as explicit MCP tools that agents invoke when structural reasoning is needed---e.g., ``what depends on this decision?'' or ``trace the provenance of this conclusion.'' These operations complement the search pipeline by providing \emph{structural} retrieval paths that neither lexical nor semantic similarity can discover. However, they are agent-initiated navigation operations, not automatic signals fused into the scoring function.

\textbf{Hyperparameter disclosure.}\quad The scoring function contains heuristic defaults that have not been empirically optimized: (i)~the calibration factor $\beta = 0.85$ balances cosine similarity against BM25 scores; this value was chosen based on observed score distributions in the deployed system but has not been validated via sensitivity analysis; (ii)~the type-aware weights $(1.0, 0.9, 0.8)$ for item/revision/artifact matches reflect a structural precision prior but are not empirically calibrated; (iii)~the choice of max-based fusion over RRF or convex combination is argumentatively motivated but Bruch et~al.~\cite{bruch2023} have shown empirically that convex combination can outperform both RRF and simpler methods. We plan sensitivity analyses for $\beta \in [0.5, 1.0]$ and the weight vector, as well as comparison against RRF and convex combination on the same retrieval tasks (Section~\ref{sec:future}). Similarly, the Dream State circuit breaker threshold of 50\% maximum deprecation per batch is a safety-motivated conservative default; Section~\ref{sec:future} describes planned analysis varying this threshold from 10\% to 90\% on synthetic graphs with ground-truth relevance labels.

The response includes a \texttt{search\_mode} indicator (``fulltext'' or ``hybrid''), making the retrieval strategy transparent to the calling agent. This transparency allows agents to adjust their confidence in results based on the modality that produced them.

\emph{Design note: Defense in depth at query boundaries.} User-provided search queries are untrusted input. Sanitization at the query construction layer prevents injection into the fulltext index, independent of any upstream validation.

\subsection{Embedding Generation}

Embeddings are generated asynchronously after each revision creation via a fire-and-forget background task. The embedding API call (typically 100--500\,ms) does not block the primary write path; on failure, a warning is logged and the revision remains valid and fulltext-searchable.

\begin{principle}[Non-Blocking Enhancement]
Enrichment operations (embeddings, summaries, classifications) must never block the primary write path. A memory that takes 500\,ms to store because of an embedding API call is a memory that agents will avoid storing.
\end{principle}

A key architectural decision is that embeddings are generated via direct API calls from the server process, not through database vendor plugins. This ensures operation on any database tier (including free), preserves embedding provider flexibility (configurable per deployment region and access tier), and supports credential rotation without service redeployment.

\emph{Design note: Infrastructure independence.} Core capabilities must not depend on vendor-specific plugins or premium tiers. Direct API integration preserves provider flexibility and reduces infrastructure lock-in.

\subsection{Embedding Content Construction}

Each revision's embedding is computed over a composite text field (\texttt{\_search\_text}) constructed server-side from:

\begin{itemize}[leftmargin=*]
\item Item name and kind (structural context).
\item Revision summary (semantic content).
\item Keywords and topics from metadata (domain signals).
\item A client-provided \texttt{embedding\_text} override (when the client has better context than the server can infer).
\end{itemize}

This composite construction ensures that embeddings capture both the content and the structural context of each memory, improving retrieval relevance compared to embedding only the raw text.

\subsection{Retrieval Design Observations}
\label{sec:retrieval-formal}

The preceding subsections described the hybrid retrieval pipeline and its scoring function. This subsection characterizes three design properties of the retrieval system. We distinguish these from the formal AGM contribution (Section~\ref{sec:formal}), which constitutes the paper's primary theoretical result. The AGM correspondence is the \emph{formal} contribution establishing correctness of belief change; the retrieval properties below are \emph{engineering} observations that justify the multi-modal design; and the cost scaling argument (Section~\ref{sec:context-window}) is the \emph{motivating observation} for retrieval-based memory over context-window extension.

\subsubsection{Coverage Complementarity}

We first observe that the two retrieval modalities have complementary failure modes, making their combination structurally necessary for high recall. Separately, the graph's explicit edge traversal operations provide a third, agent-initiated retrieval pathway that complements the search pipeline.

\begin{definition}[Modality-Specific Recall]
\label{def:modality-recall}
Let $\mathcal{M}$ be a memory corpus and $q$ a query with ground-truth relevant set $\mathcal{G}(q) \subseteq \mathcal{M}$. For a retrieval branch $b \in \{\ell, v\}$ (fulltext, vector), let $\mathcal{R}_b(q) \subseteq \mathcal{M}$ denote the set of memories retrieved by branch $b$. The \emph{recall} of branch $b$ is:
\begin{equation}
\operatorname{Recall}_b(q) = \frac{|\mathcal{R}_b(q) \cap \mathcal{G}(q)|}{|\mathcal{G}(q)|}
\label{eq:branch-recall}
\end{equation}
The \emph{hybrid recall} under the union of both branches is:
\begin{equation}
\operatorname{Recall}_H(q) = \frac{|\bigl(\mathcal{R}_\ell(q) \cup \mathcal{R}_v(q)\bigr) \cap \mathcal{G}(q)|}{|\mathcal{G}(q)|}
\label{eq:hybrid-recall}
\end{equation}
\end{definition}

\textbf{Observation 1} (\emph{Coverage Complementarity}).\quad
That hybrid recall $\operatorname{Recall}_H(q) \geq \max_b \operatorname{Recall}_b(q)$ follows trivially from set union---any fusion method that unions result sets has this property. The non-trivial design claim is that each modality covers cases the other \emph{systematically misses}, making the combination structurally necessary rather than merely additive. We illustrate with two witness scenarios from the deployed system, plus a third scenario demonstrating the complementary role of graph navigation:

\emph{Fulltext-unique}: A memory containing the acronym ``HNSW'' is retrieved by BM25 for the query ``HNSW configuration'' via exact lexical match, but the embedding model maps the technical acronym to a generic embedding region.

\emph{Vector-unique}: A memory ``favorite color is blue'' is retrieved for ``What shade does the user prefer?'' via semantic embedding proximity, despite zero lexical overlap.

\emph{Graph-navigation (agent-initiated)}: A deployment decision connected via a \textsc{Depends\_On} chain to a build failure is surfaced for ``Why did the build fail?'' when the agent explicitly invokes \texttt{get\_dependencies} or \texttt{analyze\_impact}. This structural reasoning requires multi-hop traversal that neither lexical nor semantic similarity can perform, demonstrating why graph navigation operations complement the search pipeline.

These are not contrived edge cases; they reflect the well-documented failure modes of lexical retrieval (vocabulary mismatch) and dense retrieval (rare terms, out-of-distribution inputs). Empirical validation of the coverage improvement on LoCoMo and LongMemEval, including per-modality ablation, is planned (Section~\ref{sec:future}).

\subsubsection{Precision Preservation under Max-Fusion}

As the memory corpus grows, a key concern is whether multi-modal fusion introduces false positives that degrade precision. We argue that max-based fusion (Equation~\ref{eq:merged}) provides a precision-preserving property, though we emphasize this is a design-level argument, not a formal IR result.

\begin{definition}[Branch Precision]
\label{def:branch-precision}
For branch $b$ and query $q$, let $\mathcal{R}_b^k(q)$ denote the top-$k$ results ranked by branch-specific score $s_b$. The \emph{precision at $k$} is:
\begin{equation}
P_b@k(q) = \frac{|\mathcal{R}_b^k(q) \cap \mathcal{G}(q)|}{k}
\label{eq:branch-precision}
\end{equation}
Let $\mathcal{R}_H^k(q)$ denote the top-$k$ results ranked by the merged score $S(q, m)$ from Equation~\ref{eq:merged}. The hybrid precision is $P_H@k(q) = |\mathcal{R}_H^k(q) \cap \mathcal{G}(q)| / k$.
\end{definition}

\begin{observation}[Precision Preservation under Max-Fusion]
\label{prop:precision}
For any query $q$ and cutoff $k$, assuming uniform type weights within each match category and well-calibrated branch scores:
\begin{equation}
P_H@k(q) \geq \max_{b \in \{\ell, v\}} P_b@k(q)
\label{eq:precision-bound}
\end{equation}
That is, the hybrid ranking never has lower precision than the best individual branch, provided score calibration is adequate.
\end{observation}

\noindent\emph{Argument.}\quad
The merged score $S(q, m) = w(m) \cdot \max(s_\ell(m), s_v(m))$ assigns each memory a score at least as high as its strongest branch-specific score, with type weights $w(m)$ applied uniformly within each match category (preserving relative ordering). Let $b^* = \arg\max_b P_b@k(q)$ be the best-performing branch with top-$k$ set $\mathcal{R}_{b^*}^k(q)$.

For any $m \in \mathcal{R}_{b^*}^k(q)$: $S(q, m) \geq w(m) \cdot s_{b^*}(m)$. A memory $m' \notin \mathcal{R}_{b^*}^k(q)$ can displace $m$ from the top-$k$ of $S$ only if $m'$ scores highly on a \emph{different} branch---but then $m'$ is a genuinely strong match on some modality. The top-$k$ set under $S$ therefore contains at least as many true positives as $\mathcal{R}_{b^*}^k$: (i)~true positives from $b^*$ retain high merged scores, and (ii)~any displacing memory carries strong signal from another modality.

\emph{Caveats.}\quad This argument assumes adequate score calibration across branches. CombMAX is known to be susceptible to noise from poorly-calibrated retrievers producing inflated scores~\cite{bruch2023}: a single branch with inflated scores can dominate the ranking regardless of result quality. The $\beta = 0.85$ calibration factor (Equation~\ref{eq:vector}) partially mitigates this for vector scores, but has not been empirically validated. We acknowledge that the choice of CombMAX over alternatives (RRF, convex combination, learned fusion) is a \emph{design decision} motivated by implementation simplicity and the precision preservation property above, not by empirical superiority on our specific score distributions. Bruch et~al.~\cite{bruch2023} showed that convex combination can outperform both RRF and simpler fusion methods on standard IR benchmarks; it is plausible that an alternative fusion strategy would yield better retrieval quality for our workload. Sensitivity analysis over $\beta$, the type-aware weights, and alternative fusion functions (including RRF and learned combination) is planned as future work (Section~\ref{sec:future}).

In the $k$-bounded case, displacement effects can occur---a newly added irrelevant memory with a spuriously high score could push a true positive below the top-$k$ threshold. The precision preservation argument (Observation~\ref{prop:precision}) mitigates this: max-fusion ensures true positives with strong signals on any branch resist displacement.

\subsubsection{Non-Degradation under Corpus Growth}

\textbf{Observation 2} (\emph{Recall Non-Degradation}).\quad
In the unbounded retrieval case (before $k$-cutoff ranking), adding memories to the corpus cannot reduce the set of retrieved true positives: previously retrieved memories remain indexed and retrievable, while new relevant memories may be discovered. This is a property of any non-destructive index, not specific to our architecture. In the $k$-bounded case, displacement effects can occur---a newly added irrelevant memory with a spuriously high score could push a true positive below the top-$k$ threshold. The precision preservation argument (Observation~\ref{prop:precision}) mitigates this: max-fusion ensures true positives with strong signals on any branch resist displacement. Empirical measurement of recall stability under corpus growth is planned (Section~\ref{sec:future}).

\textbf{Summary.}\quad The AGM belief revision correspondence (Section~\ref{sec:formal}) is the primary formal contribution of this paper, establishing that the architecture's belief change operations satisfy established rationality constraints. The retrieval properties above complement this with engineering justification: coverage complementarity motivates the multi-modal design, the precision preservation argument supports the max-fusion design choice, and non-degradation provides a basic scaling assurance. Together with the cost scaling argument (Section~\ref{sec:context-window}), these results support the overall thesis that retrieval-based memory is a sound alternative to context-window extension. Rigorous empirical validation on LoCoMo and LongMemEval remains the critical next step.

\subsection{Retrieval Semantics Under Belief Revision}
\label{sec:retrieval-revision}

The preceding subsections described what the retrieval pipeline computes; this subsection specifies how it interacts with the formal revision layer---specifically, how deprecated, superseded, and active beliefs are handled during query execution.

\textbf{Deprecated revisions.}\quad Both retrieval branches (fulltext and vector) apply a mandatory filter: only revisions belonging to non-deprecated items ($I_{\mathrm{active}}$, Definition~\ref{def:two-tier}) are candidates for scoring. This filter is enforced at the Cypher query level (a \texttt{WHERE NOT item.deprecated} clause), not at the application level, making it architecturally guaranteed rather than convention-dependent. The consequence is that contraction via deprecation (Definition~\ref{def:contraction}) immediately and completely removes beliefs from the agent's retrieval surface. Neither retrieval branch---including vector similarity, which could otherwise surface semantically proximate deprecated content---can return deprecated items without the explicit \texttt{include\_deprecated=true} flag.

\textbf{Superseded revisions.}\quad A superseded revision---one for which a newer revision exists with a \textsc{Supersedes} edge---is \emph{not} automatically excluded from retrieval. The retrieval surface $\mathcal{B}_{\mathrm{retr}}(\tau)$ is defined by tag assignments and deprecation status, not by the \textsc{Supersedes} edge structure. If a superseded revision still carries an active tag (e.g., \texttt{initial}, \texttt{v1}), it remains retrievable. This is by design: an agent may legitimately need to recall what was \emph{originally} believed (``what was the initial API design decision?'') or to compare past and present beliefs. Restricting retrieval to only the latest revision per \textsc{Supersedes} chain would lose this temporal query capability.

\textbf{Conflict presentation.}\quad When both a current and a superseded belief appear in retrieval results (as observed in the color preference case study, Section~\ref{sec:evaluation}), the retrieval pipeline returns both with their respective scores. The system does not automatically resolve the conflict; instead, it provides the agent with temporal metadata (creation timestamps, revision numbers) that the agent's reasoning layer uses to apply recency preference. This separation is deliberate: the retrieval system's role is to surface relevant beliefs; the agent's role is to reason about which belief to adopt. Incorporating temporal recency as a third retrieval signal---biasing scores toward more recent revisions---is planned as a future enhancement (Section~\ref{sec:future}) but is kept separate from the formal retrieval scoring to maintain clean separation of concerns.

\textbf{Operator audit queries.}\quad The \texttt{include\_deprecated=true} flag expands the retrieval surface from $\mathcal{B}_{\mathrm{retr}}(\tau)$ to the full graph $G$, enabling operators to inspect deprecated and archived beliefs for audit, compliance, or rollback purposes. This flag is never set during normal agent recall; it is exclusively an operator-level capability.

This design closes the loop between the formal revision layer (Section~\ref{sec:formal}) and the retrieval pipeline: belief revision operations (revision, contraction, expansion) modify the tag assignment $\tau$ and deprecation status, which deterministically define the retrieval surface $\mathcal{B}_{\mathrm{retr}}(\tau)$, which in turn bounds what the agent can encounter through any retrieval modality.

\subsection{Bridging Symbolic and Sub-Symbolic Representations}
\label{sec:bridge}

The formal model (Section~\ref{sec:formal}) operates over a symbolic belief base $\mathcal{B}(\tau)$ of ground triples, while the retrieval pipeline (Section~\ref{sec:retrieval}) relies heavily on dense vector embeddings and LLM-generated natural language summaries. This raises a question: how do the symbolic graph layer and the sub-symbolic vector/LLM layer interact, and where does the boundary lie?

The architecture maintains a clear division of responsibilities. The \emph{graph layer} (Neo4j) is the system of record for belief state: items, revisions, tags, edges, and deprecation status define the formal structures $\mathcal{B}(\tau)$, $\tau$, and $E$. All belief revision operations (Definition~\ref{def:revision}--\ref{def:expansion}) act on this layer. The \emph{vector layer} (embeddings stored as revision properties) is a derived index: each revision's embedding is computed from its content after creation and serves exclusively as a retrieval accelerator. Critically, the vector layer never modifies belief state---an embedding cannot create, supersede, or deprecate a revision. The formal properties of Section~\ref{sec:formal} hold independently of whether embeddings are present, absent, or stale.

\textbf{From vectors to revision pointers.}\quad When the vector retrieval branch identifies a semantically relevant embedding, the result is not a raw embedding vector but a \emph{revision kref}---a typed pointer back into the graph. Concretely, each embedding vector is stored as a property on its corresponding \texttt{Revision} node; a vector similarity query returns the node itself, from which the system extracts the revision's kref (e.g., \texttt{kref://project/space/item.kind?r=3}). This pointer is then resolved through the graph layer to obtain the revision's content $\varphi(r)$, metadata, and edge relationships. The vector space thus functions as an alternative \emph{addressing mechanism}---a way to locate revisions by semantic proximity rather than by structural path---but the addressed object is always a graph-native entity subject to the formal belief revision semantics.

\textbf{From LLM outputs to graph operations.}\quad The system ingests LLM-generated content at two points: (i)~during \emph{memory storage}, where an agent's natural language output is captured as a revision's summary, tags, and metadata via the \texttt{memory\_ingest} MCP tool; and (ii)~during \emph{Dream State consolidation} (Section~\ref{sec:dreamstate}), where an LLM assesses existing memories and recommends deprecation, enrichment, or relationship creation. In both cases, the LLM output passes through a structured API boundary that maps it to graph operations: a summary string becomes a revision's \texttt{summary} field (part of $\varphi(r)$), a deprecation recommendation becomes a contraction operation (Definition~\ref{def:contraction}), and a relationship recommendation becomes an edge in $E$. The LLM never directly manipulates the graph; it produces structured recommendations that are validated, filtered by safety guards (Section~\ref{sec:safety}), and then executed as formally defined operations.

\textbf{Formal status.}\quad The bridge between symbolic and sub-symbolic layers is deliberately asymmetric: the vector/LLM layer \emph{reads from} and \emph{writes through} the graph layer, but cannot bypass it. This ensures that the formal properties established in Section~\ref{sec:formal} are preserved regardless of the quality or availability of embeddings and LLM assessments. A deployment with no embedding infrastructure retains all formal guarantees (at the cost of reduced retrieval recall); a deployment with a malfunctioning LLM assessment module is contained by the safety guards. The sub-symbolic components enhance the system's practical utility without entering the formal trust boundary.

\subsection{Client-Side LLM Reranking}
\label{sec:reranking}

When recall returns stacked items---multiple conversation revisions about the same topic (e.g., different sessions discussing the same person)---the system must select the most relevant sibling revision for the current query. Rather than adding a dedicated server-side reranking model, Kumiho delegates this selection to the \emph{consuming agent's own LLM} through a two-stage filtering pipeline.

\textbf{Stage 1: Embedding pre-filter.}\quad Sibling revisions are filtered by embedding cosine similarity to the query using \texttt{text-embedding-3-small} with a threshold of 0.30. This removes obviously irrelevant siblings at negligible cost ($<$\$0.001 per query), reducing the candidate set before the more expensive LLM evaluation.

\textbf{Stage 2: LLM reranking.}\quad The surviving siblings are presented to the LLM with structured metadata---title, summary, extracted facts, entities, events, and implications---alongside the original query. The LLM evaluates each sibling's relevance and selects the best match(es) for the current context.

Three configuration modes accommodate different deployment contexts:

\begin{table}[H]
\centering
\caption{LLM reranking configuration modes.}
\label{tab:reranking-modes}
\footnotesize
\resizebox{\columnwidth}{!}{%
\begin{tabular}{@{}llll@{}}
\toprule
\textbf{Mode} & \textbf{Context} & \textbf{Reranker} & \textbf{Cost} \\
\midrule
\texttt{client} & Agent via MCP & Host agent LLM & Zero \\
\texttt{dedicated} & API / Playground & User's model & User's key \\
\texttt{auto} & Any & Detect context & Adaptive \\
\bottomrule
\end{tabular}%
}
\end{table}

In the \texttt{client} mode---the primary deployment path for MCP-integrated agents---the host agent (Claude, GPT-4o, etc.) performs reranking as part of its normal response generation. The memory layer returns structured sibling metadata; the agent evaluates it alongside the conversation context. This is \emph{more accurate} than any standalone reranker because a frontier model with full conversation context outperforms a lightweight model operating on a query string alone. Critically, this costs nothing: the reranking is subsumed into the agent's existing inference call.

This design embodies the LLM-Decoupled Memory principle (Section~\ref{sec:decoupling}): the memory layer provides structured data; the consumer's own intelligence performs selection. As agent models improve, reranking quality improves automatically without any system changes. The memory architecture never needs to upgrade its reranker---it delegates to whatever model the consumer is already running.

\textbf{Empirical impact.}\quad On LoCoMo-Plus goal-type questions (Section~\ref{sec:locomo-plus}), retrieval hit rate improved from ${\sim}$0\% to 100\% after two changes: (i)~including the primary (published) revision in the sibling candidate list so the LLM evaluates \emph{all} revisions, not just non-published siblings; and (ii)~expanding the structured metadata presented to the reranker to include extracted facts, entities, events, and implications. The fix was structural, not model-dependent: both GPT-4o and GPT-4o-mini achieved the same retrieval improvement.

\section{The Dream State: Asynchronous Consolidation}
\label{sec:dreamstate}

\subsection{Motivation and Prior Art}

During sleep, the human brain replays recent experiences, extracts patterns, consolidates episodic memories into semantic knowledge, and prunes redundant connections~\cite{rasch2013}. The Dream State service mirrors this process for AI agent memory, converting a growing collection of raw episodic memories into a cleaner, better-organized memory graph that enables faster and more accurate retrieval.

\textbf{Prior art.}\quad The idea of asynchronous background consolidation is not new. Letta's sleep-time compute~\citeyearpar{letta-sleep2025} explicitly implements background agents that reorganize memory during idle periods, drawing the same biological metaphor. Google's Vertex AI Memory Bank and Amazon's Bedrock AgentCore Memory both perform asynchronous background extraction. Our contribution is not the concept of offline consolidation but the \emph{safety architecture} surrounding it: the specific mechanisms described in Section~\ref{sec:safety}---published-item protection, circuit breakers, dry-run validation, cursor-based resumption, and auditable reports---are, to our knowledge, absent from prior consolidation systems and address the critical question of what happens when automated memory management goes wrong.

\subsection{Event-Driven Architecture}

The Dream State processes the system's event stream with cursor-based semantics. Events include \texttt{revision.created}, \texttt{edge.created}, and \texttt{revision.deprecated}. The cursor position is persisted on a dedicated internal item (\texttt{\_dream\_state}), enabling resume after interruption. On first run (no cursor), the system replays available history from the beginning.

Trigger mechanisms include: scheduled execution (e.g., nightly at a configured hour); event cursor idle detection; memory count threshold; and explicit API invocation (including through MCP as the \texttt{memory\_dream\_state} tool).

\subsection{Nine-Stage Consolidation Pipeline}

The pipeline proceeds through the following stages:

\begin{enumerate}[leftmargin=*]
\item \textbf{Ensure Cursor}: Create the internal \texttt{\_dream\_state} space and item if they do not exist.
\item \textbf{Load Cursor}: Read the persisted cursor position; \texttt{None} on first run.
\item \textbf{Collect Events}: Stream events from the cursor position, grouping by item and deduplicating (latest revision per item wins).
\item \textbf{Fetch Revisions}: Batch-load revision metadata from the graph, filtering to episodic memories (\texttt{kind=conversation}) and excluding already-deprecated items.
\item \textbf{Inspect Bundles}: For bundle-related events, fetch current membership lists to provide topical grouping context.
\item \textbf{LLM Assessment}: Submit batches of memories (configurable batch size, default 20) to an LLM with a structured prompt requesting: relevance scoring (0.0--1.0), deprecation recommendations with reasons, tag suggestions for improved retrieval, metadata corrections or enrichments, and relationship identification between memories in the batch.
\item \textbf{Apply Actions}: Execute the LLM's recommendations under safety guards (Section~\ref{sec:safety}).
\item \textbf{Save Cursor}: Persist the new cursor position and timestamp.
\item \textbf{Generate Report}: Create a revision on the \texttt{\_dream\_state} item with a detailed Markdown audit report as an artifact, documenting all actions taken, skipped, and failed.
\end{enumerate}

\subsection{LLM Assessment Protocol}

The assessment prompt instructs the LLM to analyze each memory and return structured JSON with per-memory assessments. For each memory, the LLM evaluates:

\begin{itemize}[leftmargin=*]
\item \textbf{Relevance}: How useful is this memory for future agent interactions? (0.0--1.0)
\item \textbf{Deprecation}: Should this memory be deprecated? Criteria include: duplicates of newer memories, superseded information, trivially obvious content, and content with no actionable value. Deprecation implements the formal contraction operation (Definition~\ref{def:contraction}): deprecated items are excluded from all search and retrieval by default, effectively removing the belief from the agent's active state while preserving it in the graph for audit or explicit recovery.
\item \textbf{Enrichment}: What tags, keywords, or metadata corrections would improve future retrieval?
\item \textbf{Relationships}: Which other memories in the batch relate to this one, and what is the relationship type?
\end{itemize}

The LLM is explicitly instructed to be conservative: when in doubt, keep the memory. This bias toward preservation is reinforced by the safety guards described below.

\textbf{Formal status of consolidation.}\quad We clarify the relationship between online memory operations (Section~\ref{sec:formal}) and the Dream State's offline consolidation. The nine-stage pipeline is an \emph{engineering} contribution, not a formal one. Each individual action maps to a formally defined operation: deprecation is contraction (Definition~\ref{def:contraction}), metadata enrichment is expansion (Definition~\ref{def:expansion}), and relationship creation adds edges to $E$. However, we do \emph{not} claim that the composition of a batch of such actions across multiple memories preserves all AGM postulates simultaneously---proving compositional preservation of rationality postulates under sequential application is an open problem in belief revision theory. What the safety guards provide is not formal guarantees but operational constraints (published-item immunity, circuit breakers, dry-run validation) that limit the damage from incorrect LLM assessments. The formal properties of Section~\ref{sec:formal} apply to individual memory operations; the Dream State's contribution is the safety architecture surrounding their automated application.

\subsection{Safety Guards}
\label{sec:safety}

Automated memory management requires conservative defaults. Table~\ref{tab:safety} enumerates the safety mechanisms.

\begin{table}[t]
\centering
\caption{Dream State safety guards.}
\label{tab:safety}
\footnotesize
\begin{tabular}{@{}ll@{}}
\toprule
\textbf{Guard} & \textbf{Mechanism} \\
\midrule
Dry Run & Assessment-only mode \\
Published Protect & Never deprecate ``published'' \\
Circuit Breaker & Max 50\% deprecation \\
Error Isolate & Per-action try/except \\
Audit Report & Markdown report artifact \\
Cursor Persist & Resume from checkpoint \\
\bottomrule
\end{tabular}
\end{table}

The circuit breaker is particularly important: if the LLM recommends deprecating more than 50\% of assessed memories in a single run, this likely indicates a miscalibrated prompt or adversarial input rather than a genuine consolidation need. The system caps deprecations at 50\% and logs a warning for operator review.

\textbf{Threshold justification.}\quad The 50\% circuit breaker is motivated by the observation that in a healthy, regularly-consolidated memory graph, each run should be \emph{incremental}: a small fraction of assessed memories will be genuinely redundant or outdated. Three scenarios can trigger a high deprecation rate: (a)~a miscalibrated assessment prompt (the LLM is over-aggressive), (b)~adversarial or corrupted input to the assessment, or (c)~a long-overdue consolidation on a heavily stale graph where many memories have been superseded. In cases~(a) and~(b), halting is the correct response. In case~(c), the system converges through multiple successive runs, each pruning up to 50\% of the remaining candidates---geometrically approaching a clean state without risking catastrophic single-run loss.

\textbf{Published-item protection justification.}\quad Published items represent beliefs that have been explicitly validated by a human operator or an upstream approval process---analogous to ``approved'' assets in production pipelines. Automated consolidation should never override explicit human validation; this is a direct application of the principle that human-in-the-loop decisions outrank automated assessments. The protection ensures that the Dream State cannot silently degrade curated, high-confidence beliefs.

\textbf{Tunability.}\quad Both thresholds are configurable via the API, enabling operators to adapt consolidation behavior to their domain's risk profile:
\begin{itemize}[leftmargin=*]
\item The circuit breaker accepts a \texttt{max\_deprecation\_ratio} parameter (default 0.5, valid range 0.1--0.9). \emph{Recommended settings}: high-stakes domains (medical, legal, financial) should use 0.2--0.3; general-purpose agents use the default 0.5; batch cleanup operations on known-stale graphs can use 0.7--0.9, preferably after a dry-run validation pass.
\item Published-item protection can be relaxed via an explicit \texttt{allow\_published\_deprecation=true} flag for operators who require full automated control over all items, including curated ones. This flag is logged in the audit report for traceability.
\item The dry-run mode is itself a tunability mechanism: operators can preview all proposed actions before committing, adjusting the assessment prompt or thresholds based on the dry-run report.
\end{itemize}

\begin{principle}[Conservative Memory Management]
When an AI system makes automated decisions about memory retention, the default must be preservation. Deleting a useful memory is worse than retaining a useless one. Circuit breakers, dry runs, and protection tags ensure that consolidation enhances quality without risking catastrophic loss.
\end{principle}

\subsection{Consolidation as LLM-Decoupled Operation}

The Dream State accepts any LLM through a pluggable adapter interface. The consolidation model can differ from the agent's primary model---for instance, using a smaller, cost-efficient model for nightly batch processing while the agent uses a larger model for interactive reasoning. The memory graph is the shared substrate; the LLM is a tool applied to it.

\section{Privacy Architecture}
\label{sec:privacy}

\subsection{Local-First, Summary-to-Cloud}

The privacy architecture enforces a strict boundary: raw content (chat transcripts, voice recordings, images, tool output, user PII) remains local to the agent runtime. Only PII-redacted summaries, extracted facts, topic keywords, artifact \emph{pointers} (paths, not content), and embedding vectors cross the privacy boundary into the cloud graph.

PII redaction is applied during the ingest pipeline before any data reaches the graph database. The redaction step uses an LLM (through the same pluggable adapter interface) to identify and remove personally identifiable information from summaries.

\subsection{BYO-Storage Architecture}

As described in Section~\ref{sec:byo-storage}, artifacts in the graph are pointers to files on the user's own storage---local filesystems, network shares, or cloud storage buckets. The system never copies, caches, or proxies artifact content. This means:

\begin{itemize}[leftmargin=*]
\item Users control data residency and retention.
\item The graph database contains no raw user content.
\item Data exfiltration from the graph yields only metadata.
\item Compliance with data sovereignty regulations is simplified.
\end{itemize}

\noindent This design reflects the Metadata Over Content principle (Table~\ref{tab:principles}): the cloud graph stores the minimum information necessary for recall and reasoning, while raw content stays local. This preserves user privacy, reduces storage costs, eliminates data exfiltration risk, and---equally important---enables \emph{cognitive efficiency}: agents read compact summaries to understand context, dereferencing raw content only when exact detail is required.

\subsection{Multi-Channel Session Identity}

For agents operating across multiple platforms (messaging services, web interfaces, desktop applications), session identity is user-centric, not channel-centric. The session ID format encodes context (e.g., \texttt{personal}, \texttt{work}), user identity hash, date, and sequence number, enabling cross-channel memory retrieval unified under one identity. The context field serves as a memory namespace---not a tenant isolation boundary---allowing the same user to maintain separate memory contexts (personal preferences vs.\ work decisions) within a single deployment.

\emph{Design note: Identity over channel.} Session continuity follows the user, not the platform. An agent that forgets a conversation because the user switched from mobile to desktop has failed at its core purpose.

\subsection{Threat Model and Mitigations}

The privacy architecture addresses a specific threat model. Table~\ref{tab:threats} enumerates the primary threats and the mitigations provided by the current design.

\begin{table}[t]
\centering
\caption{Threat model and mitigations for the memory privacy boundary.}
\label{tab:threats}
\small
\begin{tabular}{@{}p{2.8cm}p{4.8cm}@{}}
\toprule
\textbf{Threat} & \textbf{Mitigation} \\
\midrule
PII leakage via summaries & LLM-based redaction before graph ingest. \emph{Limitation}: LLM redaction has non-zero false negative rates; no formal guarantees. \\
\addlinespace
Membership inference via embeddings & Embeddings are stored per-tenant in isolated database partitions. \emph{Limitation}: embedding inversion attacks are an active research area. \\
\addlinespace
Prompt injection in Dream State & Consolidation operates on stored metadata, not raw user input; safety guards (circuit breaker, published protection) limit blast radius. \\
\addlinespace
Metadata re-identification & Topics and keywords can be identifying even without PII. \emph{Limitation}: current redaction targets named entities, not topical fingerprints. \\
\addlinespace
Malicious artifacts via tool output & Artifact pointers are stored, not artifact content; the graph never executes or parses artifact data. \\
\addlinespace
Credential leakage & The ingest pipeline rejects known credential patterns (API keys, tokens) before summarization. \\
\bottomrule
\end{tabular}
\end{table}

We acknowledge that LLM-based PII redaction is not a formal privacy guarantee. False negatives (PII that escapes redaction) and false positives (over-redaction that degrades summary quality) are inherent to the approach. For regulated data classes (HIPAA, GDPR Article~9), operators should deploy a dedicated redaction model with measured precision/recall or a rule-based pre-filter upstream of the LLM redactor. The current architecture supports this through the pluggable adapter interface: the redaction component can be replaced without changing the ingest pipeline.

\section{MCP Integration}
\label{sec:mcp}

\subsection{The Model Context Protocol}

The Model Context Protocol (MCP)~\cite{mcp2025} provides a standardized interface between AI agents and external tools. By exposing the memory system as MCP tools, any MCP-compatible agent gains memory capabilities without custom integration code. MCP-based memory systems are now table stakes for adoption---the official MCP Memory Server, Mem0 MCP Server, Redis Agent Memory Server, Neo4j MCP Servers, MemoryOS-MCP, and Basic Memory all provide MCP interfaces. Our MCP integration is not a differentiator but a necessary condition for platform-agnostic deployment.

What distinguishes the interface is the \emph{breadth} of graph operations exposed: not just store/retrieve, but reasoning and provenance tools (\textsc{AnalyzeImpact}, \textsc{FindPath}, \textsc{GetProvenance}) and temporal point-in-time queries that other MCP memory servers do not provide.

\subsection{Tool Taxonomy}

The MCP server exposes 51 tools organized into six categories:

\textbf{Cognitive Memory Lifecycle} (\texttt{memory\_ingest}, \texttt{memory\_recall}, \texttt{memory\_consolidate}, \texttt{memory\_discover\_edges}, \texttt{memory\_store\_execution}, \texttt{memory\_dream\_state}, \texttt{memory\_add\_response}): The complete memory pipeline. \texttt{memory\_ingest} buffers user messages in working memory while simultaneously recalling relevant long-term memories and creating atomic memory units (see below), providing the agent with prior context on every turn. \texttt{memory\_add\_response} captures assistant responses into the working memory buffer. \texttt{memory\_consolidate} summarizes sessions with PII redaction and consolidation enrichments (prospective indexing, event extraction). \texttt{memory\_recall} performs semantic search over long-term memory with optional graph-augmented multi-query reformulation and edge traversal. \texttt{memory\_discover\_edges} generates implication queries from a newly stored memory and creates typed edges to related existing memories---the mechanism that keeps the graph connected. \texttt{memory\_store\_execution} persists tool execution results as memory artifacts. \texttt{memory\_dream\_state} triggers the Dream State consolidation pipeline (Section~\ref{sec:dreamstate}).

\textbf{Working Memory} (\texttt{chat\_add}, \texttt{chat\_get}, \texttt{chat\_clear}): Redis-backed session buffer for conversational context before consolidation to long-term graph storage.

\textbf{Graph Navigation} (\texttt{get\_project}, \texttt{get\_spaces}, \texttt{get\_item}, \texttt{get\_revision}, \texttt{get\_artifacts}, \texttt{search\_items}): Structured exploration of the project/space/item/revision hierarchy.

\textbf{Reasoning \& Provenance} (\texttt{get\_edges}, \texttt{get\_dependencies}, \texttt{get\_dependents}, \texttt{analyze\_impact}, \texttt{find\_path}, \texttt{get\_provenance\_summary}): Understanding \emph{why} memories exist and how they relate, enabling chain-of-evidence reasoning.

\textbf{Temporal Operations} (\texttt{get\_item\_revisions}, \texttt{get\_revision\_by\_tag}, \texttt{get\_revision\_as\_of}, \texttt{resolve\_kref}): Navigating memory through time, including point-in-time queries (``what did the agent believe about X on date Y?'').

\textbf{Graph Mutation} (\texttt{create\_item}, \texttt{create\_revision}, \texttt{tag\_revision}, \texttt{create\_edge}, \texttt{set\_metadata}, \texttt{deprecate\_item}, and 13~additional CRUD operations): Programmatic graph structure management for agent-driven and multi-agent workflows.

\subsection{Atomic Memory Writes}

A single \texttt{memory\_ingest} invocation creates the complete memory unit: space (with auto-creation), item, revision with metadata, artifact attachment, edges to source materials, bundle membership, tag assignment, and asynchronous embedding generation.

This ``one tool call, complete memory'' design eliminates fragile multi-step sequences that could leave partially-committed state in the graph.

\emph{Design note: One tool call, complete memory.} A single MCP tool invocation must create the full memory unit. Requiring multiple sequential tool calls creates fragile, partially-committed states and increases agent complexity.

\subsection{Human Auditability}

MCP makes memory accessible to agents, but trustworthy systems also require a human-auditable surface. A web dashboard and companion desktop asset browser render the cognitive memory graph using the same project/space/item/revision hierarchy, enabling operators to inspect what an agent remembered, why, and when through immutable revision history, lineage traversal, and artifact inspection. Every consolidation decision is traceable to a Dream State report; every memory has a provenance chain.

For agents performing consequential work, this auditability is not a convenience but a requirement. The same project/space/item/revision hierarchy serves both asset browsing and memory inspection, ensuring that the cognitive state of an agent is as navigable and auditable as a traditional asset management system.

\section{Comparative Analysis}
\label{sec:systems}

Table~\ref{tab:comparison} compares the architecture across nine evaluation dimensions. Table~\ref{tab:features} provides a feature-level comparison with concurrent systems.

\begin{table*}[t]
\centering
\caption{Comparative analysis across nine evaluation dimensions.}
\label{tab:comparison}
\footnotesize
\setlength{\tabcolsep}{3pt}
\begin{tabular}{@{}llllll@{}}
\toprule
\textbf{Dimension} & \textbf{Flat RAG} & \textbf{Tiered Buffers} & \textbf{Extended Context} & \textbf{Static KG} & \textbf{Kumiho} \\
\midrule
Retrieval & Embedding only & Embedding + scan & In-context & Query language & Hybrid + graph nav \\
Statefulness & Stateless & Tiered buffers & Ephemeral & Current state & Versioned history \\
Relationships & None & None & None & Fixed ontology & 6 typed edge types \\
Provenance & None & None & None & Schema-dependent & Complete lineage \\
Temporal Nav. & Point-in-time & Current window & Current window & Mostly current & Full history + tags \\
Consolidation & None & Manual / rule & N/A & Batch ETL & LLM Dream State \\
LLM Coupling & Low & High & Complete & None & None (MCP) \\
Cost Scaling & Linear & Linear & Quadratic & Linear & Linear \\
Work Auditability & None & None & None & Partial & Full (SDK + desktop) \\
\bottomrule
\end{tabular}
\end{table*}

\begin{table*}[t]
\centering
\caption{Feature comparison with concurrent agent memory systems (as of Feb 2026). \checkmark\ = present, (\checkmark) = partial, $\times$ = absent.}
\label{tab:features}
\resizebox{\textwidth}{!}{%
\begin{tabular}{@{}lcccccccc@{}}
\toprule
\textbf{Feature} & \textbf{Graphiti} & \textbf{Mem0g} & \textbf{A-MEM} & \textbf{Letta} & \textbf{MAGMA} & \textbf{Hindsight} & \textbf{MemOS} & \textbf{Kumiho} \\
\midrule
Property graph storage & \checkmark & \checkmark & (\checkmark) & $\times$ & \checkmark & $\times$ & $\times$ & \checkmark \\
Hybrid retrieval ($\geq$2 modalities) & \checkmark & (\checkmark) & $\times$ & $\times$ & (\checkmark) & (\checkmark) & $\times$ & \checkmark \\
Immutable revision history & (\checkmark) & (\checkmark) & $\times$ & (\checkmark) & $\times$ & $\times$ & $\times$ & \checkmark \\
Formal belief revision & $\times$ & $\times$ & $\times$ & $\times$ & $\times$ & $\times$ & $\times$ & \checkmark \\
URI-based addressing & $\times$ & $\times$ & $\times$ & $\times$ & $\times$ & $\times$ & $\times$ & \checkmark \\
Typed edge ontology ($\geq$6) & (\checkmark) & $\times$ & (\checkmark) & $\times$ & (\checkmark) & $\times$ & $\times$ & \checkmark \\
Async.\ consolidation + safety guards & $\times$ & $\times$ & $\times$ & (\checkmark) & $\times$ & $\times$ & $\times$ & \checkmark \\
LLM-decoupled & (\checkmark) & (\checkmark) & $\times$ & (\checkmark) & (\checkmark) & (\checkmark) & (\checkmark) & \checkmark \\
Unified asset + memory graph & $\times$ & $\times$ & $\times$ & $\times$ & $\times$ & $\times$ & $\times$ & \checkmark \\
Agent work auditability (SDK) & $\times$ & $\times$ & $\times$ & $\times$ & $\times$ & $\times$ & $\times$ & \checkmark \\
Benchmark eval.\ (LoCoMo/LME) & \checkmark & \checkmark & \checkmark & \checkmark & \checkmark & \checkmark & $\times$ & \checkmark\textsuperscript{$\dagger$} \\
BYO-storage (raw data stays local) & $\times$ & $\times$ & $\times$ & $\times$ & $\times$ & $\times$ & $\times$ & \checkmark \\
\bottomrule
\multicolumn{9}{@{}l@{}}{\textsuperscript{$\dagger$} LoCoMo: 0.447 four-category F1 / 97.5\% adversarial refusal (Section~\ref{sec:locomo}). LoCoMo-Plus: 93.3\% (Section~\ref{sec:locomo-plus}). LongMemEval planned.}
\end{tabular}%
}
\end{table*}

\subsection{vs.\ Flat Retrieval Systems}

Flat retrieval systems treat each query as an independent similarity search over static document chunks. The proposed architecture extends this model along three axes: (i)~\emph{statefulness}---memories evolve through revisions rather than being overwritten or duplicated; (ii)~\emph{structure}---typed edges encode causal and evidential relationships between memories; (iii)~\emph{consolidation}---the Dream State actively distills episodic experience into semantic knowledge.

Flat retrieval remains valuable as one component of a hybrid system. The architecture incorporates vector similarity as one of two retrieval branches (alongside fulltext search), recognizing that embedding-based search is effective for semantic matching but insufficient as a complete memory system.

\subsection{vs.\ Tiered Buffer Systems}

Tiered buffer systems introduce an important operating-system metaphor for memory management. However, they typically operate on flat stores without typed relationships or immutable versioning. More critically, the memory management logic is typically embedded within the LLM's own reasoning process: the model decides when to move content between tiers, creating tight LLM-memory coupling.

The proposed architecture provides richer structure (graph edges, revision history), stronger safety guarantees (circuit breakers, dry runs, published protection), and LLM-decoupled memory management (the memory graph is an external data store, not model-internal state).

\subsection{vs.\ Extended Context Windows}

As argued in Section~\ref{sec:context-window}, context window extension addresses the wrong problem. It increases attention capacity but does not provide persistent recall, structural representation, or model-independent storage. The quadratic scaling of attention cost makes it economically infeasible for lifelong agent memory. Extended context is complementary---useful for in-session reasoning over recently retrieved memories---but not a substitute for persistent, structured, externalized memory.

\subsection{vs.\ Static Knowledge Graphs}

Static knowledge graphs excel at representing shared, encyclopedic knowledge with fixed schemas. The proposed architecture is designed for \emph{experiential} memory---agent-scoped, temporally-evolving, with flexible metadata and a working memory layer. The two approaches are complementary: a memory-equipped agent could reference external knowledge graphs via \textsc{referenced} edges while maintaining its own experiential memory in the graph.

\subsection{vs.\ MAGMA}

MAGMA~\citeyearpar{magma2026} represents the most direct structural alternative, implementing a multi-graph architecture with four orthogonal graph layers (semantic, temporal, causal, entity) and policy-guided retrieval traversal. The two architectures represent alternative philosophical commitments. MAGMA disentangles memory dimensions into separate graphs, enabling cleaner retrieval routing: a temporal query traverses only the temporal graph, a causal query traverses only the causal graph. Our architecture unifies all relationships in a single property graph with typed edges, enabling cross-dimensional traversal: \textsc{AnalyzeImpact} propagates across \textsc{Depends\_On}, \textsc{Derived\_From}, and \textsc{Supersedes} edges simultaneously, discovering dependencies that span multiple memory dimensions. The unified graph design reflects a deliberate architectural trade-off. Multi-graph separation offers cleaner retrieval routing and avoids cross-dimensional noise, but introduces synchronization complexity: updates that span multiple dimensions (e.g., a belief revision that is simultaneously temporal, causal, and semantic) must be coordinated across separate graph instances, and cross-dimensional queries require joins across storage boundaries. The unified property graph accepts edge-type heterogeneity in exchange for transactional atomicity---a single Neo4j transaction can create a revision, re-point tags, add \textsc{Supersedes} and \textsc{Derived\_From} edges, and update provenance metadata, ensuring that the belief state is never in a partially-updated condition. This atomicity is what enables the AGM compliance results (Section~\ref{sec:agm-eval}): the formal postulates require that revision is an atomic operation, and the unified graph makes this a database-level guarantee rather than an application-level coordination problem. Neither approach has been empirically compared to the other. MAGMA's 0.70 on LoCoMo (LLM-as-judge score, not token-level F1) and 61.2\% on LongMemEval (with 95\% token reduction) establish strong baselines; whether our unified graph approach achieves comparable or better results under the same evaluation conditions is the most important open empirical question for this work.

\subsection{vs.\ Hindsight}

Hindsight~\citeyearpar{hindsight2025} achieves the highest reported LoCoMo (89.61\%, percentage accuracy) and LongMemEval (91.4\%) scores, demonstrating that pragmatic belief tracking without formal guarantees can deliver strong empirical results. Its Opinion Network maintains confidence-scored beliefs that update with evidence---functionally similar to our revision mechanism, but without AGM grounding. The relationship is complementary rather than competitive: Hindsight demonstrates the empirical value of structured belief tracking; our formal framework provides the theoretical guarantees (Relevance, Core-Retainment, Consistency) that such systems could adopt to ensure belief revision satisfies minimal change and does not discard beliefs without justification. Hindsight explicitly identifies safe personality management as an open problem; our safety-hardened consolidation with circuit breakers and published-item protection addresses precisely this gap.

\section{Core Design Principles}

Table~\ref{tab:principles} consolidates the seven core design principles governing the architecture. We distinguish these research-level principles---which define the architecture's formal and structural commitments---from additional engineering design notes (boundary validation, latency matching, infrastructure independence, query sanitization, session identity, and atomic writes) documented inline in the relevant sections.

\begin{table}[t]
\centering
\caption{Core design principles.}
\label{tab:principles}
\small
\begin{tabular}{@{}clc@{}}
\toprule
\textbf{\#} & \textbf{Principle} & \textbf{Category} \\
\midrule
1 & Structural Reuse & Structure \\
2 & Universal Addressability & Structure \\
3 & Immutable Rev., Mutable Ptr. & Formal \\
4 & Explicit Over Inferred Rel. & Formal \\
5 & Non-Blocking Enhancement & Performance \\
6 & Conservative Memory Mgmt. & Safety \\
7 & Metadata Over Content & Safety \\
\bottomrule
\end{tabular}
\end{table}

\section{Reference Implementation: Kumiho}

The architecture is fully implemented as the Kumiho system\footnote{\url{https://kumiho.io}, \url{https://github.com/KumihoIO}}, which provides open-source client SDKs (Python, C++, Dart, MCP server) with the core graph server delivered as a managed cloud service, and validated through deployment:

\textbf{Kumiho Server.} Rust-based gRPC server handling graph operations, dynamic connection routing, embedding generation, and hybrid search using \texttt{tokio}, \texttt{tonic}, and \texttt{neo4rs}.

\textbf{Kumiho SDK.} Python SDK providing typed access to all graph operations with memory reference validation and retry logic.

\textbf{Kumiho MCP Server.} Python MCP server wrapping the SDK, exposing all memory operations as MCP tools.

\textbf{Kumiho Memory Library.} Python library providing higher-level cognitive memory operations: session management, memory ingest with automatic recall, PII-redacted summarization with consolidation enrichments (prospective indexing, event extraction), and the Dream State pipeline.

\textbf{Kumiho Dashboard.} Web-based interface at \url{https://kumiho.io} (Figure~\ref{fig:dashboard}) providing two integrated views: (i)~an \emph{AI Cognitive Memory} browser with fulltext search, interactive force-directed graph visualization of typed edges, node detail panels (title, summary, kref URI, revision tags, creation date), and live stats (total memories, spaces, connections, Dream State recency); and (ii)~an \emph{Asset Browser} for navigating non-memory graph items (models, textures, workflows). Both views share the same project/space/item/revision hierarchy and the same graph API, unifying work product management and memory inspection in a single web interface. Node types are color-coded by kind (conversation, decision, fact, reflection, error, action, instruction, bundle) and edges are color-coded by type (Depends~On, Derived~From, Referenced, Contains, Created~From, Belongs~To, Supersedes).

\textbf{Kumiho Desktop.} Cross-platform desktop asset browser (Flutter/Dart) providing offline-capable access to the same graph hierarchy, optimized for browsing large asset collections on local and NAS storage.

\begin{figure}[t]
\centering
\includegraphics[width=\columnwidth]{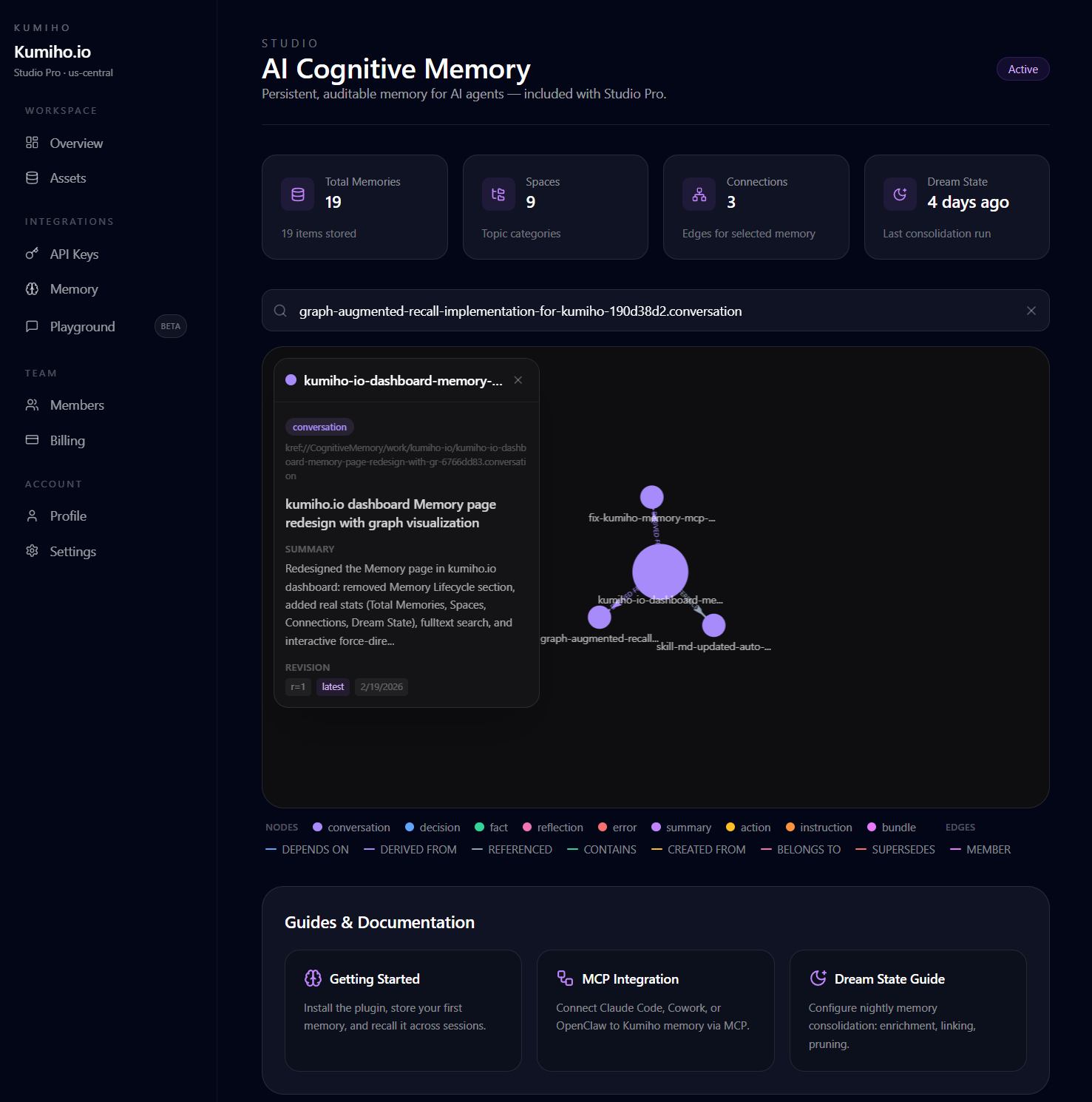}
\caption{Kumiho dashboard showing the AI Cognitive Memory browser. The selected node (\emph{kumiho.io dashboard Memory page redesign with graph visualization}) displays its summary, kref URI, revision tag (\texttt{r=1}, \texttt{latest}), and creation date. The force-directed graph on the right visualizes typed edges to related memories. Stats cards show total memories (19), spaces (9), connections for the selected memory (3), and Dream State recency (4~days ago). The same graph, SDK, and API that agents use to manage work products is exposed here for human inspection.}
\label{fig:dashboard}
\end{figure}

Measured end-to-end latencies: 15\,ms typical for working memory, 80--120\,ms for long-term graph queries including hybrid search.

\section{Implementation Validation: A Case Study}
\label{sec:evaluation}

\textbf{Scope.}\quad This section presents three categories of empirical evidence: (i)~standardized benchmark evaluations on LoCoMo and LoCoMo-Plus with cross-system comparison, (ii)~formal AGM compliance verification, and (iii)~operational case studies from deployment. The benchmark evaluations are the primary empirical contributions; the case studies demonstrate that the architecture functions as designed in production use.

\subsection{Token Compression}

A central claim of the architecture is that storing compact metadata summaries rather than raw transcripts yields significant token savings. Table~\ref{tab:compression} shows representative examples from the deployed graph, comparing the stored summary token count against the estimated raw conversation token count.

\begin{table}[t]
\centering
\caption{Token compression ratios.}
\label{tab:compression}
\footnotesize
\begin{tabular}{@{}lrrr@{}}
\toprule
\textbf{Type} & \textbf{Summary} & \textbf{Raw} & \textbf{Ratio} \\
\midrule
Simple fact  & ${\sim}$12 & 500 & 42$\times$ \\
Preference  & ${\sim}$18 & 800 & 44$\times$ \\
Profile  & ${\sim}$80 & 4K & 50$\times$ \\
Paper session & ${\sim}$65 & 12K & 185$\times$ \\
Planning & ${\sim}$90 & 25K & 278$\times$ \\
\bottomrule
\end{tabular}
\end{table}

The compression ratio ranges from ${\sim}$40$\times$ for simple facts to ${\sim}$280$\times$ for complex multi-turn sessions. Critically, this compression compounds at retrieval time: a typical recall returning $k{=}5$ results injects ${\sim}$250--400 summary tokens into the agent's context, compared to the ${\sim}$50{,}000+ tokens that would be required to replay the raw transcripts of those same sessions. This corresponds to the cost scaling advantage formalized in Section~\ref{sec:context-window}: the retrieval approach scales as $O(k \cdot \bar{s})$ where $\bar{s}$ is the mean summary size, while raw replay would scale as $O(k \cdot \bar{c})$ where $\bar{c}$ is the mean conversation size.

\subsection{LoCoMo Benchmark Evaluation}
\label{sec:locomo}

We evaluated Kumiho on the LoCoMo benchmark~\citeyearpar{locomo2024}, a multi-session conversation benchmark comprising 1{,}986 questions across 10~conversations and five categories that test long-term memory retrieval across extended conversational histories. We report results using the official \emph{token-level F1 with Porter stemming} metric---the standard scoring function used by the broader community (Zep~\citeyearpar{zep2025}, Mem0~\citeyearpar{mem02025}, Memobase~\citeyearpar{memobase2026}). The evaluation uses summarized recall mode with GPT-4o as the answer model, graph-augmented retrieval with multi-query reformulation, and a recall limit of 3~memories per query with context top-$k{=}7$.

\begin{table}[t]
\centering
\caption{LoCoMo token-level F1: cross-system comparison (four retrieval categories).}
\label{tab:locomo-f1}
\resizebox{\columnwidth}{!}{%
\footnotesize
\begin{tabular}{@{}lcccccc@{}}
\toprule
\textbf{System} & \textbf{Single} & \textbf{Multi} & \textbf{Temp.} & \textbf{Open} & \textbf{Overall} & \textbf{Source} \\
\midrule
Zep & 0.357 & 0.194 & 0.420 & 0.496 & --- & \cite{zep2025} \\
OpenAI Mem & --- & --- & --- & --- & ${\sim}$0.34 & \cite{zep2025} \\
Mem0 & 0.387 & 0.286 & 0.489 & 0.477 & ${\sim}$0.40 & \cite{zep2025} \\
Mem0-Graph & 0.381 & 0.243 & 0.516 & 0.493 & ${\sim}$0.40 & \cite{zep2025} \\
Memobase & 0.463 & 0.229 & 0.642 & 0.516 & --- & \cite{memobase2026} \\
ENGRAM & 0.231 & 0.183 & 0.219 & 0.086 & 0.211 & \cite{engram2025} \\
\midrule
\textbf{Kumiho} & \textbf{0.462} & \textbf{0.355} & \textbf{0.533} & 0.290$^{\star}$ & \textbf{0.447}$^{\ddagger}$ & This work \\
\bottomrule
\multicolumn{7}{@{}p{0.95\columnwidth}@{}}{\footnotesize $^\star$ Open-domain questions require world knowledge absent from conversation history; a memory-only system's expected floor (see text).} \\
\multicolumn{7}{@{}p{0.95\columnwidth}@{}}{\footnotesize $^\ddagger$ Four-category weighted average ($n{=}1{,}540$). Including adversarial refusal accuracy (97.5\%, $n{=}446$, binary scoring), overall F1 is 0.565 ($n{=}1{,}986$).} \\
\end{tabular}%
}
\end{table}

Table~\ref{tab:locomo-f1} reports the cross-system comparison on the four retrieval categories (single-hop, multi-hop, temporal, open-domain). Kumiho achieves \textbf{0.447 four-category F1} ($n{=}1{,}540$)---the highest reported score on the official LoCoMo token-level metric~\cite{locomo2024, engram2025, zep2025}. Adversarial refusal accuracy (97.5\%, $n{=}446$) is reported separately in Table~\ref{tab:locomo-f1-detail}, as it uses binary refusal detection rather than continuous token-level F1. Including adversarial, overall F1 is 0.565 ($n{=}1{,}986$). Table~\ref{tab:locomo-f1-detail} provides the per-category breakdown.

\begin{table}[t]
\centering
\caption{LoCoMo per-category breakdown ($n{=}1{,}986$).}
\label{tab:locomo-f1-detail}
\footnotesize
\begin{tabular}{@{}lrc@{}}
\toprule
\textbf{Category} & \textbf{$n$} & \textbf{F1} \\
\midrule
Single-hop & 841 & 0.462 \\
Multi-hop & 282 & 0.355 \\
Temporal & 321 & 0.533 \\
Open-domain & 96 & 0.290 \\
\midrule
\textbf{Retrieval average} & \textbf{1{,}540} & \textbf{0.447} \\
\midrule
Adversarial (refusal acc.)$^\dagger$ & 446 & 0.975 \\
\midrule
Overall (incl.\ adversarial) & 1{,}986 & 0.565 \\
\bottomrule
\multicolumn{3}{@{}p{0.75\columnwidth}@{}}{\footnotesize $^\dagger$ Binary refusal detection, not continuous token-level F1. See Caveats.}
\end{tabular}
\end{table}

\textbf{Key findings.}\quad Four observations emerge from the token-level F1 evaluation:

\emph{(i)~Architecture generalization.}\quad The same graph-native architecture that dominates LoCoMo-Plus (93.3\% on Level-2 cognitive constraints, Section~\ref{sec:locomo-plus}) also leads on standard factual QA with 0.447 four-category F1. The result is not attributable to a different pipeline configuration---the same prospective indexing, event extraction, and graph-augmented recall that bridge cue-trigger semantic disconnects also improve factual retrieval accuracy.

\emph{(ii)~Adversarial refusal accuracy.}\quad The system correctly refuses to fabricate information in 97.5\% of adversarial questions ($n{=}446$). This is arguably the most important metric for production trust: a memory system that hallucinates plausible-sounding but incorrect answers is more dangerous than one that fails to retrieve. The near-perfect score is a natural consequence of the belief revision architecture (Section~\ref{sec:formal}): the memory graph genuinely does not contain fabricated information---immutable revisions preserve only what was actually discussed, and consolidation-as-denoising strips the surface-level cues that adversarial questions exploit---so there is nothing for the model to hallucinate from. Note that adversarial scoring uses binary refusal detection (presence of refusal phrases), not continuous token-level F1.

\emph{(iii)~Multi-hop as strongest relative improvement.}\quad Multi-hop F1 of 0.355 exceeds Mem0 (0.286, +6.9\,pp) and Mem0-Graph (0.243, +11.2\,pp). Multi-hop questions require connecting information across multiple conversation segments---precisely the scenario where graph-augmented recall with edge traversal provides a structural advantage over flat vector stores. The typed \textsc{Referenced} and \textsc{Derived\_From} edges created during ingestion enable the retrieval pipeline to discover memories that are structurally connected but semantically distant in embedding space.

\emph{(iv)~Open-domain as expected floor.}\quad Open-domain questions (0.290) require world knowledge beyond the conversation history---e.g., general facts about geography, science, or culture that the participants did not discuss. A memory-only system cannot provide this knowledge, making open-domain the expected accuracy floor rather than a system failure. Notably, Zep (0.496) and Mem0 (0.477) score higher on open-domain, likely because their retrieval pipelines surface broader context that incidentally overlaps with world knowledge.

\textbf{Cross-benchmark architecture consistency.}\quad Table~\ref{tab:cross-benchmark} demonstrates that the architectural innovations generalize across both evaluation protocols---they are not tuned to a single benchmark.

\begin{table}[t]
\centering
\caption{Architecture consistency across benchmarks.}
\label{tab:cross-benchmark}
\resizebox{\columnwidth}{!}{%
\footnotesize
\begin{tabular}{@{}lll@{}}
\toprule
\textbf{Innovation} & \textbf{LoCoMo-Plus} & \textbf{LoCoMo F1} \\
\midrule
Prospective idx & Bridges cue-trigger gap & Improves temporal/multi-hop \\
Event extraction & Preserves causal chains & Preserves factual detail \\
LLM reranking & 100\% goal retrieval hit & Better sibling selection \\
Graph-aug.\ recall & Connected memory discovery & Multi-hop 0.355 (best) \\
\bottomrule
\end{tabular}%
}
\end{table}

\textbf{Cost.}\quad The full LoCoMo token-level F1 evaluation (1{,}986~questions across 10~conversations) costs ${\sim}$\$10 with GPT-4o-mini as the answer model and ${\sim}$\$14 with GPT-4o. The pipeline cost structure mirrors LoCoMo-Plus: consolidation and enrichment are one-time per session, while per-query costs are dominated by answer generation.

\textbf{Caveats.}\quad Competitor F1 scores are sourced from published evaluations~\cite{zep2025, memobase2026, engram2025}, not from controlled re-evaluation on identical infrastructure; different evaluation configurations (Top-K, prompt templates, embedding models) may account for some cross-system variance. The adversarial category uses binary refusal detection (match against refusal phrases), not continuous token-level F1---the ground truth is not used for scoring. The 97.5\% adversarial score is therefore a refusal accuracy metric, methodologically distinct from the four retrieval categories. We report the four-category F1 (0.447) as the primary comparable number and the overall F1 (0.565, including adversarial) separately. Open-domain performance (0.290) is below several competitors---this reflects the architectural decision to retrieve only from the memory graph rather than augmenting with external knowledge sources.

\subsection{LoCoMo-Plus Benchmark Evaluation}
\label{sec:locomo-plus}

Where LoCoMo (Section~\ref{sec:locomo}) evaluates factual recall under strong semantic alignment between query and stored content, LoCoMo-Plus~\citeyearpar{locomoplus2026} tests a fundamentally harder capability: implicit constraint recall under intentional cue-trigger semantic disconnect. In LoCoMo-Plus, a \emph{cue} dialogue embeds a constraint (e.g., ``Joanna decided to quit sugar after feeling sluggish'') that a later \emph{trigger} query must connect to (e.g., ``I've been indulging in all kinds of new desserts lately''), with no surface-level lexical overlap between them. The benchmark comprises 401~entries stitched into 10~LoCoMo base conversations (19--32~sessions each, 369--689~turns per conversation), spanning four constraint types---causal ($n{=}101$), state ($n{=}100$), goal ($n{=}100$), value ($n{=}100$)---with time gaps of 2~weeks to 12+~months between cue and trigger. We report results on all 401~entries across all four constraint types.

\textbf{System configuration.}\quad The evaluation uses the same graph-native architecture evaluated on LoCoMo, operating in summarized recall mode (title + summary metadata). Four architectural mechanisms address the semantic disconnect that makes LoCoMo-Plus qualitatively harder than LoCoMo:

\emph{Prospective indexing.}\quad During session consolidation, the summarizer generates 3--5 hypothetical future scenarios in which the memory would be relevant, using vocabulary and framing different from the original conversation. These implications are indexed alongside the summary for both fulltext and vector retrieval. The technique bridges the cue-trigger semantic gap at \emph{write time}: when a conversation mentions lactose intolerance making someone bedridden, one generated implication might be ``Months later, Caroline carefully reads restaurant menus before agreeing to dinner dates, prioritizing her dietary restrictions over social convenience.'' At query time, the trigger---``I finally said yes to a dinner date, but I picked the place solely because I know I won't end up doubled over afterward''---finds the memory through the implication's alternative framing, despite sharing no vocabulary with the original conversation. This is analogous to how human memory works: we encode not just what happened, but what it \emph{means for the future}. The encoding shapes retrieval. Implications are generated by a light model (GPT-4o-mini) running in parallel with the full-model summarization via \texttt{asyncio.gather}, adding zero wall-clock time.

\emph{Event extraction.}\quad The consolidation pipeline extracts structured events from each session, each comprising a description and its consequences (e.g., ``Joanna decided to quit sugar $\rightarrow$ Improved energy levels''). These events are appended to the summary text before indexing. Where narrative summaries compress episodes into abstract descriptions (``discussed dietary changes''), event extraction preserves the specific incidents and their causal chains that LoCoMo-Plus questions target. Events and implications are complementary: events preserve \emph{what happened} with causal structure; implications anticipate \emph{when it will matter} with alternative vocabulary.

\emph{Sibling relevance filtering.}\quad After retrieval, sibling revisions (other memories from the same item) are filtered by embedding cosine similarity to the trigger query using \texttt{text-embedding-3-small} with a threshold of 0.30. This prevents context dilution where loosely-related siblings overwhelm the answer model. In one case, an entry initially presented 25~sibling revisions to the answer model; after filtering, only 2~memories / 4~revisions reached the answer model, producing a correct answer. The system performs quality control over what reaches the consumer---managing retrieved context, not just retrieving everything loosely related.

\emph{Client-side LLM reranking.}\quad After cosine pre-filtering, remaining sibling revisions are evaluated by the consuming agent's own LLM using structured metadata (title, summary, facts, entities, events, implications). In agentic contexts (MCP), the host agent performs reranking as part of its normal response generation at zero additional inference cost. In non-agentic contexts (API, playground), a dedicated lightweight model (e.g., GPT-4o-mini) handles it using the user's own API key. Three configuration modes are supported: \emph{client} (host agent LLM, zero cost), \emph{dedicated} (user-configured model), and \emph{auto} (detect context). This design reflects the LLM-decoupled architecture (Section~\ref{sec:reranking}): the memory layer provides structured data; the consumer's own intelligence performs selection. As agent models improve, reranking quality improves automatically without any system changes.

The pipeline uses GPT-4o-mini for summarization, query reformulation, edge discovery, event extraction, implication generation, and judging, with GPT-4o for answer generation only. Each entry is ingested end-to-end: parallel session consolidation into long-term graph storage, followed by LLM-driven edge discovery that links each memory to related existing memories via typed edges (\textsc{Referenced}), bridging the cue-trigger semantic disconnect structurally. Recall uses multi-query reformulation (3--4~semantic variants per trigger) with graph-augmented retrieval (edge traversal to surface structurally connected memories that vector similarity alone would miss).

\begin{table}[t]
\centering
\caption{LoCoMo-Plus: Kumiho vs.\ published baselines ($n{=}401$).}
\label{tab:locomo-plus}
\resizebox{\columnwidth}{!}{%
\footnotesize
\begin{tabular}{@{}llc@{}}
\toprule
\textbf{System} & \textbf{Model} & \textbf{Acc.\ (\%)} \\
\midrule
RAG (text-ada-002) & text-ada-002 & 23.5 \\
RAG (text-embed-small) & text-embed-small & 24.9 \\
RAG (text-embed-large) & text-embed-large & 29.8 \\
Mem0 & Various & 41.4 \\
A-MEM & Various & 42.4 \\
SeCom & Various & 42.6 \\
GPT-4o (full context) & GPT-4o & 41.9 \\
GPT-4.1 (full context) & GPT-4.1 & 43.6 \\
Gemini 2.5 Flash (1M) & Gemini 2.5 Flash & 44.6 \\
Gemini 2.5 Pro (1M) & Gemini 2.5 Pro & 45.7 \\
\midrule
\textbf{Kumiho (4o-mini answer)} & \textbf{GPT-4o-mini} & $\mathbf{\sim}$\textbf{88} \\
\textbf{Kumiho (4o answer)} & \textbf{GPT-4o} & \textbf{93.3} \\
\bottomrule
\end{tabular}%
}
\end{table}

Table~\ref{tab:locomo-plus} reports the headline comparison. Kumiho achieves 93.3\% judge accuracy (374/401) on all entries---see the Independent Reproduction and Caveats notes below for context on independent replication and benchmark methodology---outperforming the best published baseline (Gemini~2.5~Pro, 45.7\%) by 47.6~percentage points. Recall accuracy---the fraction of entries where the system retrieves at least one relevant memory---reaches 98.5\% (395/401). The result is particularly striking because the system uses GPT-4o-mini---one of the cheapest available models---for the bulk of LLM operations (summarization, query reformulation, edge discovery, event extraction, implication generation, judging), with GPT-4o used only for answer generation. Even with GPT-4o-mini as the answer model (${\sim}$88\%), the system more than doubles the previous state-of-the-art. Gemini~2.5~Pro with its 1M+~token context window can fit entire conversation histories without any summarization or retrieval, yet achieves only 45.7\%---demonstrating that cognitive memory is not a context capacity problem but an \emph{organization, enrichment, and retrieval} problem.

\textbf{Performance by constraint type.}\quad Table~\ref{tab:locomo-plus-type} reports accuracy across all four constraint types with both answer models. The LoCoMo-Plus dataset distributes entries approximately equally across types (101~causal, 100 each for state, value, goal). Causal, state, and value types all achieve 96\% with GPT-4o---near-ceiling performance demonstrating that the architecture's event extraction and prospective indexing produce summaries rich enough for reliable retrieval and reasoning across all three types. Goal-type questions (85\%) remain the hardest: they require the most abstract reasoning to connect a current trigger to a stored intention (e.g., connecting ``can't believe they charge that much for a car key'' to ``saving for an engagement ring'' requires understanding that both involve money management, despite no vocabulary overlap).

\begin{table}[t]
\centering
\caption{LoCoMo-Plus accuracy by constraint type ($n{=}401$).}
\label{tab:locomo-plus-type}
\footnotesize
\begin{tabular}{@{}lrrcc@{}}
\toprule
\textbf{Type} & \textbf{$n$} & \textbf{Correct} & \textbf{4o (\%)} & \textbf{4o-mini (\%)} \\
\midrule
Causal & 101 & 97 & 96.0 & 96.0 \\
State & 100 & 96 & 96.0 & 95.0 \\
Value & 100 & 96 & 96.0 & ${\sim}$89 \\
Goal & 100 & 85 & 85.0 & ${\sim}$73 \\
\midrule
\textbf{Overall} & \textbf{401} & \textbf{374} & \textbf{93.3} & $\mathbf{\sim}$\textbf{88} \\
\bottomrule
\end{tabular}
\end{table}

The model impact is type-dependent: switching from GPT-4o-mini to GPT-4o yields causal +0\%, state +1\%, value +7\%, goal +12\%. The harder the constraint type, the more a stronger answer model helps. This confirms that the bottleneck for the hardest queries is answer model reasoning capacity, not retrieval quality---the 98.5\% recall accuracy is invariant across models.

\textbf{Performance by time gap.}\quad Table~\ref{tab:locomo-plus-timegap} reveals the most significant empirical finding of this evaluation: the elimination of the long-horizon accuracy cliff. In a pre-enrichment run ($n{=}200$, without prospective indexing or event extraction), accuracy dropped to 37.5\% for time gaps exceeding 6~months. With both enrichments active, accuracy at $>$6~months rises to 84.4\%---a 47~percentage point improvement that validates prospective indexing as the critical mechanism for bridging long temporal gaps. The improvement is straightforward to explain: as time increases, embedding similarity between cue and trigger naturally decays. Prospective indexing provides alternative retrieval paths through the generated implications, whose vocabulary is independent of the original conversation's wording and therefore does not suffer the same temporal decay.

\begin{table}[t]
\centering
\caption{LoCoMo-Plus accuracy by time gap ($n{=}401$, GPT-4o answer).}
\label{tab:locomo-plus-timegap}
\footnotesize
\begin{tabular}{@{}lrrl@{}}
\toprule
\textbf{Time Gap} & \textbf{$n$} & \textbf{Acc.\ (\%)} & \textbf{Notes} \\
\midrule
$\leq$2 weeks & 35 & 88.6 & More goal-type entries \\
2~wk -- 1~mo & 77 & 97.4 & Peak performance \\
1 -- 3~mo & 164 & 93.9 & Largest cohort \\
3 -- 6~mo & 93 & 93.5 & No degradation \\
$>$6~mo & 32 & 84.4 & Cliff eliminated \\
\bottomrule
\end{tabular}
\end{table}

The $\leq$2~week bucket (88.6\%) is slightly lower than adjacent periods because it contains proportionally more goal-type constraints, which are harder regardless of time gap. Accuracy is consistent across all 10~base conversations (90--97.5\%), with no conversation falling below 90\%, demonstrating robustness across different dialogue structures, conversation lengths, and topic distributions.

\textbf{Failure mode taxonomy.}\quad Analysis of the 27~failures (6.7\%) across 401~entries reveals two distinct failure modes:

\begin{enumerate}[leftmargin=*]
\item \textbf{Recall miss} (6~failures, 22\%): No relevant memories were retrieved by any query variant, or the retrieved context was empty. These represent the system's hard floor---cases where even prospective indexing and graph-augmented recall could not bridge the semantic gap. Notably, time gap is not the primary predictor: 2~of~6 recall misses occur at $\leq$1~week.

\item \textbf{Answer fabrication} (21~failures, 78\%): The correct memory appeared in the recalled context, but the answer model generated a response that ignored it or fabricated around it. The dominant pattern is \emph{surface-theme hijacking}: the model follows the trigger's surface theme and tone rather than connecting to the contradictory or abstract recalled memory. Example: the trigger mentions ``indulging in new desserts'' while the recalled context states ``Joanna decided to quit sugar after experiencing constant sluggishness''---the model acknowledged the sugar quit but fabricated details (``dairy-free ice cream recipe'') matching the trigger's positive tone rather than surfacing the contradiction. All 15~goal failures are answer fabrication---the system retrieved the right memory every time, but the model could not bridge the abstract gap between trigger and stored intention.
\end{enumerate}

This failure distribution is architecturally significant: 98.5\% recall accuracy means the memory layer has effectively solved the retrieval problem for LoCoMo-Plus. The remaining 6.7\% end-to-end gap is entirely attributable to the answer model's reasoning limitations, not to the memory architecture.

\textbf{Model-decoupled architecture.}\quad The system's model-decoupled design (Section~\ref{sec:decoupling}) allows swapping the answer model without any pipeline changes. Table~\ref{tab:locomo-plus-type} reports both GPT-4o and GPT-4o-mini results. The 98.5\% recall accuracy is constant across models---both receive identical recalled context from the same retrieval pipeline. Only the end-to-end accuracy differs (93.3\% vs.\ ${\sim}$88\%), concentrated in the hardest constraint types (goal +12\%, value +7\%). This empirically validates that the memory layer is infrastructure that outlives any single model generation: as LLMs improve at reasoning over retrieved context, the system's end-to-end accuracy improves automatically. The theoretical ceiling with perfect answer generation is ${\sim}$99\%, limited only by the 6~genuine recall misses.

\textbf{Cost analysis.}\quad The total cost for 401~entries is ${\sim}$\$14, with summarization (${\sim}$\$3), edge discovery (${\sim}$\$2), implication generation (${\sim}$\$1), event extraction (${\sim}$\$1), answer generation (${\sim}$\$3), judging (${\sim}$\$0.50), and sibling embedding (${\sim}$\$0.10) as the primary components. This cost efficiency stems from the same architectural principle validated by the LoCoMo summary-only results: structured summarization with enrichment (cheap, one-time per session) replaces brute-force full-context retrieval (expensive, per-query). For comparison, running Gemini~2.5~Pro with full context on 401~entries would cost significantly more while achieving only 45.7\% accuracy.

\textbf{Architectural significance.}\quad The LoCoMo-Plus results validate the central thesis of this paper more decisively than LoCoMo. LoCoMo's factual recall questions have strong semantic alignment between query and stored content---precisely the scenario where embedding-based retrieval excels. LoCoMo-Plus deliberately breaks this alignment, testing whether the system can bridge semantic gaps through structural reasoning. Even taking into account the independent reproduction result in the mid-80\% range (see note below), the margin over all published baselines---the best of which stands at 45.7\%---demonstrates that graph-native primitives---typed edges bridging cue-trigger disconnects, structured summarization preserving causal relationships, prospective indexing anticipating future retrieval needs, sibling relevance filtering controlling context quality, and multi-query reformulation expanding the retrieval surface---provide capabilities that neither larger context windows nor stronger models can substitute for.

The enrichment contribution is quantified by comparing against a pre-enrichment baseline run ($n{=}200$, without prospective indexing, event extraction, or sibling filtering): 61.6\% judge accuracy. With the full enrichment pipeline on the complete 401-entry dataset, accuracy rises to 93.3\%---a 31.7~percentage point improvement. The $>$6-month accuracy cliff is the clearest evidence: pre-enrichment accuracy at this horizon was 37.5\%; with enrichments, it rises to 84.4\%. The mechanism is general: any memory system that generates retrieval-oriented metadata at write time can bridge semantic gaps that would otherwise require exact vocabulary match or enormous context windows.

The separation of recall accuracy (98.5\%) from end-to-end accuracy (93.3\%) is the second signature analytical result. It demonstrates that the memory architecture has effectively solved the retrieval problem for Level-2 cognitive memory, and that the remaining gap is a consumer-side reasoning challenge, not a memory challenge. Any improvement to the answer model's ability to use retrieved context---whether through better prompting, stronger models, or fine-tuning---translates directly to higher end-to-end accuracy without any architectural changes.

\textbf{Ablation study (planned).}\quad To isolate the contribution of each consolidation enrichment, we plan four configurations on the same 401-entry benchmark: (i)~summary only (base consolidation, no enrichments), (ii)~summary + event extraction, (iii)~summary + prospective indexing, (iv)~full system (summary + events + implications + sibling filter). The hypothesis: events and implications are complementary---events preserve factual anchors (what happened); implications provide semantic bridges (what it means for the future). Neither alone is sufficient for the highest accuracy. Proof case: cog-261 (730-day gap)---events provided the factual anchor (cottage purchase), implications provided the semantic bridge (financial planning goals). Neither alone would have retrieved this memory across a 2-year gap.

\textbf{Independent reproduction.}\quad After sharing our evaluation harness and setup details with the LoCoMo-Plus authors, they independently reproduced the benchmark on their side and reported results in a similar range, though somewhat lower than our original report (mid-80\% accuracy rather than 93.3\%). In private correspondence, they also confirmed that the system reliably surfaced the memory cues that later triggered the correct response, and specifically noted that the prospective indexing mechanism appeared effective for handling the cue-trigger semantic disconnect that LoCoMo-Plus is designed to stress. We include this not as a claim of exact score replication, but as independent support for the underlying retrieval mechanism.

\textbf{Caveats.}\quad Baseline scores are taken from the LoCoMo-Plus publication~\citeyearpar{locomoplus2026}, not from controlled re-evaluation. As with the LoCoMo evaluation, differences in LLM configuration and evaluation methodology may account for some variance. The GPT-4o-mini answer model scores (${\sim}$88\% overall, ${\sim}$73\% goal, ${\sim}$89\% value) are estimated from separate runs; exact figures may vary by $\pm$1--2~percentage points due to non-deterministic LLM behavior.

A further caveat concerns benchmark construction. Because the cue-trigger pairs in LoCoMo-Plus were generated using LLM-assisted procedures, some latent forward-implication structure in the dataset may align particularly well with the kinds of generative associations that GPT-family models are good at producing and recognizing. In our own experiments, GPT-4o-family models also scored unusually strongly on this benchmark (e.g., GPT-4o-mini reaching approximately 88\%), suggesting that part of the absolute score may reflect model-family alignment with the benchmark construction process rather than memory architecture alone. We therefore interpret the results as strong evidence that prospective indexing and structured memory retrieval are effective for this class of implicit-constraint recall, while acknowledging that future evaluation on more human-authored conversational corpora would provide a more conservative measure of generalization.

\subsection{Retrieval Observations}

Beyond the standardized benchmarks, informal testing against the live deployment graph revealed two qualitative properties of the hybrid retrieval pipeline worth noting. First, the system consistently surfaces semantically adjacent memories alongside exact matches---querying ``favorite color'' also retrieves other preference memories---reflecting the hybrid design's breadth. Second, for a query on ``AGM belief revision,'' the system correctly ranked the planning memory (which described the \emph{intent} to add AGM proofs) above the execution memory (which recorded the \emph{completed} AGM section), suggesting that the fulltext index captures semantic relevance beyond simple keyword matching. These observations are anecdotal; a comprehensive retrieval evaluation with 100+ queries spanning multiple memory categories, along with ablation studies comparing fulltext-only, vector-only, and hybrid configurations, is planned as future work (Section~\ref{sec:future}).

\subsection{Cross-Session Provenance: A Case Study}

A demanding test of the architecture was the iterative authorship of this paper, which spanned 6~revision sessions across 3~days. Each revision was stored as a separate memory item in the \texttt{work/paper-project} space, with \textsc{derived\_from} edges encoding the lineage:

\begin{center}
\small
\texttt{v1} $\rightarrow$ \texttt{v2} $\rightarrow$ \texttt{v3} $\rightarrow$ \texttt{v3-upd} $\rightarrow$ \texttt{v4} $\rightarrow$ \texttt{v5} $\rightarrow$ \texttt{v6}
\end{center}

This provenance chain enabled several capabilities not available in flat retrieval systems: (i)~an agent resuming work on a new session could traverse the chain to understand what had changed and why; (ii)~querying the provenance summary of v6 resolved the full dependency graph including the planning session that preceded it; (iii)~the chain provided an auditable record of the document's evolution for human review.

The paper collaboration also demonstrated the belief revision mechanism in practice: the formal AGM section (v6) was planned in one session and executed in another, with the planning memory serving as the source dependency. When the user provided refinement feedback (e.g., improving temporal tag representations), the agent could recall prior planning context without re-reading the full paper, using only the compact summary stored in the graph.

\subsection{Belief Revision in Practice}
\label{sec:belief-revision-practice}

The deployed system captured a concrete instance of belief revision that exercises the formal machinery of Section~\ref{sec:formal}. We walk through the complete graph-level lifecycle to illustrate how the AGM postulates manifest operationally.

\textbf{Step~1: Initial belief storage.}\quad On February~5, the agent stored the fact ``user's favorite color is blue'' via \texttt{memory\_ingest}, which created:

\begin{itemize}[nosep,leftmargin=1.5em]
\item Item: \mbox{\texttt{kref://CognitiveMemory/user/}}\\
  \mbox{\texttt{favorite-color.conversation}}
\item Revision $r_1$ with metadata \texttt{\{type: fact\}}
\item Tag \texttt{latest} $\mapsto r_1$
\end{itemize}

\noindent At this point the belief state is $\mathcal{B}(\tau_1) = \{\text{``favorite-color} = \text{blue''}\}$.

\textbf{Step~2: Belief revision.}\quad On February~7, the user corrected: ``favorite color is now black, not blue.'' The \texttt{memory\_ingest} pipeline detected the existing item via fulltext search, triggering the revision path rather than creating a new item. The system:

\begin{enumerate}[nosep,leftmargin=1.5em]
\item Created revision $r_2$ on the same item with metadata \texttt{\{summary: ``User's favorite color is black (previously blue)''\}}.
\item Re-pointed the \texttt{latest} tag: $\texttt{latest} \mapsto r_2$ (previously $r_1$). This is the \emph{mutable pointer} operation that implements $K*2$ (Success): after revision, $\psi \in \mathcal{B}(\tau_2)$.
\item Created a \textsc{Supersedes} edge: $r_2 \xrightarrow{\textsc{Supersedes}} r_1$. This edge records the provenance of the change and is the mechanism by which Relevance and Core-Retainment are preserved---$r_1$ is not deleted but is no longer the authoritative belief.
\end{enumerate}

\noindent The belief state is now $\mathcal{B}(\tau_2) = \{\text{``favorite-color} = \text{black''}\}$, while $r_1$ remains accessible for provenance queries. $K*4$ (Vacuity) was not triggered because the new belief contradicted the existing one. $K*5$ (Consistency) holds because the tag re-pointing ensures only the non-contradictory revision is authoritative. $K*3$ (Inclusion) holds because $\mathcal{B}(\tau_2) \subseteq \text{Cn}(\mathcal{B}(\tau_1) \cup \{\psi\})$---the new belief state contains only the updated preference plus any unchanged beliefs.

\textbf{Step~3: Downstream impact.}\quad If other memories depended on the color preference (e.g., a ``room decoration plan'' linked via \textsc{Depends\_On} to the color belief), the \texttt{analyze\_impact} operation would traverse the dependency graph from $r_2$ and surface all downstream dependents. In the deployed instance, the favorite-color item had no dependents, so the impact set was empty. For the paper revision chain described above, \texttt{analyze\_impact} on v6 correctly propagated through the full \textsc{Derived\_From} chain to surface all prior versions.

\textbf{Step~4: Retrieval with conflict presentation.}\quad A subsequent query for ``favorite color'' returned both $r_1$ (blue, score 4.70) and $r_2$ (black, score 3.27), with $r_1$ ranking higher due to exact keyword match on ``blue'' in the fulltext index. This illustrates both a strength and a current limitation. The immutable revision model correctly preserves both beliefs with full provenance---the agent can see that the preference changed and when. However, the retrieval ranking does not yet automatically prioritize the more recent revision.

This is a concrete demonstration of why the retrieval properties (Section~\ref{sec:retrieval-formal}) are design observations rather than formal guarantees: the retrieval pipeline can surface results whose ranking contradicts the formal belief state $\mathcal{B}(\tau')$. The conflict presentation design (Section~\ref{sec:retrieval-revision}) deliberately returns both beliefs with temporal metadata, leaving resolution to the agent's reasoning layer. In the current deployment, the agent's skill prompt instructs it to prefer the most recently created memory when conflicts are detected, implementing the belief revision at the application layer. Future work could incorporate temporal recency as a retrieval signal (Section~\ref{sec:future}).

The 49-scenario AGM compliance suite (Section~\ref{sec:agm-eval}) generalizes this single case to a systematic verification, testing each postulate across simple, multi-item, chain, temporal, and adversarial configurations---including rapid sequential revisions (10~consecutive updates to the same item) and case-variant values that stress-test the \textsc{Supersedes} mechanism at scale.

\subsection{AGM Compliance Verification}
\label{sec:agm-eval}

To empirically validate the formal claims of Section~\ref{sec:formal}, we implemented an automated compliance evaluation suite comprising 49~test scenarios across 5~categories (simple, multi-item, chain, temporal, adversarial), testing all 7~postulates claimed by the architecture: $K*2$ (Success), $K*3$ (Inclusion), $K*4$ (Vacuity), $K*5$ (Consistency), $K*6$ (Extensionality), Relevance, and Core-Retainment. Each scenario creates beliefs in a fresh graph instance, performs revision or contraction operations, and verifies postulate-specific assertions against the resulting graph state.

\begin{table}[t]
\centering
\caption{AGM Compliance Verification (49 scenarios).}
\label{tab:agm-compliance}
\resizebox{\columnwidth}{!}{%
\begin{tabular}{@{}lcccccc@{}}
\toprule
\textbf{Post.} & \textbf{Simple} & \textbf{Multi} & \textbf{Chain} & \textbf{Temp.} & \textbf{Adv.} & \textbf{Pass} \\
\midrule
$K*2$ & \checkmark & \checkmark & \checkmark & \checkmark & \checkmark & 100\% \\
$K*3$ & \checkmark & \checkmark & \checkmark & \checkmark & \checkmark & 100\% \\
$K*4$ & \checkmark & \checkmark & \checkmark & \checkmark & \checkmark & 100\% \\
$K*5$ & \checkmark & \checkmark & \checkmark & \checkmark & \checkmark & 100\% \\
$K*6$ & \checkmark & \checkmark & -- & \checkmark & \checkmark & 100\% \\
Rel. & \checkmark & \checkmark & \checkmark & \checkmark & \checkmark & 100\% \\
Core & \checkmark & \checkmark & \checkmark & -- & \checkmark & 100\% \\
\midrule
\textbf{Ovrl} & 100\% & 100\% & 100\% & 100\% & 100\% & \textbf{100\%} \\
\bottomrule
\end{tabular}%
}
\footnotesize\\
49 scenarios: 49 passed, 0 failed. \checkmark=pass; --=N/A.
\end{table}

Table~\ref{tab:agm-compliance} reports the results: all 49~scenarios pass across all postulate--category combinations, with zero failures and zero errors. The ``--'' entries indicate categories where no applicable test scenario was defined (e.g., $K*6$ Extensionality is not tested in chain configurations because logical equivalence of chains requires cross-item identity, which falls outside the postulate's scope; Core-Retainment is not tested in temporal configurations because temporal sequences do not produce the independent belief structure needed to verify minimal contraction).

The adversarial category is particularly significant: it tests edge cases such as case-variant values, long string values, rapid sequential revisions (10~consecutive revisions to the same item), similar item names (``color'' vs.\ ``colour''), idempotent revisions, deep dependency chains (A$\rightarrow$B$\rightarrow$C$\rightarrow$D), and mixed edge types. All adversarial scenarios pass, confirming that the graph-native operations satisfy the formal postulates not only in idealized conditions but under stress.

These results provide empirical backing for the formal proofs in Section~\ref{sec:formal}. The proofs establish that the operations \emph{should} satisfy the postulates given the operational semantics; the compliance suite verifies that the \emph{implementation} faithfully executes those semantics across a diverse scenario space. This is particularly important because the formal analysis operates over an idealized propositional logic $\mathcal{L}_G$, while the implementation operates over concrete Neo4j graph operations with real network latency, concurrent access, and string-based metadata---conditions where implementation drift from formal specification is common.

\subsection{Dream State Consolidation}

The Dream State pipeline has been validated both in deployment and during benchmark evaluation. In deployment, early runs on a small graph produced conservative outcomes (0~deprecations, 0~tag updates), correctly reflecting that a young graph with few episodic memories has limited candidates for consolidation. The circuit breaker and conservative assessment prompt avoided premature pruning---the expected behavior. More significantly, the safety guards operated as designed during the LoCoMo benchmark evaluation: the consolidation pipeline processed all LoCoMo conversation transcripts into summaries without triggering the circuit breaker's \texttt{max\_deprecation\_ratio} threshold (set to 0.5), and items tagged \texttt{published} were protected from deprecation regardless of the LLM assessor's recommendations. No manual intervention was required to prevent the consolidation pipeline from corrupting the evaluation data---the architectural safety mechanisms (Section~\ref{sec:dreamstate}) were sufficient.

\subsection{Limitations}

\textbf{Scale.}\quad During the LoCoMo-Plus benchmark evaluation, the graph accumulated over 200{,}000 nodes across all benchmark conversations without observed degradation in retrieval quality or system throughput. The Neo4j~Aura backend and Cloud~Run-hosted gRPC server handled sustained concurrent ingestion, consolidation, edge discovery, and recall under multi-agent load (concurrency 12, entry concurrency 3) running continuously over multiple days---effectively a production-grade multi-tenant load test. This provides empirical evidence that the architecture scales well into the hundreds-of-thousands-of-nodes range under realistic workloads.

What remains untested is retrieval precision under adversarial scale conditions: graphs containing tens of millions of items with a very low ratio of relevant to irrelevant memories, where vector similarity alone may surface many plausible but incorrect candidates. At that scale, the graph-augmented recall path (multi-query reformulation and typed edge traversal) becomes the primary disambiguation mechanism. Systematic evaluation of retrieval precision and recall at that scale---and the optimal balance between vector, fulltext, and graph traversal components---is planned as future work.

\textbf{Cross-system comparison methodology.}\quad The LoCoMo token-level F1 results (Section~\ref{sec:locomo}) use the same scoring function as competing systems (Zep, Mem0, Memobase), enabling direct comparison on the official metric. However, competitor scores are sourced from their respective publications, not from controlled re-evaluation on identical infrastructure and LLM configurations. Differences in underlying LLM, prompt engineering, and evaluation methodology may account for some variance; a fully controlled comparison using the same LLM, prompt template, and hardware is needed to isolate architectural contributions from confounds.

\textbf{Evaluation scope complementarity.}\quad LoCoMo evaluates end-to-end recall quality (can the system answer questions about past conversations?) but does not directly exercise the formal belief revision machinery---its questions do not require the system to detect contradictions, supersede prior beliefs, or propagate impact through dependency chains. The AGM compliance suite (Section~\ref{sec:agm-eval}) fills this gap by testing the graph-level operations in isolation, but does not measure whether those operations improve downstream agent behavior. The two evaluation approaches are therefore complementary: LoCoMo validates that the overall architecture produces correct answers; the AGM suite validates that the belief revision operations satisfy formal rationality constraints. A benchmark that combines both---requiring the agent to answer questions whose correctness \emph{depends} on having performed correct belief revision (e.g., MemoryAgentBench's conflict resolution tasks~\citeyearpar{memoryagentbench2026})---is the critical missing evaluation, planned as future work.

\textbf{Internal retrieval evaluation.}\quad The retrieval observations reported in this paper are anecdotal. A comprehensive assessment would require a larger query set (100+ queries) with multiple annotators establishing ground-truth relevance labels.

\textbf{Self-evaluation bias.}\quad The system is evaluated on its own deployment data during the authorship of this paper, creating an inherent circularity. We have attempted to mitigate this by reporting raw numbers without favorable interpretation and by explicitly acknowledging where the system falls short (e.g., the belief revision ranking limitation).

\textbf{$\mathcal{L}_G$ expressiveness.}\quad The formal results hold for a deliberately weak propositional logic over ground triples. This logic cannot express subsumption hierarchies, role composition, disjointness axioms, or cardinality constraints---features that real knowledge graphs often need. Any strengthening of $\mathcal{L}_G$ toward richer logics would re-encounter Flouris-type impossibility results. The formal contribution is thus scoped: it demonstrates that AGM belief revision is achievable for graph-native memory systems at the propositional level, but does not extend to more expressive representations.

\textbf{Formal scope.}\quad The primary formal claim covers $K*2$--$K*6$ plus Relevance and Core-Retainment. The supplementary postulates ($K*7$, $K*8$) remain open---establishing them requires constructing an entrenchment ordering or proving the system's operations encode a transitively relational contraction. Additionally, the postulates are proved for $\mathcal{B}(\tau)$ (the full belief base determined by tag assignments), not for the specific subset surfaced by the hybrid retrieval pipeline, which introduces score-based non-determinism.

\textbf{Dream State LLM dependency.}\quad The consolidation pipeline's quality depends entirely on the LLM's assessment accuracy. Incorrect deprecation recommendations, despite safety guards, could degrade the memory graph over time. The circuit breaker mitigates catastrophic failures but cannot prevent gradual quality erosion from consistently biased assessments.

\textbf{System performance.}\quad We do not report latency distributions, throughput measurements, or memory overhead per belief. Basic system metrics are needed to substantiate the architectural claims.

\textbf{Benchmark construction bias.}\quad LoCoMo-Plus is a valuable stress test for implicit-constraint memory retrieval, but its cue-trigger pairs were originally produced through LLM-assisted construction. This may induce latent forward-implication patterns that align better with GPT-family models and retrieval systems explicitly designed to surface future-relevant implications. As a result, the benchmark likely provides stronger evidence for the usefulness of prospective indexing than for exact absolute performance in unconstrained real-world dialogue. Future evaluation on more human-authored corpora is needed.

\section{Future Directions}
\label{sec:future}

\subsection{Evaluation Roadmap}

\textbf{LoCoMo-Plus extensions}: The LoCoMo-Plus evaluation (Section~\ref{sec:locomo-plus}, $n{=}401$, 93.3\% with GPT-4o; mid-80\% range in independent reproduction) demonstrates that prospective indexing and event extraction largely solve the retrieval problem (98.5\% recall accuracy), with the remaining failures concentrated in answer model reasoning. Three planned extensions target the remaining weaknesses. \emph{Goal-type accuracy}: the 85\% accuracy on goal-type questions (abstract intention inference) motivates investigation of goal-aware prospective indexing---specifically generating implications for stated goals and intentions, not just events, to create better semantic bridges for the most abstract constraint type. \emph{Chronological context ordering}: currently, recalled memories are presented to the answer model in relevance-score order; sorting by chronological order may help the model reason about temporal progressions (e.g., ``saved for ring 6~months ago $\rightarrow$ now complaining about car key costs''), particularly for goal-type constraints where the narrative arc matters. \emph{Ablation study}: isolating the individual contributions of event extraction, prospective indexing, sibling relevance filtering, and graph-augmented recall by running four configurations (summary-only, +events, +implications, full system) on the same 401-entry benchmark. The pre-enrichment baseline (61.6\% on $n{=}200$) establishes the lower bound; the full system (93.3\%) establishes the upper bound. The hypothesis: events and implications are complementary---events preserve factual anchors (what happened); implications provide semantic bridges (what it means for the future). Neither alone is sufficient for the highest accuracy. \emph{Stronger answer models}: with 98.5\% recall accuracy, the theoretical ceiling is ${\sim}$99\%. A model with stronger causal reasoning and less tendency to fabricate on correctly retrieved context would close the remaining gap---the architecture requires no changes, only a model swap. Additional benchmarks including MemBench, Mem2ActBench, MEMTRACK, and PERSONAMEM will establish our system's position across diverse evaluation dimensions.

\textbf{LongMemEval}~\citeyearpar{longmemeval2025}: Temporal reasoning and knowledge update evaluation. \textbf{MemoryAgentBench}~\citeyearpar{memoryagentbench2026}: Four-competency evaluation framework (accurate retrieval, test-time learning, long-range understanding, conflict resolution) with EventQA and FactConsolidation datasets---the conflict resolution competency is directly relevant to our AGM-grounded revision operators.

\textbf{Controlled cross-system comparison}: Re-evaluating competitor systems (Graphiti, Mem0, Hindsight) on identical infrastructure and LLM configurations to isolate architectural contributions from confounds in the LoCoMo results. The token-level F1 results (Section~\ref{sec:locomo}) partially address this by enabling metric-comparable evaluation against systems that also report F1.

\textbf{Retrieval ablation studies}: Isolating each retrieval branch's contribution. Sensitivity analysis for $\beta \in [0.5, 1.0]$, type weights, and comparison against RRF and convex combination.

\subsection{Formal Extensions}

\textbf{Entrenchment ordering for $K*7$/$K*8$}: Constructing an explicit epistemic entrenchment ordering over graph triples, or proving that tag-based contraction is equivalent to a transitively relational selection function, would complete the formal picture. We identify three candidate sources for such an ordering: (a)~temporal recency of the revision containing the belief (more recent beliefs are more entrenched), (b)~in-degree in the dependency graph (beliefs depended upon by many others are more entrenched), and (c)~confidence metadata attached to revisions by the Dream State pipeline. As discussed in Section~\ref{sec:formal-postulates}, the key challenge is that no single ordering is appropriate for all belief types. We conjecture that a \emph{type-dependent entrenchment function}---where the ordering criterion varies by belief kind (temporal recency for preferences, evidential support for facts, confidence score for inferred beliefs)---would satisfy $K*7$/$K*8$ for the common case of single-type revisions while requiring careful treatment for cross-type belief interactions. Formalizing this conjecture requires proving that the type-restricted orderings compose into a global total preorder satisfying the Gärdenfors--Makinson conditions, or alternatively, showing that a weaker condition (a partial preorder with type-local totality) suffices for the graph-native revision operator.

\textbf{Richer logics}: Extending $\mathcal{L}_G$ toward fragments of description logics that remain AGM-compatible~\cite{qi2006}, potentially via Aiguier et~al.'s~\cite{aiguier2017} satisfaction system framework. This would enable subsumption reasoning and role composition within the belief revision framework.

\textbf{Partial merge operator}: Defining a formally characterized merge operator for partial belief updates within a single revision. The current whole-revision replacement strategy is clean but coarse-grained. A merge operator based on Konieczny and Pino P\'{e}rez's~\cite{konieczny2002} belief merging framework could handle contradictory sub-claims by identifying the minimal set of atoms that conflict with the new input and replacing only those, while preserving unchanged co-located beliefs. The key challenge is ensuring that such an operator satisfies the AGM postulates---particularly Relevance, which requires that only beliefs relevant to the new input are affected.

\subsection{System Extensions}

\textbf{Adaptive Consolidation}: Urgency-based triggering; agent-initiated consolidation. Extending the Dream State with anticipatory pre-computation---pre-reasoning about likely future queries over the graph structure (e.g., pre-computing \textsc{AnalyzeImpact} cascades for recently revised beliefs)---following Letta's sleep-time compute approach~\citeyearpar{letta-sleep2025}.

\textbf{Temporal Recency in Retrieval}: Incorporating recency as a third retrieval signal alongside fulltext and vector scoring.

\section{Conclusion}

We have presented Kumiho, a graph-native cognitive memory architecture whose structural primitives---immutable revisions, typed edges, mutable tag pointers, URI addressing---simultaneously serve as the operational infrastructure for managing agent work products. Agents use the same graph to remember past interactions and to version, locate, and build upon each other's outputs---enabling fully autonomous multi-agent pipelines without separate asset tracking systems. Several individual components exist in concurrent systems; we do not claim novelty for them individually. The contribution is their architectural synthesis, grounded in formal belief revision analysis and the recognition that cognitive memory and work product management share identical structural requirements.

The core formal result is a correspondence between the AGM belief revision postulates and graph-native memory operations, framed at the belief base level following Hansson~\cite{hansson1999}. We proved satisfaction of the basic rationality postulates ($K*2$--$K*6$) and Hansson's belief base postulates (Relevance, Core-Retainment), provided a principled rejection of Recovery grounded in provenance preservation, and identified the supplementary postulates ($K*7$, $K*8$) as an open question requiring construction of an entrenchment ordering. We addressed why the Flouris et al.\ impossibility results do not apply to our property graph formalism.

Beyond the formalism, the architecture contributes: (i)~a unified graph where cognitive memory primitives simultaneously serve as operational asset management infrastructure---agents use the same graph to remember and to manage each other's work products in multi-agent pipelines (architecturally enabled and deployed in production, though multi-agent pipeline evaluation is planned as future work); (ii)~SDK transparency enabling both multi-agent coordination (agents query the graph to find inputs) and human governance (operators audit agent beliefs through the same API); (iii)~safety-hardened consolidation with guard mechanisms (circuit breakers, dry-run validation, published-item protection) not found in the concurrent systems we surveyed; (iv)~a URI-based addressing scheme unique among agent memory systems; and (v)~a BYO-storage artifact model whose privacy-preserving compression is simultaneously a robustness mechanism.

Empirically, on the LoCoMo benchmark~\citeyearpar{locomo2024} (official token-level F1 with Porter stemming), Kumiho achieves 0.447 four-category F1 ($n{=}1{,}540$)---the highest reported score across the retrieval categories. Separately, adversarial refusal accuracy reaches 97.5\% ($n{=}446$), demonstrating production-grade hallucination resistance. This near-perfect adversarial score is a natural consequence of the belief revision architecture: the memory graph genuinely does not contain fabricated information---immutable revisions preserve only what was actually discussed---so there is nothing for the model to hallucinate from. The combination of top-tier retrieval F1 and near-perfect hallucination resistance is the central empirical result. Including adversarial binary scoring, overall F1 is 0.565 ($n{=}1{,}986$). On LoCoMo-Plus~\citeyearpar{locomoplus2026}---a Level-2 cognitive memory benchmark testing implicit constraint recall under cue-trigger semantic disconnect---Kumiho achieves 93.3\% judge accuracy ($n{=}401$) with 98.5\% recall accuracy, outperforming the best published baseline (Gemini~2.5~Pro, 45.7\%) by 47.6~percentage points while using GPT-4o-mini for bulk operations at a total cost of ${\sim}$\$14. Three architectural innovations drive both results: \emph{prospective indexing} (LLM-generated future-scenario implications indexed alongside summaries), \emph{event extraction} (structured events with consequences preserving causal detail), and \emph{client-side LLM reranking} (Section~\ref{sec:reranking}), where the consuming agent's own LLM selects the most relevant sibling revision from structured metadata at zero additional cost. The enrichments drove LoCoMo-Plus accuracy from a pre-enrichment baseline of 61.6\% to 93.3\%, eliminating the $>$6-month accuracy cliff (37.5\% $\rightarrow$ 84.4\%). The architecture is model-decoupled: the 98.5\% recall accuracy is invariant across answer models, while end-to-end accuracy scales from ${\sim}$88\% (GPT-4o-mini) to 93.3\% (GPT-4o), with the gap concentrated in goal-type constraints requiring abstract reasoning. Of the 27~failures, 78\% are answer model fabrication on correctly retrieved context---demonstrating that the memory layer has effectively solved the retrieval problem, and the remaining gap is a consumer-side reasoning challenge. Where all tested baselines---including premium models with million-token context windows---achieve 23--46\%, the graph-native primitives (typed edges, structured summarization, prospective indexing, sibling relevance filtering, multi-query reformulation, client-side LLM reranking) provide the structural reasoning capabilities that surface-level retrieval cannot substitute.

Automated AGM compliance verification (49~scenarios, 100\% pass rate) confirms that the implementation faithfully executes the formal specification. Controlled cross-system comparison, consolidation enrichment ablation, and additional benchmarks (LongMemEval, MemoryAgentBench) remain important directions for future work.

In an era where AI agents perform consequential work---producing artifacts, making decisions, and collaborating autonomously in multi-agent pipelines---their memory must serve double duty: as the cognitive substrate for individual agent intelligence and as the shared operational infrastructure through which agents coordinate, build upon each other's outputs, and maintain auditable provenance chains. The seven design principles distilled from this work provide a reusable framework for building such systems: systems where every belief can be traced to its evidence, every output can be found and built upon by downstream agents, and every decision chain is open to human inspection.

\bibliographystyle{plainnat}

\appendix

\section{Key Data Structures}

\textbf{Memory Reference URI Format:}
\begin{small}\begin{verbatim}
kref://project/space[/sub]/
  item.kind[?r=N][&a=art]
\end{verbatim}\end{small}

\textbf{Revision Node Properties:}
\begin{small}\begin{verbatim}
{
  ref: "kref://proj/space/item.kind?r=N",
  _search_text: "concat. searchable text...",
  embedding: [float64 x 1536],
  embedding_text: "text embedded",
  embedding_updated_at: datetime,
  metadata: {
    schema, type, summary,
    topics, keywords
  },
  created_at: datetime,
  author: "agent-id"
}
\end{verbatim}\end{small}

\textbf{Redis Working Memory Key Structure:}
\begin{small}\begin{verbatim}
cogmem:{proj}:sessions:{sid}:messages
cogmem:{proj}:sessions:{sid}:metadata
cogmem:{proj}:consol_queue
\end{verbatim}\end{small}

\textbf{Dream State Assessment Structure:}
\begin{small}\begin{verbatim}
MemoryAssessment {
  revision_ref: str
  relevance_score: float  // 0.0-1.0
  should_deprecate: bool
  deprecation_reason: str
  suggested_tags: List[str]
  metadata_updates: Dict[str, str]
  related_memories: List[(ref, type)]
}
\end{verbatim}\end{small}

\section{Dream State Report Format}

Each Dream State run produces a Markdown report stored as a revision artifact, documenting events processed, memories assessed, actions taken (deprecations, metadata updates, tags added, relationships created), and the cursor for resumption.

\begin{small}\begin{verbatim}
# Dream State Report -- 2026-02-04T03:00Z
**Events:** 42 | **Assessed:** 18
**Duration:** 4500ms

## Actions Taken
### Deprecated (3)
- kref://CognitiveMemory/personal/...
  -- duplicate
### Metadata Updated (7)
- kref://CognitiveMemory/work/...
  -- topics added
### Tags Added (12)
- kref://CognitiveMemory/personal/...
  -- ssh, error
### Relationships Created (5)
- source -> target (DERIVED_FROM)

## Cursor
eyJjdXJzb3IiOiIxMjM0NTY3ODkwIn0=
\end{verbatim}\end{small}

\end{document}